\pdfoutput=1
\documentclass{article}

\usepackage[utf8]{inputenc}
\usepackage[T1]{fontenc}
\usepackage[letterpaper,margin=1in]{geometry}
\usepackage{microtype}
\usepackage{graphicx}
\usepackage{booktabs}

\usepackage{amsmath}
\usepackage{amssymb}
\usepackage{amsthm}
\usepackage{mathtools}

\usepackage[round,semicolon]{natbib}
\usepackage{hyperref}
\usepackage[capitalize,noabbrev]{cleveref}

\newcommand{\norm}[1]{\left\|#1\right\|}
\newtheorem{theorem}{Theorem}[section]

\newtheorem{definition}[theorem]{Definition}
\newtheorem{lemma}[theorem]{Lemma}

\newtheorem{corollary}[theorem]{Corollary}

\newtheorem{remark}[theorem]{Remark}
\newtheorem{proposition}[theorem]{Proposition}

\crefname{condition}{Condition}{Conditions}
\crefname{fact}{Fact}{Facts}
\crefname{claim}{Claim}{Claims}
\crefname{assumption}{Assumption}{Assumptions}

\newcommand{\simplex}{\Delta_d}
\newcommand{\mumax}{\mu_{\max}}

\newcommand{\cV}{\mathcal{V}}
\newcommand{\cC}{\mathcal{C}}
\newcommand{\cG}{\mathcal{G}}
\newcommand{\cGt}{\widetilde{\mathcal{G}}}
\newcommand{\Dl}{D}
\newcommand{\Cl}{\Phi}
\newcommand{\cA}{\mathcal{A}}
\newcommand{\cP}{\mathcal{P}}
\newcommand{\R}{\mathbb{R}}
\newcommand{\T}{\mathbb{T}}
\newcommand{\E}{\mathbb{E}}
\newcommand{\abs}[1]{\lvert #1 \rvert}
\newcommand{\Lip}{\mathrm{Lip}}
\newcommand{\supp}{\mathrm{supp}}
\newcommand{\diam}{\mathrm{diam}}
\newcommand{\cK}{\mathcal{K}}
\newcommand{\cM}{\mathcal{M}}
\newcommand{\cS}{\mathcal{S}}
\newcommand{\cL}{\mathcal{L}}
\newcommand{\cGw}{\mathcal{G}^{w}}
\newcommand{\Mw}{M_{w}}
\newcommand{\Dt}{\widetilde{D}}
\newcommand{\vr}{\varrho}
\newcommand{\Lipo}{\mathrm{Lip}_0}
\newcommand{\cost}{\mathrm{cost}}
\newcommand{\Wone}{W_1}
\newcommand{\Sh}{\mathsf{Sh}}
\newcommand{\Ball}{\mathsf{B}}

\newcommand{\cX}{\mathcal{X}}

\newcommand{\Wt}{\widetilde{W}_1}
\newcommand{\Nt}{\widetilde{N}}
\newcommand{\Rad}{\widehat{\mathcal{R}}}
\newtheorem{convention}[theorem]{Convention}
\crefname{convention}{Convention}{Conventions}

\title{Learning Graded Sparse Mechanisms:\\ A Representation Theorem\\ and Certified Generalization}

\author{
  Mark Kozdoba$^{1}$\\
  Technion, IIT
  \and
  Shie Mannor\\
  Technion, IIT and NVIDIA
}
\date{}

\begin{document}

\maketitle
\footnotetext[1]{\texttt{markk@technion.ac.il}}

\begin{abstract}

In a variety of multi-agent and physical environments, an element's update
depends strongly on a few other elements and weakly on many, and both the
identity and the number of the elements that matter change with the state
of the system. Such graded sparse
mechanisms are common, for instance, in flocks, crowds, learned particle simulators and the neighborhood aggregation of graph networks,
yet they have received little
systematic study as a class of functions with the grading given. In this paper we study the supervised
learning of the mechanism when the grading is provided by the environment,
as a weight on every element. First, we define the class by axioms, where each
element's influence is capped by its weight, the dependence on the weights
is regular, and no element is distinguished, and prove a representation
theorem: the class is exactly the Lipschitz ball of a transport metric on
weighted multisets, and the metric is forced by the axioms. Next, the geometry
of the representation gives generalization bounds that depend on the data in two ways. One is via the sparsity profile of the data---the law of how many elements matter and by how much---which bounds the rate over all designs with that profile, with an exponent from the profile and no dependence on the number of elements. In a number of natural cases we also show matching lower bounds, up to constants in the exponent.
The other is empirical: due to the representation result, classical empirical
discrepancy bounds apply to the class as they stand, and the resulting
bound is the transport distance between the two halves of the sample in
the class's own metric, computed exactly by a matching, with an
additional concentration term.
Finally, we study a simulated Vicsek flock
under four interaction rules at two noise levels, and show that the
certificate is non-vacuous, a small fraction of the trivial value on
every rule.
\end{abstract}

\section{Introduction}\label{sec:intro}

In models of complex systems, a property of an element depends on its own
state and on the states of some other elements. The dependence is often
\emph{sparse} and \emph{graded}: a few elements have a strong influence,
and the rest a weaker and decreasing influence, or none at all. It is
often \emph{changing}: as the system evolves the grading changes, so the
sparse set varies, and varies in size. And it is
often \emph{anonymous}: the elements form a set, with no enumeration that
the mechanism reads.
One mechanism is applied at every state of the
system to the elements the state selects, and we call the environment's rule that assigns the weights from
the state, which elements count and by how much, the \emph{selector}.

The examples of this paper come from multi-agent and physical systems, and
in each of them the selector is the environment's, a rule with at most a
few tuned parameters that reads distance, speed, visibility or rank. The
running example is the Vicsek flock \citep{vicsek1995novel}. Its particles
move at a common speed on a periodic square, and at every step each
particle turns toward the mean heading of the particles within a fixed
distance of itself, up to noise. The selector is the ball: the particles
inside it count with equal weight and the rest not at all, and how many
are inside varies with the local crowding. Graded variants of the rule weight the neighbors by a
kernel of the distance instead, so that the influence fades rather than
stops. The selectors of the other examples are of the same kind: distance
and speed thresholds in drone flocks \citep{vasarhelyi2018optimized}, the
visible neighbors in crowds \citep{wirth2023neighborhood}, the six or seven
nearest in starling flocks \citep{ballerini2008interaction} and fish
schools \citep{gautrais2012deciphering}, a view cone or a sensing radius
in driving simulators \citep{vinitsky2022nocturne,kazemkhani2025gpudrive}
and navigation policies
\citep{huttenrauch2019deep,nayak2023informarl,zhang2025gcbfplus}, and a
cutoff in learned particle simulators
\citep{sanchezgonzalez2020learning,musaelian2023allegro}.

The learning problem is this. A configuration is a pair $(x, \mu)$, the
values $x$ of the $d$ elements and the weights $\mu$ the selector gives
them, a probability vector: the selector is the environment's and is
given, not learned, so the weights are part of the input, and the unknown
is the mechanism $f(x, \mu)$, a member of the class $\cG_L$ of
\Cref{sec:rep}. A sample is $n$ configurations drawn from the
\emph{design}, the law of $(x, \mu)$, with the labels $f(x, \mu)$, under
the absolute loss, and what we bound is the uniform deviation of the class
on the design, the generalization gap of every empirical risk minimizer.

While such systems are classical, the graded sparse mechanism as a class of
functions, and the complexity of learning one with the selector given,
have received little study beyond hard sparsity at a fixed count
(\Cref{sec:lit}), and this paper develops first results toward both.
Taking the selector as given
is the decision the graph-network bounds make in taking the graph as
given. It is also given in simulators, and
in the animal groups some aspects of it are inferred in some cases
(\Cref{sec:lit}).

\paragraph{Axioms and representation.} We propose an axiomatic definition
of the class. The alternative, a function of a fixed number $k_0$ of
arguments fed a different selection of inputs at each configuration, is
the form of neighborhood aggregation in the generalization bounds for
graph networks, stated under a maximum degree $k_0$
\citep{garg2020generalization,liao2021pacbayes}, a worst-case parameter
that sees neither the average sparsity across configurations nor a grading
inside the $k_0$. We define the class instead by three axioms on a
mechanism, a function of the elements' values and of a weight on each
element, so that the number of arguments changes with the configuration,
and changes gradually.
An element's influence is capped by its weight, none at weight zero. The
dependence on the weights is regular: shifting weight between elements
changes the output by no more than the same shift can change a weighted
mean of the values, the simplest mechanism, which the class must contain.
And no
element is distinguished, so elements of equal weight are interchangeable. The axioms concern the mechanism alone: the selector enters only through
the data distribution, and it may be discontinuous, as a hard radius is.

The examples come in two semantics of influence, a mean, fixed in total
and diluted by more participants, and a sum, growing with their number,
and the framework has a class for each, \emph{conservative} and
\emph{accumulative}. The main text works with the conservative class. The
accumulative class has the same representation and rates of the same
exponents, at the price of a factor of the count (\Cref{sec:axioms},
\Cref{app:accumulative}).

The axioms are a list of inequalities, and the structure of the class is
not visible in them. We show that the class is exactly the set of Lipschitz
functions on a metric space of weighted configurations (\Cref{thm:rep}):
all of the locality structure is in the metric, above it the class is a
generic Lipschitz class, and this is what makes the rest of the analysis
possible. The metric is the class's own, each of its terms from one axiom,
a transport distance between weighted configurations read as multisets of
points on a cone over the value space.
We note that the permutation-invariant representation results, Deep Sets
\citep{weyl1939classical,macdonald1995symmetric,zaheer2017deepsets,
wagstaff2022universal}, have a similar structure: there the sum embedding
encodes the invariance, with a generic class above it. There is also an
exact connection: permutation-invariant sum networks as above, with a
Lipschitz, bounded per-element map and our weights on the sum, $\sum_i
\mu_i\varphi(x_i)$, are readouts of finitely many linear functionals of
our embedding, thus constituting a strict subset of our class.

\paragraph{Complexity.} Sparsity is a property of the data, and we prove
generalization bounds for the class that are data dependent in two ways,
both resting on the representation. The first is through the
\emph{sparsity profile} of the design, the law of the weight that a
configuration leaves beyond its $k$ heaviest elements: \Cref{fig:certdim}(a)
is the profile of the flock, the mean weight beyond the $k$ heaviest
neighbors under the ball and under three graded kernels at its scale. The
bounds have an exponent set by the profile, the worst case over the designs
with that profile, and no dependence on the number of elements, and lower
bounds in the profile show that the dependence is the right one
(\Cref{sec:complexity}).
The second is through the sample
itself. It gives no rate, but a bound at the actual sample size, and
often a better one. The bound is the classical maximum discrepancy of
\citet{bartlett2002model}, the largest difference, over the class, between
the averages of a function on the two halves of the sample. For a
Lipschitz class this is a transport distance \citep{luxburg2004distance},
and here it is computed exactly by a matching.
The result is an \emph{empirical generalization certificate}, the value of
the bound on the data with one concentration term, which on the
flock is $0.034\kappa$ on the hard ball, with $\kappa$ the value at which
the bound says nothing (\Cref{sec:flocking}).

To summarize, the contributions in this paper are as follows.
\textbf{(i)} \emph{Axioms and representation.} We define the class of
graded sparse mechanisms by axioms and prove that it is the Lipschitz ball
of a specific metric space, whose metric encodes the locality
(\Cref{thm:rep}).
\textbf{(ii)} \emph{Rates from the profile.} We define the sparsity profile
of the design and prove generalization bounds with exponent determined by
the profile, with no dependence on the number of elements
(\Cref{thm:gradedmain,cor:sparse,cor:exp}), and lower bounds matching up
to constants in the exponent in several cases
(\Cref{cor:sparse,cor:exp,thm:learnlower}).
\textbf{(iii)} \emph{The certificate.} We show that the generalization gap
of the class can be certified from the sample alone, by a matching between
its two halves (\Cref{thm:twosample}).
\textbf{(iv)} \emph{The flock.} We study the Vicsek flock under four
interaction rules at two noise levels, measure their sparsity profiles,
compute the certificate for every rule, a thirtieth to a sixth of the
trivial value, and read it against a covering certificate and against
the gap of a member with zero training risk (\Cref{sec:examples}).

\section{Literature}\label{sec:lit}

\paragraph{Graph networks.} A message-passing layer is our object with a
hard selector, the neighborhood of each node, and it is of the Deep Sets
form, a per-neighbor map summed and read out, one readout of finitely many
linear sums of our embedding (\Cref{sec:intro}), at a fixed width, a
bounded degree and no weighting by the environment. Graph networks gain
their expressivity from depth, where their sparsity lives: the receptive
field is set by the graph and the number of layers, and the bounds of
\citet{garg2020generalization} and \citet{liao2021pacbayes} pay for it
with the maximum degree $k_0$ per layer, compounded over the depth,
\citet{ju2023generalization} with its spectral relaxation. Here the
sparsity is explicit, the weights of the selector, and the bound reads it
as the profile: the number of neighbors enters through its law, with rates
where there is no nominal sparsity at all and no dependence on the maximum
count where there is one (\Cref{sec:profiles}, \Cref{rem:gnncurrency}).
Whether any finer statistic of the data than the maximum degree can be
used was posed as a question in that literature
\citep{liao2021pacbayes,morris2023wl,morris2024position,vasileiou2025survey}. The two
kinds of bound also bound different things: theirs hold over a parametric
family at a norm budget and, like our rates, are statements about scaling
(\Cref{app:related}), and the class adds a
bound of a different kind, the certificate, which holds for the whole
class at once and is computed from the sample by a matching
(\Cref{sec:certificate}).
A second line puts a metric on
the graphs under which message-passing networks are Lipschitz
\citep{chuang2022tree,levie2023graphon,boker2023fine,rauchwerger2025generalization,amran2026graphop},
and covers the graph space at one scale, or with the sample's cell
proportions at one radius \citep{carrasco2025rademacher}, Covered Forest
\citep{vasileiou2025covered} being its latest form. We differ in
the object, a general weighted class whose metric the class itself
implies (\Cref{thm:rep}) rather than one metric among several under which
a model is Lipschitz, and in the data dependence: rates through the
profile (\Cref{sec:profiles}) and, on the data, the two-sample transport
in place of a cover (\Cref{sec:certificate}).

\paragraph{The matching distance.} The metric of \Cref{thm:rep} is
produced by the axioms, which determine it in closed form, a matching on a
cone: the identification is new, and the closed form is what makes the
rest computable. The metric space itself,
a cone with a transport that may leave a point unmatched at a price, is
known: unbalanced transport \citep{chizat2018unbalanced,liero2018optimal},
the matching distances of persistence diagrams
\citep{cohensteiner2010lipschitz,turner2014frechet}, of multisets
\citep{chuang2022tree,davidson2025holder,nikolentzos2026lipschitz} and of
point patterns
\citep{schuhmacher2008new,schuhmacher2008consistent,huang2026wasserstein},
none of them with the correspondence the theorem establishes, between an
element's weight and the mechanism's continuity in it.

\paragraph{Rates from a dimension.} The dimensions of a law that bound the
transport distance between an empirical measure and its law, the upper Wasserstein dimension of \citet{weed2019sharp} and the
entropic dimension of \citet{chakraborty2025gan}, rest on one bulk covered
at fixed dimension, the rest of the mass charged at the diameter. Our
bounds stratify the design by its weight tail, a device with earlier
variants \citep{fournier2015rate,lei2020convergence,huang2026wasserstein},
made to work on configurations that are never exactly sparse.
Quasi-polynomial rates of the shape of \Cref{cor:exp} appear before in
two other settings, empirical measures in a Hilbert space under an
exponential-ellipsoid moment condition \citep{lei2020convergence} and
point processes with a geometric count law \citep{huang2026wasserstein}. Ours is for
the mean weight tail of a configuration, with free weights and values,
where the decay is inside each sample and there is no count to stratify
by, and on fixed, geometrically decaying weights the class's metric
resembles Lei's ellipsoid (\Cref{app:lowerread}).

\paragraph{Interaction kernels.} The closest setting to the flock is the
learning of interaction laws from trajectories
\citep{lu2019nonparametric,lu2021learning}, where the mechanism is a
mean-field sum with an unknown scalar kernel, a linear sum in the terms of
\Cref{sec:axioms}, learned at the one-dimensional rate with no profile in
play (\Cref{app:related}).
On real and simulated flocks, attention networks learn the update with
the weights learned and read off the neighbors that matter
\citep{heras2019deep,ha2021unraveling}, and relational inference learns
the interaction graph \citep{kipf2018neural}: the selector is the object
learned there, and it is given here.

The nearest results in learning theory are on transformers: a target on
a few tokens at given positions, a hard selection at a fixed count in a
parametric class
\citep{mousavihosseini2025transformers,edelman2022inductive,sanford2023representational,trauger2024sequence},
an input-dependent importance of the coordinates
\citep{takakura2023approximation} or a per-sample sensitivity cap
\citep{nagler2023statistical}, and further remarks on the literature are
in \Cref{app:related}.

\section{Locality Axioms and Representation}\label{sec:rep}

This section defines the class by its axioms and proves its
representation.

\subsection{Axioms}\label{sec:axioms}

\paragraph{Mechanism and selector.} Throughout, $(\cV, \rho)$ is a compact
metric space with at least two points, the \emph{value space}, and
\[
  R \;:=\; \tfrac12\,\diam\cV
\]
is the \emph{value scale}, so that $\rho \le 2R$ with equality attained. A
\emph{configuration} is a pair $(x, \mu) \in \cC := \cV^d \times \simplex$:
a value $x_i \in \cV$ on each of $d$ coordinates and a weight vector $\mu$ in
the simplex. A \emph{mechanism} is a function $f : \cC \to \R$. A
\emph{selector} produces the weights from the state of the system: each
element carries a state $\omega_i \in \Omega$, of which its value $x_i =
\pi(\omega_i)$ is the part the mechanism reads, for a fixed map $\pi : \Omega \to \cV$. The selector is a map $c : \Omega^d \to \simplex$. The
object of the examples is the composite
\begin{equation}\label{eq:composite}
  g(\omega) \;=\; f\bigl(\pi(\omega), c(\omega)\bigr), \qquad \pi(\omega)
  := (\pi(\omega_1), \dots, \pi(\omega_d)),
\end{equation}
one mechanism applied at every state to the coordinates the selector
weights. In the flock the state of a particle is its position and heading,
the values are the headings, the selector is uniform on the particles
within the interaction radius, and the mechanism is the mean heading. The
axioms concern $f$ alone. The selector enters only through the
\emph{design}, the law $Q$ of $(\pi(\omega), c(\omega))$ under the law of
the state $\omega$, and nothing is assumed about it:
 it may be discontinuous, as a hard radius is, and in the
examples it usually is. Throughout, the selector is given: the learner
observes the configuration $(x, \mu)$, weights included, and learns $f$. Which coordinates of the state are the values, and in which semantics the
mechanism is read, are part of the specification (\Cref{app:flocking},
\Cref{cor:fiber}). Permutations $\sigma$ of
$[d]$ act on both
arguments, $(\sigma x)_i = x_{\sigma^{-1}(i)}$ and likewise for $\mu$. Write
$\supp\mu = \{i : \mu_i > 0\}$ and $\mumax = \max_i \mu_i$.

\paragraph{Cap weights.} The dependence allowed on a coordinate is a
function of its weight, and the axioms are stated for a generic such
function. A \emph{cap weight} is a continuous map $w : \simplex \to
[0,1]^d$ that commutes with permutations, is injective, and has $w_i(\mu) =
0$ exactly when $\mu_i = 0$. The two instances of this paper are
\[
  w^{\mathrm c}(\mu) \;=\; \mu
  \qquad\text{and}\qquad
  w^{\mathrm a}(\mu) \;=\; \mu/\mumax ,
\]
the weights themselves, the \emph{conservative} cap of the mean semantics,
and the weights relative to the largest one, the \emph{shares}, the
\emph{accumulative} cap of the sum semantics. On weights uniform over $k$
coordinates the first allots influence $L/k$ to each selected coordinate
and the second $L$, a fixed budget that more participants dilute against
a total that accumulates with the count. The \emph{conservative} class
$\cG_L := \cG^{w^{\mathrm c}}_L$ is the class of the main text. The
\emph{accumulative} class $\cGt_L := \cG^{w^{\mathrm a}}_L$ has the same
representation, \Cref{thm:rep} below holding for every cap weight, and
rates of the same exponents at the price of a factor of the count
(\Cref{app:instances}, \Cref{app:accumulative,app:capprice}).
\Cref{tab:mechanisms} in \Cref{app:instances} sorts the common
aggregations between the two classes, with their constants.

\begin{definition}[The class $\cGw_L$]\label{def:class}
For a cap weight $w$ and $L \ge 0$, $\cGw_L$ is the set of mechanisms
$f : \cC \to \R$ such that
\begin{enumerate}
\item[(V)] \emph{(value cap)} if $x, y \in \cV^d$ differ in coordinate $i$
only, then for every $\mu$,
$\abs{f(x, \mu) - f(y, \mu)} \le L\, w_i(\mu)\, \rho(x_i, y_i)$.
\item[(W)] \emph{(weight regularity)} for every $x \in \cV^d$ and all
$\mu, \nu \in \simplex$,
$\abs{f(x, \mu) - f(x, \nu)} \le L\, R\, \norm{w(\mu) - w(\nu)}_1$.
\item[(A)] \emph{(anonymity)} $f(\sigma x, \sigma\mu) = f(x, \mu)$ for every
permutation $\sigma$.
\end{enumerate}
\end{definition}

(V) is the locality and (A) the anonymity of the introduction, and together
they make $f$ a function of the multiset of pairs $\{(x_i, \mu_i) : \mu_i >
0\}$ (\Cref{lem:basic}). On an interval and for differentiable $f$ the
axioms are bounds on partial derivatives (\Cref{rem:derivative}). The
weight axiom (W) is the one that needs a justification.

\paragraph{Where the weight axiom comes from.} The class is meant to
contain the weighted sums of a score, $\sum_i w_i(\mu)\,\varphi(x_i)$ with
$\varphi$ $1$-Lipschitz and bounded by $R$, the \emph{linear sums} of
\Cref{app:basic}. (W) is exactly their sensitivity to the weights: a move
of weight from $\mu$ to $\nu$ changes a linear sum by at most
$R\,\norm{w(\mu) - w(\nu)}_1$, and some score attains this
(\Cref{prop:modulus}), so (W) is the tightest weight axiom under
which every linear sum is still a member at constant $1$. The axioms are
chosen so that the simplest mechanisms have constant $1$, and what
\Cref{thm:rep} adds is that the metric they induce is the matching
distance at the same constant (the coordinatewise form of the axiom and
the readouts it excludes are in \Cref{app:basic}).
The modulus in (W) is stated in the weights alone, and this is a choice:
pricing a weight move by the values it moves between gives the flat metric
on the weighted value measure $\sum_i\mu_i\delta_{x_i}$, a smaller class in
which the weights are mass rather than individuals, that of distribution
regression and the measure-input networks (\Cref{app:related}), which
excludes the selected maximum and the top weight, members here
(\Cref{rem:whyuniform}).
A mechanism may also read a distinguished input of its own, the state of
the agent evaluating it, as in \Cref{sec:examples}, and everything below
then holds fiberwise in it (\Cref{rem:slots}).

\subsection{Representation}\label{sec:representation}

\paragraph{The cone and the lift.} The \emph{cone} over $\cV$ is
$\cK := (\cV \times [0,1]) / (\cV \times \{0\})$, the bottom edge collapsed
to one point, the \emph{apex} $o$. A point other than the apex is a pair
$(x, u)$ of a value and a \emph{height} $u \in (0, 1]$, and the \emph{cone
cost} is
\begin{equation}\label{eq:conecost}
  \vr\bigl((x,u), (y,v)\bigr) \;:=\; \min(u,v)\,\rho(x,y) + R\,\abs{u - v},
  \qquad
  \vr\bigl((x,u), o\bigr) \;:=\; R\,u ,
\end{equation}
a metric on $\cK$ (\Cref{lem:cone}). The \emph{lift} of a configuration is the counting measure on the cone
\begin{equation}\label{eq:lift}
  \Cl(x, \mu) \;:=\; \sum_{i \in \supp\mu} \delta_{(x_i,\, w_i(\mu))} ,
\end{equation}
one atom per coordinate of positive weight, at its value and at the height
of its cap weight.
Write $\cM_d := \Cl(\cC)$, and $n_P$ for the number of
atoms of $P \in \cM_d$. Two configurations have the same lift exactly when
their multisets of (value, weight) pairs coincide, so $\Cl$ identifies $\cC$
modulo permutations and modulo the values at zero-weight coordinates.
Every member of $\cGw_L$ factors through it (\Cref{lem:factor}).

\paragraph{The matching distance.} For $P, Q \in \cM_d$, a \emph{partial
matching} is a bijection $\sigma$ between a sub-multiset of the atoms of $P$
and a sub-multiset of the atoms of $Q$. Its cost is the cone cost summed
over the matched pairs plus $R$ times the height of every unmatched atom on
either side, and the \emph{matching distance} is the least cost,
\begin{equation}\label{eq:matching}
  \Mw(P, Q) \;:=\; \min_{\sigma}\Bigl[\,
  \sum_{p \text{ matched}} \vr\bigl(p, \sigma(p)\bigr)
  + \sum_{p \in P \text{ unmatched}} \vr(p, o)
  + \sum_{q \in Q \text{ unmatched}} \vr(q, o) \Bigr] .
\end{equation}
It is a metric on $\cM_d$, and it is a transport distance: padding each side
with the other's number of atoms at the apex,
\begin{equation}\label{eq:padding}
  \Mw(P, Q) \;=\; \Wone^{\vr}\bigl(P + n_Q\,\delta_o,\; Q + n_P\,\delta_o\bigr),
\end{equation}
the Kantorovich distance with ground cost $\vr$ between two measures of
equal mass, whose optimal plans are matchings because the masses are
integers (\Cref{lem:metric}). We write $\Dl := M_{w^{\mathrm c}}$ for the
conservative instance, heights equal to the weights, and $\Dt := M_{w^{\mathrm a}}$ for the accumulative one, heights equal to the
shares.

\begin{theorem}[Representation]\label{thm:rep}
Let $w$ be a cap weight and $L \ge 0$.
\begin{enumerate}
\item[(i)] A mechanism $f$ belongs to $\cGw_L$ if and only if $f = F \circ
\Cl$ for a function $F : \cM_d \to \R$ with $\abs{F(P) - F(Q)} \le L\,\Mw(P,
Q)$ for all $P, Q \in \cM_d$. The constant is $1$ in both directions, for
all configurations.
\item[(ii)] For all $P, Q \in \cM_d$,
\[
  \Mw(P, Q) \;=\; \sup\bigl\{\, f(z) - f(z') \;:\; f \in \cGw_1,\
  \Cl(z) = P,\ \Cl(z') = Q \,\bigr\} .
\]
Consequently, as for any metric and its Lipschitz ball, every metric on
$\cM_d$ for which the unit ball $\cGw_1$ is $1$-Lipschitz dominates $\Mw$,
and $\Mw$ is the unique metric whose Lipschitz balls, composed with $\Cl$,
are the classes $\cGw_L$.
\end{enumerate}
\end{theorem}

Part (i) is the representation: the class is the Lipschitz ball of $(\cM_d,
\Mw)$ in the axioms' own constant, with nothing lost. Its proof is a
routing: the axioms bound $f$ along three moves, one value at fixed
weights, one weight at fixed values, a permutation, so $f$ is Lipschitz for
the path metric the moves generate, and the theorem computes that metric
in closed form as \eqref{eq:matching} by moving each matched value at the
smaller of its two heights, before the weight move where the height rises
and after it where it falls, at exactly the matching's cost
(\Cref{app:thmproof}; compare \citealp{balan2022permutation,cahill2024groupinvariant}). Part (ii) is a
corollary of (i), a metric being the supremum of the differences of its
own $1$-Lipschitz functions, and it is the sense in which the metric is
the class's own, each of its ingredients one of the axioms (the price $R$
of a unit of height is the modulus of \Cref{prop:modulus}). The class is convex, closed under Lipschitz readouts of finitely many
members, and contains the vector-valued linear sums, so attention pooling
and the flocking resultant are members
(\Cref{app:thmproof}, \Cref{lem:familycap}). On a fixed hard selection of
$k$ coordinates the class is the symmetric Lipschitz functions of the
selected values with the constant $L/k$, the factor $k$ attained by the
selected maximum (\Cref{cor:fiber}).

\section{Complexity}\label{sec:complexity}

A \emph{design} is a Borel probability measure $Q$ on $\cC$, the law of
the configurations, in the composite \eqref{eq:composite} the law of
$(\pi(\omega), c(\omega))$, and its
\emph{lifted design} is $\Gamma := \Cl_\# Q$, a probability measure on
$(\cM_d, \Dl)$. The data are $Z_1, \dots, Z_n$ i.i.d.\ $\sim Q$, with
lifts $P_i = \Cl(Z_i)$ and empirical lifted design $\hat\Gamma_n :=
n^{-1}\sum_{i \le n}\delta_{P_i}$, and $\kappa := 2R$ bounds the diameter
of $(\cM_d, \Dl)$ (\Cref{prop:qual}). The unit ball $\cG_1$ is the set of $F \circ \Cl$ with $F$ $1$-Lipschitz on
$(\cM_d, \Dl)$ (\Cref{thm:rep}), so the generalization gap of the class
on a sample is a transport,
\[
  \Delta_S \;:=\; \sup_{f \in \cG_1}\Bigl(\E_Q f - \frac1n\sum_{i=1}^n
  f(Z_i)\Bigr) \;=\; \Wone(\hat\Gamma_n, \Gamma) ,
\]
by Kantorovich--Rubinstein duality \citep{sriperumbudur2012empirical}.
Every bound of this section bounds the mean $\Delta_n(Q) := \E\,\Delta_S$
or its value at the sample at hand, a statement about the design's
geometry: the class itself is never covered.

\subsection{Rates from the sparsity profile}\label{sec:profiles}

\begin{definition}[Tail mass and profile]\label{def:profile}
For $(x, \mu) \in \cC$ let $\mu_{(1)} \ge \mu_{(2)} \ge \cdots$ be the
weights in decreasing order and, for $0 \le k \le d$, $T_k(x, \mu) :=
\sum_{i > k}\mu_{(i)}$ the weight beyond the $k$ heaviest coordinates, a
continuous function of the lift alone (\Cref{lem:descend}). The
\emph{profile} of the design is $G_k(t) := Q(T_k > t)$ for $0 \le k \le d$
and $t \in [0, 1)$, nonincreasing in $k$ and in $t$.
\end{definition}

\begin{definition}[Sparsity profiles]\label{def:profiles}
A design is \emph{exactly $k_0$-sparse} if $Q(T_{k_0} = 0) = 1$, a
selector that never reads more than $k_0$ coordinates. It has an
\emph{exponential profile} at rate $\beta > 0$ if $\E_Q[T_k] \le A
e^{-\beta k}$ for some $A \ge 1$ and all $k \ge 0$, a statement about the mean tail (\Cref{fig:certdim}(a)).
\end{definition}

The profile decays in $t$ at fixed $k$, the weights of one configuration,
and in $k$ at fixed $t$, the law of how many coordinates matter. The rates
are stated through the
\emph{stratum exponent} $e(k)$, the covering exponent of the lifts
carried by at most $k$ atoms, for a value space of metric dimension $\dim\cV$, $N(\varepsilon, \cV, \rho) \le (C_\cV R/\varepsilon)^{\dim\cV}$. On interval values $e(k) = 2k - 1$, one
value and one free weight per atom. On a \emph{selector design}, one whose
cropped weights are, at every crop level, an $L_c$-Lipschitz function of
the selected values or are uniform on the support, the ball of
\Cref{sec:flocking} being one, $e(k) = k\dim\cV$ (\Cref{app:strata}).
No upper bound below depends
on $d$. The stratum of $k$ atoms has covering count $N_k(\varepsilon) \le
(C_1R/\varepsilon)^{e(k)}$ at resolution $\varepsilon \le \kappa$, by balls of
radius $\varepsilon/2$, with $C_1$
depending on the covering constant of $\cV$, and on $L_c$ on a selector
design (\Cref{app:strata}).
\begin{theorem}[Generalization over the class]\label{thm:gradedmain}
Fix $K \le d$ and resolutions $0 < \varepsilon_1 \le \cdots \le
\varepsilon_K$ with thresholds $\theta_k := \varepsilon_k/(2\kappa) < 1$,
and let $w_k := Q(T_k \le \theta_k,\ T_{k-1} > \theta_{k-1})$, with
$\theta_0 := \theta_1$, be the mass of the configurations that $k$ atoms
resolve at $\varepsilon_k$ and $k - 1$ do not at $\varepsilon_{k-1}$. Then
for every design $Q$,
\begin{equation}\label{eq:gradedmain}
  \Delta_n(Q) \;\le\;
  2\sum_{k=1}^{K} w_k\,\varepsilon_k
  \;+\; \frac{\kappa}{2}\sum_{k=1}^{K}\sqrt{\frac{w_k\,N_k(\varepsilon_k)}{n}}
  \;+\; \kappa\,G_K(\theta_K) .
\end{equation}
\end{theorem}

The three terms are each shell's resolution at its mass, its cell-frequency
fluctuation at its count, and the trim beyond the last shell at the
diameter, or twice its own tail if cropped (\Cref{thm:graded}). The two
theorems below follow from it and a multiscale refinement of its coupling
(\Cref{prop:multiscale}), each with its converse.

\begin{theorem}[A single sparsity level]\label{cor:sparse}
For every design and every $K \le d$ with $e(K) \ge 3$,
\[
  \Delta_n(Q) \;\le\; 2\kappa\,\E_Q T_K \;+\; C R\; n^{-\frac{1}{e(K)}} ,
\]
with $C$ depending on the covering constant of $\cV$ only, and on $L_c$
on a selector design, uniformly in $K$ and in $d$ (the case $e(K) \le 2$
is in \Cref{app:sparseproof}): an exactly $k_0$-sparse design learns at
$n^{-1/e(k_0)}$ once $e(k_0) \ge 3$, which is $n^{-1/(2k_0-1)}$ on
interval values and $n^{-1/(k_0\dim\cV)}$ on a selector design.
Conversely, on interval values there is an exactly $k_0$-sparse design
$Q_{k_0}$, with weights bounded below, such that every $n$-atom measure
$\nu_n$ has $\Wone(\nu_n, \Gamma_{k_0}) \ge c_{k_0}\,n^{-1/(2k_0-1)}$, for
a constant $c_{k_0}$: on interval values the worst-case exponent of a
sparsity level $k_0 \ge 2$ is $1/(2k_0-1)$ from both sides.
\end{theorem}

The first term is the mean tail paid as a bias, the cost of cropping
every configuration to $K$ atoms, and the second is the rate of the
exactly $K$-sparse design so obtained, by a multiscale coupling on the
cropped stratum (\Cref{prop:multiscale}), whose exponent on interval
values is the converse's (\Cref{app:lowerread}), and on the converse's
design no learner, from any number of noiseless labels, has risk below
$c\,n^{-1/(2k_0-1)}$ uniformly over the class (\Cref{thm:learnlower}).
The simplest members escape the rate: the
conservative linear sums \eqref{eq:linearsums}, the shape of a sum
network, learn at $n^{-1/(\dim\cV+2)}$ whatever the count and the profile
(\Cref{prop:countfree}), and the class pays the profile for everything
beyond Lipschitz readouts of finitely many of them.

\begin{theorem}[Exponential profiles]\label{cor:exp}
Let $\cV = [-R, R]$ or a circle of diameter $2R$.
Under an exponential profile at rate $\beta$, for every $n \ge 1$,
\[
  \Delta_n(Q) \;\le\; C\, e^{-\sqrt{\beta\ln n/2}} ,
\]
and $\Delta_n(Q) \le C\,e^{-\sqrt{\beta\ln n}}$ on a selector design,
with $C$ depending on $(\beta, A, R)$, and on $L_c$ on a selector design.
Conversely there is a design
$Q^\beta$ with $T_k \le e^{-\beta k}$ at every configuration such that,
for every $n$ with $d \ge \sqrt{2\ln(2n)/\beta} + 1$, every $n$-atom
measure $\nu_n$ has $\Wone(\nu_n, \Gamma^\beta) \ge c\,
e^{-\sqrt{2\beta\ln(2n)}}$, with $c$ depending on $(\beta, R)$.
\end{theorem}

The rate is slower than every polynomial in $n$ and faster than every
power of $\ln n$, and it is the problem's own, the constant in the
exponent lying between $1/\sqrt2$ and $\sqrt2$. The optimizing cutoff in
\Cref{cor:sparse} is $K_n \asymp \sqrt{\ln n/\beta}$, so the dimension
the rate sees grows without bound, until it reaches $d$, and the profile costs nothing for being
a statement about the mean tail, the crop paying each configuration its
own tail (\Cref{rem:expconstants}). The circle is the value space of the
flock (\Cref{app:flocking}).

\subsection{The certificate}\label{sec:certificate}

\begin{theorem}[The two-sample certificate]\label{thm:twosample}
Let $n$ be even and let $S$ be split, by a rule fixed in advance, into
halves $A$ and $B$ of $n/2$ points each, with empirical lifted designs
$\hat\Gamma_A$ and $\hat\Gamma_B$. With probability at least $1 -
\delta$,
\[
  \Wone(\hat\Gamma_A, \Gamma) \;\le\; \Wone(\hat\Gamma_A, \hat\Gamma_B)
  + \kappa\sqrt{\frac{\ln(1/\delta)}{n}} ,
\]
and $\Wone(\hat\Gamma_A, \hat\Gamma_B) = \min_\pi\,\frac{2}{n}\sum_{i \le
n/2}\Dl(P_i, P'_{\pi(i)})$, the minimum over the bijections $\pi$ between
the lifted points of the two halves: an assignment problem on $n/2$
points, each of whose costs is itself an assignment between the atoms of
two lifts (\Cref{lem:Dassign}).
\end{theorem}

What is ours is the object: on the class the maximum discrepancy
\citep{bartlett2002model} is a transport by duality
\citep[eq.~(11)]{luxburg2004distance}, a matching computes it, and the
symmetrization of \citet{vapnik1971uniform} carries it to the population
with one term (\Cref{thm:twosamplefull},
\Cref{rem:onesplit,rem:rademacher,rem:certread}). The matching is the
least of the certificates on the halves, every covering-type bound on them
being one feasible transport (\Cref{prop:sampleshells}): on the flock the
cell certificate is twenty times the matching on the ball and trivial on
every graded kernel (\Cref{sec:flocking}).

\section{Experiments}\label{sec:examples}

\subsection{Flocking}\label{sec:flocking}

The Vicsek flock of the introduction \citep{vicsek1995novel} is the
reference model of collective motion \citep{vicsek2012collective,
ginelli2016physics}, and its update is a mechanism of the class on the
ordered region (\Cref{app:flocking}): the configuration is the set of headings, the
selector the ball, uniform on the $m + 1$ particles within distance $r$ of
the focal one, itself included, and the \emph{update} the direction of
the mean heading vector, up to an angular noise of amplitude $\eta$.
We simulate the
standard setting, $3200$ particles at density $2$ with $r = 1$, at noise
$\eta = 1.2$ in the ordered phase \citep{chate2008collective}. Alongside
the ball rule we simulate three flocks whose rules are graded, a bump, a
Gaussian and a Cauchy kernel, each at the range at which its mean tail
matches the ball's, at a finite cutoff, and each read through the
selector of its own family with its lift cropped at a priced tail
(\Cref{app:flocking}).
A \emph{sample} is one focal particle of one replica at one instant, one
sample per replica, so the samples are independent: $64{,}000$ replicas
of each flock, each started at random and read once at step $3000$, after
the flock has ordered, the view's headings rotated by an independent
uniform angle: the design is the law of that view.
The regime $\eta = 2.4$ and the kernels' parameters are in \Cref{app:flocking}.

\paragraph{The profile.} \Cref{fig:certdim}(a) is the mean tail of each selector, and \Cref{fig:count} in
\Cref{app:flocking} the count law of the ball flock. On the replicas
the ball holds $11$ particles on average and $67$ at most,
other than the self, from $2$ to $23$ between the $5\%$ and $95\%$ samples. Every sample is exactly sparse
with uniform weights on its $m + 1$ atoms, so the mean tail is a
functional of the count law. It is geometric beyond the knee at the mean
support because the count law is, at a decay rate of $0.23$ per atom. The
graded kernels are the other case, a decay inside each sample at a rate set
by the local crowding, the Gaussian's geometric, the average-exponential
profile of \Cref{cor:exp} on a design whose weights fade rather than stop,
and the Cauchy's polynomial. The profile bounds give the rate and the
certificate gives the number: the ball design's support reaches $68$ atoms
and a kernel's is every particle within its cutoff, so a bound in the
dimension of the support says nothing at these $n$, where the certificate does.
The rate
the sample shows is the design's own: over the last three doublings the
ball's transport decays with the exponent of a three-dimensional sample and
the graded kernels' with that of thirteen to fifteen, rising by a
dimension per doubling as the sample resolves more atoms of the weight
vector, both far below the profile's worst-case exponent, eleven at the
ball's mean count (\Cref{app:flocking}): the ordered flock's headings are
nearly collinear, and the profile bound is over every design with that
profile.
That number bounds the gap of every Lipschitz function of the local
configuration, on a design of strata of dimension one to sixty-eight,
twelve on average (\Cref{fig:count}), the $3200$ headings of the flock's
state nowhere in it.

\begin{figure}[t]
\centering
\raisebox{-\height}{\includegraphics[width=0.29\textwidth]{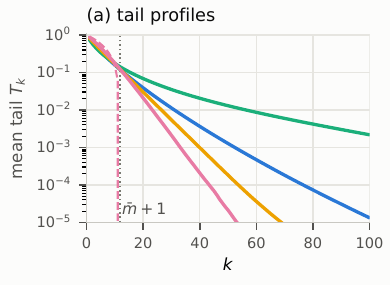}}\hspace{0.018\textwidth}%
\raisebox{-\height}{\includegraphics[width=0.28\textwidth]{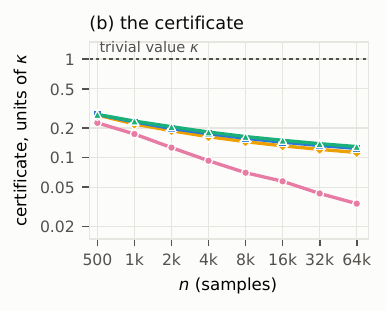}}\\[-2pt]
\includegraphics[width=0.56\textwidth]{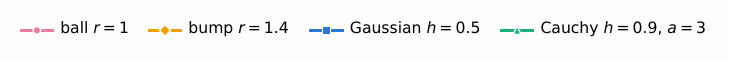}
\caption{The flock at $\eta = 1.2$, $64{,}000$ independent replicas per
flock, one sample each. \emph{(a)} The profile: the mean tail $T_k$, the
weight beyond the $k$ heaviest atoms, under the four selectors, the
dashed line the tail of one sample at the mean count $\bar m + 1$.
\emph{(b)} The certificate of \Cref{thm:twosample} for half $A$ at
$\delta = 0.05$, in units of $\kappa$, on samples after the burn-in of
$3000$ steps, by which the flock has ordered.}
\label{fig:certdim}
\end{figure}

\paragraph{The certificate.} \Cref{fig:certdim}(b) is \Cref{thm:twosample}
computed on the split by replicas, half $A$ the first half of the replicas
and half $B$ the second, fixed in advance. The transport is one
assignment on the halves (\Cref{lem:Dassign}), $32{,}000$ samples per half
at the largest $n$. At $\eta = 1.2$ the certificate is $0.034\kappa$ on
the ball and $0.11$ to $0.13\kappa$ on the graded kernels, where the mean
distance between two random samples is $0.61$ to $0.68\kappa$ and $\kappa
= \pi$: a component of the local mean velocity, of range $2$, deviates by
at most $0.11$. Every step here is a theorem on the data:
the samples are independent replicas, so the i.i.d.\ hypothesis of
\Cref{thm:twosample} holds as stated, and the concentration term is its
$0.007\kappa$ at $n = 64{,}000$.

\paragraph{What the certificate bounds.} The certificate is the exact
uniform deviation of the class on the two halves, $\Wone(\hat\Gamma_A,
\hat\Gamma_B) = \max_{f \in \cG_1}(\E_A f - \E_B f)$, so no bound uniform
over the class on these halves is smaller, and by \Cref{rem:regression} it bounds the generalization gap
of every empirical risk minimizer of $\cG_L$ on every target in $\cG_L$ by
$2L$ times its value. On a member it is nearly attained: for the target
$L\,\Dl(\cdot, A)$, zero on half $A$, the lower Lipschitz extension
$-L\,\Dl(\cdot, A)$ is a member of $\cG_L$ with zero training risk whose
gap on half $B$ is $2L$ times the mean nearest-neighbor distance from $B$
to $A$, $0.038\kappa$ against the certified $0.068\kappa$ at $L = 1$ on the
ball, and within a factor $1.13$ of the certificate on every graded kernel
(\Cref{app:flocking}). For one chosen learner on one target a holdout
estimate on half $B$ is tighter by definition, and the certificate does
not compete with it: what it gives is uniformity at no labels, over every
target on the design and every predictor in the class at once, in hand
before a learner is chosen or a label seen, so that a choice made after
seeing the data costs nothing (\Cref{thm:twosamplefull}(ii)). Against the
covering certificates on the same halves it is the smaller by far: the
cell certificate of \Cref{prop:sampleshells}, optimized on the data, is
$0.68\kappa$ on the ball, twenty times the matching, and at or above
$\kappa$ on every graded kernel at every $n$ (\Cref{app:flocking}).

\section{Conclusions and Future Work}\label{sec:conclusion}

We defined graded sparse mechanisms by axioms, showed that they are the
Lipschitz ball of a transport metric on weighted configurations, and used
that geometry for rates from the design's sparsity profile and for a
certificate computed from the sample by a matching, with non-vacuous
bounds on the flock. So although the class, the full Lipschitz ball
of the lift, is prohibitively large for any data-independent bound, it is
learnable, and the design alone can bound its rate.

We pose two questions. The first is the learning of the selector.
Although in some situations it is dictated fully by the environment, and
this is the setting we assume here, in other situations some aspects of
it can be learned
\citep{heras2019deep,ha2021unraveling,kipf2018neural}. What makes the
problem interesting is its separation: the selector is the environment's,
one rule common to all agents, while the mechanisms of the agents may
differ.
Second, empirically in this paper we concentrated on various rules on the
flock, with the conservative cap. We believe the next natural step would
be to evaluate additional scenarios, in terms of the caps, the observed
sparsity profiles, and the implied empirical complexity.

\bibliography{refs}
\bibliographystyle{plainnat}

\clearpage
\appendix
\section{Proofs for \Cref{sec:rep}}\label{app:repproofs}

\subsection{Basic lemmas, and the weight axiom}\label{app:basic}

Two remarks on \Cref{def:class} come first.

\begin{remark}[Differential form]\label{rem:derivative}
On an interval $\cV = [-R, R]$ and for $f$ differentiable, (V) is the bound
$\abs{\partial f/\partial x_i} \le L\, w_i(\mu)$ at every configuration, and
for the weight $w = \mu$, (W) is the bound
$\max_i \partial f/\partial\mu_i - \min_i \partial f/\partial\mu_i \le 2LR$
along the simplex, since a tangent direction $h$ of the simplex has
$\sum_i h_i = 0$ and $\sup_{\norm{h}_1 \le 1}\langle \nabla_\mu f, h\rangle$
is half the spread of the partial derivatives.
\end{remark}

\begin{remark}[Distinguished inputs]\label{rem:slots}
A mechanism may read, besides the anonymous coordinates, inputs of its own
that no permutation moves: the state of the agent that evaluates it, or of
its environment, as in \Cref{sec:examples}. Such an input is one
distinguished variable $q$ in a metric space, and the class becomes the set
of $f(q;\, x, \mu)$ that lie in $\cGw_L$ for every $q$ and are Lipschitz in
$q$. Everything in the paper holds fiberwise in $q$, the metric of
\Cref{sec:representation} becomes a product metric, and the covering numbers
of \Cref{app:strata} multiply by that of the space of $q$: with the metric
$d_q(q, q') + \Mw(P, P')$ the class is its $L$-Lipschitz ball, the covering
number at $\varepsilon$ is at most the product of the two covering numbers
at $\varepsilon/2$, and the diameter of \Cref{thm:twosample} is the sum of
the two. We keep the notation without the slot.
\end{remark}

\begin{lemma}[Blindness, factorization, telescope]\label{lem:basic}
Let $f \in \cGw_L$.
\begin{enumerate}
\item[(i)] If $\mu_i = 0$ then $f(x, \mu)$ does not depend on $x_i$.
\item[(ii)] If the multisets $\{(x_i, \mu_i) : \mu_i > 0\}$ and
$\{(y_j, \nu_j) : \nu_j > 0\}$ coincide, then $f(x, \mu) = f(y, \nu)$.
\item[(iii)] For all $x, y \in \cV^d$ and $\mu$:
$\abs{f(x, \mu) - f(y, \mu)} \le L \sum_i w_i(\mu)\, \rho(x_i, y_i)$.
\end{enumerate}
\end{lemma}

\begin{proof}
(i) is (V) with $w_i(\mu) = 0$. (ii) Choose a permutation $\sigma$ matching
the positive-weight coordinates of $(y,\nu)$ to those of $(x,\mu)$ pairwise.
Then $(\sigma y, \sigma\nu)$ and $(x,\mu)$ agree on every positive-weight
coordinate and on the weight vector, and differ at most in values at
zero-weight coordinates, which are free by (i). Conclude by (A). (iii) Change
the coordinates one at a time at fixed $\mu$. Each step is a (V) move and
every intermediate point is a configuration.
\end{proof}

\paragraph{Where the weight axiom comes from.} The principle, one mechanism
applied to every weighted coordinate, has a formula: a \emph{score}
$\varphi : \cV \to \R$ applied to every value and summed with the cap
weights. Let $\cS := \{\varphi : \cV \to \R \;:\; \Lip(\varphi) \le 1,\
\abs{\varphi} \le R\}$ be the score ball, which contains a translate of every
$1$-Lipschitz score because $R = \frac12\diam\cV$, and let
\begin{equation}\label{eq:linearsums}
  \cL^{w} \;:=\; \Bigl\{\, f_\varphi(x, \mu) = \sum_{i=1}^d w_i(\mu)\,
  \varphi(x_i) \;:\; \varphi \in \cS \,\Bigr\}
\end{equation}
be the \emph{linear sums} of the cap weight. They contain the weighted means
and value aggregations $\sum_i \mu_i\varphi(x_i)$ at $w = \mu$, the
unnormalized sums $\sum_i (\mu_i/\mumax)\varphi(x_i)$ at $w = \mu/\mumax$
and, with $\varphi \equiv R$, the count $R/\mumax$. Every linear sum
satisfies (V) with $L = 1$ and (A) whatever the weight axiom
(\Cref{lem:familycap}), and the class is meant to contain them at constant
$1$. The weight axiom is the tightest one that does.

\begin{proposition}[The modulus of the linear sums]\label{prop:modulus}
For all $\mu, \nu \in \simplex$,
\[
  \sup_{x \in \cV^d,\ \varphi \in \cS}\,
  \abs{\, f_\varphi(x, \mu) - f_\varphi(x, \nu) \,}
  \;=\; R\, \norm{w(\mu) - w(\nu)}_1 .
\]
\end{proposition}

So (W) says exactly that $f$ is no more sensitive to a change of weights
than the linear sums are, uniformly over the values. Among weight axioms
$\abs{f(x,\mu) - f(x,\nu)} \le L\,c(\mu,\nu)$ with a modulus $c$ that does
not look at $x$, it is the tightest one under which every linear sum keeps
constant $1$, because any such $c$ dominates the supremum above. In this
sense $\cGw_L$ is the smallest class cut out by (V), (A) and a value-free
weight axiom that contains the linear sums in its unit ball. A modulus that
does look at $x$ is a transport distance between the weighted value
measures, and the class it cuts out is smaller and excludes the selection
readouts of the next subsection: two coordinates at the same value $R$
with equal shares, and a share $\delta$ shifted from one to the other,
leave the weighted value measure unchanged while the selected maximum
moves by $\delta R$, so no constant serves, where (W) prices the move at
$2R\delta$ (\Cref{rem:whyuniform}). Weights are individuals and not mass,
and this is why the metric of \Cref{sec:representation} is a matching
distance on multisets and not a distance between measures. What (W)
excludes is the rank readouts: the weighted median and quantiles, trimmed
means and top-$k$ selections jump when a weight crosses a threshold, two
atoms at $\pm R$ with weights $\tfrac12 \pm \delta$ moving the weighted
median by $2R$ as $\delta$ changes sign, so no constant serves; the
selected maximum, which is weighted and not ranked, stays in.

\begin{lemma}[Linear sums are members]\label{lem:familycap}
Every $f_\varphi \in \cL^w$ satisfies (V) with $L = 1$ and (A). More
generally, let $\varphi : \cV \to \R^m$ satisfy $\abs{\varphi(s) -
\varphi(t)}_2 \le \rho(s, t)$ and $\abs{\varphi}_2 \le R$, and put
$f_\varphi(x, \mu) := \sum_i w_i(\mu)\,\varphi(x_i)$. Then for all
configurations
\[
  \abs{\, f_\varphi(x,\mu) - f_\varphi(y,\nu) \,}_2 \;\le\;
  \Mw\bigl(\Cl(x,\mu), \Cl(y,\nu)\bigr),
\]
so every coordinate of $f_\varphi$, and every $1$-Lipschitz function of
$f_\varphi$, lies in $\cGw_1$.
\end{lemma}

\begin{proof}
A change of $x_i$ alone changes $f_\varphi$ by
$w_i(\mu)\abs{\varphi(x_i) - \varphi(y_i)} \le w_i(\mu)\rho(x_i,y_i)$, and
anonymity is the equivariance of $w$. For the joint bound write $P =
\Cl(x,\mu)$, $Q = \Cl(y,\nu)$, so that $f_\varphi(x,\mu) = \sum_{(s,u) \in
P} u\,\varphi(s)$, and let $\sigma$ be any partial matching. A matched pair
$(s,u) \mapsto (t,v)$ contributes $\abs{u\varphi(s) - v\varphi(t)}_2 \le
\min(u,v)\abs{\varphi(s) - \varphi(t)}_2 + \abs{u - v}\sup\abs{\varphi}_2
\le \vr((s,u),(t,v))$, and an unmatched atom contributes $\abs{u\varphi(s)}_2
\le Ru = \vr((s,u), o)$. Summing gives the cost of $\sigma$, and minimizing
over $\sigma$ gives $\Mw(P,Q)$. The last claim is \Cref{thm:rep}(i), or
directly: a coordinate of $f_\varphi$ satisfies (V) and (A) as above and (W)
because $\abs{\sum_i (w_i(\mu) - w_i(\nu))\varphi(x_i)}_2 \le
R\norm{w(\mu) - w(\nu)}_1$.
\end{proof}

\begin{proof}[Proof of \Cref{prop:modulus}]
Upper bound: $\abs{\sum_i (w_i(\mu) - w_i(\nu))\varphi(x_i)} \le
\norm{w(\mu) - w(\nu)}_1 \sup\abs{\varphi} \le R\norm{w(\mu) - w(\nu)}_1$.
Lower bound: by compactness pick $a, b \in \cV$ with $\rho(a,b) = 2R$, and
let $\varphi(z) := R - \rho(z, a)$, which is $1$-Lipschitz with $\varphi(a) =
R$, $\varphi(b) = -R$ and $\abs{\varphi} \le R$, so $\varphi \in \cS$. Put
$x_i = a$ where $w_i(\mu) > w_i(\nu)$ and $x_i = b$ where $w_i(\mu) <
w_i(\nu)$, anywhere else otherwise. Then $\sum_i (w_i(\mu) -
w_i(\nu))\varphi(x_i) = R\sum_i \abs{w_i(\mu) - w_i(\nu)}$.
\end{proof}

\begin{remark}[The pointwise alternative]\label{rem:whyuniform}
One could instead take the modulus at the given $x$, $c_x(\mu,\nu) :=
\sup_{\varphi \in \cS}\abs{f_\varphi(x,\mu) - f_\varphi(x,\nu)}$, which by
Kantorovich--Rubinstein duality on $\cV$ is a transport distance between the
weighted value measures $\sum_i w_i(\mu)\delta_{x_i}$ and $\sum_i
w_i(\nu)\delta_{x_i}$: the balanced Kantorovich distance for the conservative
weight and the flat distance for the accumulative one. The resulting axiom
is stronger and its class smaller, and under it a reweighting among
coordinates at equal values is free. That is the right axiom when the
members are functionals of the weighted value measure, averages and
attention pooling, and the wrong one for the selection readouts. Take $\cV =
[-R, R]$ and $f = \max_i s_i x_i$, with a leader at value $0$ and share $1$
and two coordinates at value $R$ with equal shares $s$, and shift a share
$\delta$ from one of the two to the other. The weighted value measure does
not change at all, while $f$ changes by $\delta R$, so no constant serves.
Under (W) the move is priced at $2R\delta$ and the maximum has constant $1$
(\Cref{prop:roster}).
\end{remark}

\subsection{The cone and the matching distance}\label{app:cone}

\begin{lemma}[The cone is a compact metric space]\label{lem:cone}
$\vr$ is a metric on $\cK$, $(\cK, \vr)$ is compact, and the distance to the
apex is $R$ times the height.
\end{lemma}

\begin{proof}
$\vr$ is well defined on the quotient, since $\vr((x,u),(y,0)) = Ru$ does not
depend on $y$, it is symmetric, and it vanishes only on the diagonal: for
$u, v > 0$, $\vr = 0$ forces $u = v$ and then $\rho(x,y) = 0$. Triangle
inequality: let $p = (x,u)$, $q = (z,t)$, $r = (y,v)$ with positive heights.
We show $\vr(p,r) \le \vr(p,q) + \vr(q,r)$ by the position of $t$. If $t \ge
u \vee v$, the right side is $u\rho(x,z) + v\rho(z,y) + R(2t - u - v) \ge
(u\wedge v)\rho(x,y) + R\abs{u-v}$, since $2t - u - v \ge \abs{u - v}$. If $t
\le u \wedge v$, the right side is $t(\rho(x,z) + \rho(z,y)) + R(u + v - 2t)
\ge t\rho(x,y) + R(u + v - 2t)$, and subtracting $\vr(p,r) = (u\wedge
v)\rho(x,y) + R(u\vee v - u\wedge v)$ leaves $(u\wedge v - t)(2R - \rho(x,y))
\ge 0$ because $\rho \le 2R$. If $u \le t \le v$, the other order being
symmetric, the right side is $u\rho(x,z) + t\rho(z,y) + R(v - u) \ge
u(\rho(x,z) + \rho(z,y)) + R(v-u) \ge \vr(p,r)$. With the apex in the middle,
$\vr(p,o) + \vr(o,r) = R(u+v) \ge (u\wedge v)\rho(x,y) + R\abs{u-v}$ because
$(u \wedge v)\rho(x,y) \le 2R(u\wedge v)$, and with the apex as an endpoint
$\vr(o,r) = Rv \le Rt + R\abs{t - v} \le \vr(o,q) + \vr(q,r)$. Compactness:
$\vr((x,u),(y,v)) \le \rho(x,y) + R\abs{u-v}$ and $\vr((x,u),o) = Ru$, so the
quotient map from the compact $\cV \times [0,1]$ is continuous into
$(\cK,\vr)$ and onto.
\end{proof}

\begin{lemma}[Factorization]\label{lem:factor}
$\Cl(x,\mu) = \Cl(y,\nu)$ if and only if the multisets $\{(x_i,\mu_i) :
\mu_i > 0\}$ and $\{(y_j,\nu_j) : \nu_j > 0\}$ coincide. Consequently every
$f \in \cGw_L$ factors as $f = F \circ \Cl$ for a unique $F : \cM_d \to
\R$, and $\Cl$ is a bijection from the quotient of $\cC$ by permutations and
by values at zero-weight coordinates onto $\cM_d$.
\end{lemma}

\begin{proof}
If the multisets of $(x_i,\mu_i)$ coincide, so do those of $(x_i, w_i(\mu))$,
by equivariance. Conversely, if the multisets of pairs (value, cap weight)
coincide, there is a permutation $\pi$ of $[d]$ carrying $\supp\nu$ onto
$\supp\mu$ with $(y_{\pi^{-1}(i)}, w_{\pi^{-1}(i)}(\nu)) = (x_i, w_i(\mu))$
for $i \in \supp\mu$, so $\pi y$ agrees with $x$ on $\supp\mu$ and
$w(\pi\nu) = w(\mu)$ by equivariance. Then $\pi\nu = \mu$ by injectivity.
The factorization is \Cref{lem:basic}(ii).
\end{proof}

\begin{lemma}[Metric, measure form, duality]\label{lem:metric}
\begin{enumerate}
\item[(i)] $\Mw$ is a metric on $\cM_d$.
\item[(ii)] $\Mw(P,Q) = \Wone^{\vr}\bigl(P + n_Q\,\delta_o,\; Q + n_P\,\delta_o
\bigr)$, the Kantorovich distance with ground cost $\vr$ between the two
apex-padded measures, both of total mass $n_P + n_Q$.
\item[(iii)] With $\Lipo(\cK) := \{\psi : \cK \to \R \;:\; \Lip_{\vr}(\psi)
\le 1,\ \psi(o) = 0\}$,
\[
  \Mw(P, Q) \;=\; \sup_{\psi \in \Lipo(\cK)}\;
  \Bigl[\, \sum_{(x,u) \in P} \psi(x,u) - \sum_{(y,v) \in Q} \psi(y,v)
  \,\Bigr],
\]
and every $\psi \in \Lipo(\cK)$ satisfies $\abs{\psi(x,u)} \le Ru$.
\end{enumerate}
\end{lemma}

\begin{proof}
(i) Symmetry is clear. If $\Mw(P,Q) = 0$ the optimal $\sigma$ has zero cost,
so no atom is unmatched, an unmatched atom costing $Ru > 0$, and every
matched pair coincides, hence $P = Q$. Triangle inequality: given partial
matchings $\sigma_1$ from $P$ to $Q$ and $\sigma_2$ from $Q$ to $S$, define
$\sigma$ from $P$ to $S$ by $\sigma(p) = \sigma_2(\sigma_1(p))$ whenever both
are defined, and leave every other atom unmatched. By the triangle inequality
of $\vr$ with the apex included (\Cref{lem:cone}): a $p$ matched through $q$
to $s$ costs $\vr(p,s) \le \vr(p,q) + \vr(q,s)$, a $p$ matched to a $q$ that
$\sigma_2$ leaves unmatched costs $\vr(p,o) \le \vr(p,q) + \vr(q,o)$, an $s$
matched by $\sigma_2$ to a $q$ that $\sigma_1$ leaves unmatched costs
$\vr(s,o) \le \vr(q,s) + \vr(q,o)$, and atoms unmatched in $\sigma_1$ or
$\sigma_2$ keep their apex cost. Each term of $\cost(\sigma_1) +
\cost(\sigma_2)$ is used at most once, so $\cost(\sigma) \le \cost(\sigma_1)
+ \cost(\sigma_2)$.
(ii) The Kantorovich problem between the padded measures is a transportation
linear program between two counting measures of the same mass, $n_P + n_Q$
unit atoms on each side, so by Birkhoff's theorem its polytope is the convex
hull of the permutation matrices \citep[Prop.~2.1]{peyre2019computational}
and an optimal integer plan exists, a bijection between the atoms of the
two sides.
A bijection matches some atoms of $P$ to atoms of $Q$, the remaining atoms of
each side to apex copies, and apex copies to each other at cost $0$, and its
cost is the cost of the partial matching it induces. Conversely every partial
matching extends to such a bijection, the counts of unmatched atoms and spare
apex copies agreeing on both sides. Hence the two minima coincide.
(iii) Kantorovich--Rubinstein duality on the compact metric space $(\cK,
\vr)$, for two finite measures of equal mass: $\Wone(\alpha,\beta) =
\sup\{\int\psi\,d(\alpha-\beta) : \Lip_\vr\psi \le 1\}$. Here $\alpha -
\beta = P - Q + (n_Q - n_P)\delta_o$ has total mass $0$, so $\psi \mapsto
\psi - \psi(o)$ does not change the integral and one may restrict to
$\psi(o) = 0$, which kills the apex term. The bound is $\abs{\psi(x,u)} =
\abs{\psi(x,u) - \psi(o)} \le \vr((x,u),o) = Ru$.
\end{proof}

\begin{remark}[The dual ball]\label{rem:dualball}
For $\varphi \in \cS$ the function $\psi_\varphi(x,u) := u\,\varphi(x)$ lies
in $\Lipo(\cK)$ and $\int\psi_\varphi\, d\Cl(x,\mu) = f_\varphi(x,\mu)$, so
the linear sums are the sub-ball of scores that read the height linearly.
The dual ball is larger: it contains scores that read the height
nonlinearly, such as $\psi(x,u) = \tfrac{R}{2}u^2$, whose functional is
$\tfrac R2\sum_i w_i(\mu)^2$. This is where the members beyond the linear
sums live: the supremum over $x$ in \Cref{prop:modulus} enlarged the score
ball from $\{u\varphi(x)\}$ to all Lipschitz functions on the cone vanishing
at the apex, and $\Mw$ is the dual of the enlarged ball.
\end{remark}

\begin{lemma}[Elementary moves]\label{lem:moves}
Let $(x,\mu) \in \cC$.
\begin{enumerate}
\item[(i)] (value move, exact) If $y$ differs from $x$ in coordinate $i$
only, then $\Mw\bigl(\Cl(x,\mu), \Cl(y,\mu)\bigr) = w_i(\mu)\,\rho(x_i, y_i)$.
\item[(ii)] (weight move) For every $\nu$, $\Mw\bigl(\Cl(x,\mu), \Cl(x,\nu)
\bigr) \le R\,\norm{w(\mu) - w(\nu)}_1$, with strict inequality possible:
two atoms at the same value exchanging heights are at distance $0$ while
$\norm{w(\mu) - w(\nu)}_1 > 0$.
\item[(iii)] (re-representation) $\Cl(\sigma x, \sigma\mu) = \Cl(x,\mu)$,
and changing $x_i$ with $\mu_i = 0$ does not change $\Cl$.
\end{enumerate}
\end{lemma}

\begin{proof}
(i) The identity matching gives $\le \vr((x_i,w_i),(y_i,w_i)) =
w_i\rho(x_i,y_i)$. For $\ge$, take $\psi(z,u) := u\,(R - \rho(z, x_i))$,
which is $\psi_\varphi$ for $\varphi = R - \rho(\cdot, x_i) \in \cS$, hence
in $\Lipo(\cK)$ by \Cref{rem:dualball}. The two lifts differ in one atom and
$\psi(x_i,w_i) - \psi(y_i,w_i) = w_i\rho(x_i,y_i)$, so \Cref{lem:metric}(iii)
gives the lower bound. (ii) Match coordinate $i$ to itself where both weights
are positive, at cost $R\abs{w_i(\mu) - w_i(\nu)}$ with no value motion, and
send an atom with a zero partner to the apex, at cost $Rw_i$, again
$R\abs{w_i(\mu) - w_i(\nu)}$. The example: $\{(t,u),(t,v)\}$ and
$\{(t,v),(t,u)\}$ are the same multiset. (iii) is the definition of the lift.
\end{proof}

\begin{lemma}[Continuity, compactness, diameter]\label{lem:compact}
\begin{enumerate}
\item[(i)] For all $(x,\mu), (y,\nu) \in \cC$,
\[
  \Mw\bigl(\Cl(x,\mu), \Cl(y,\nu)\bigr) \;\le\; \abs{w(\mu)}\,\max_i\rho(x_i,y_i) +
  R\,\norm{w(\mu) - w(\nu)}_1 ,
\]
where $\abs{w(\mu)} := \sum_i w_i(\mu)$. In particular $\Cl : \cC \to
(\cM_d, \Mw)$ is continuous.
\item[(ii)] $(\cM_d, \Mw)$ is compact and $\Cl$ is a quotient map: a
function on $\cM_d$ is continuous if and only if its composition with $\Cl$
is. In particular the tail masses $T_k(\mu) = \sum_{i > k}\mu_{(i)}$, with
the weights sorted decreasingly, and $\mumax$ are continuous on $(\cM_d,
\Mw)$.
\item[(iii)] $\Mw(P,Q) \le R\,(\abs{P} + \abs{Q})$ with $\abs{P} := \int
u\,dP(x,u)$ the total height, hence $\diam(\cM_d, \Mw) \le 2R\sup_\mu
\abs{w(\mu)}$.
\end{enumerate}
\end{lemma}

\begin{proof}
(i) Go through $\Cl(y,\mu)$: the first leg costs at most $\sum_i
w_i(\mu)\rho(x_i,y_i)$ by the identity matching, the second at most
$R\norm{w(\mu) - w(\nu)}_1$ by \Cref{lem:moves}(ii). Continuity follows with
the continuity of $w$. (ii) $\cC$ is compact and $\Cl$ continuous and onto,
so $(\cM_d, \Mw)$ is compact. A continuous map from a compact space onto a
Hausdorff space is closed, hence a quotient map. $T_k$ and $\mumax$ are
continuous on $\cC$ and constant on the fibers of $\Cl$. (iii) The empty
matching sends every atom to the apex.
\end{proof}

\subsection{Proof of \Cref{thm:rep}}\label{app:thmproof}

\begin{proof}[Proof of \Cref{thm:rep}(i), the easy direction]
If $\abs{F(P) - F(Q)} \le L\,\Mw(P,Q)$ for all $P, Q$ and $f = F \circ
\Cl$, then (V), (W) and (A) are \Cref{lem:moves}(i), (ii), (iii)
respectively, read through the Lipschitz bound.
\end{proof}

\begin{proof}[Proof of \Cref{thm:rep}(i), the routing direction]
Let $f \in \cGw_L$. By \Cref{lem:factor}, $f = F \circ \Cl$ for a unique
$F$, and it remains to show that for all $(x,\mu), (y,\nu) \in \cC$,
\[
  \abs{\, f(x,\mu) - f(y,\nu) \,} \;\le\; L \cdot \Mw\bigl(\Cl(x,\mu),\,
  \Cl(y,\nu)\bigr).
\]
Write $P = \Cl(x,\mu)$, $Q = \Cl(y,\nu)$, $S = \supp\mu$, $T = \supp\nu$,
and let $\sigma$ be an optimal partial matching, viewed at the level of
coordinates as a bijection from $S' \subseteq S$ onto $T' \subseteq T$. We
may take $\sigma$ \emph{maximal}: if $i \in S$ and $j \in T$ were both
unmatched, adding the pair $(i,j)$ would change the cost by
$\vr((x_i,w_i(\mu)),(y_j,w_j(\nu))) - \vr((x_i,w_i(\mu)),o) -
\vr((y_j,w_j(\nu)),o) \le 0$, by the triangle inequality through the apex.
So either $S' = S$ or $T' = T$.

\emph{Relabeling.} Define an injection $t$ of $T$ into $[d]$ by
$t(\sigma(i)) := i$ for $i \in S'$, and, if $T \setminus T' \ne \emptyset$
(so $S' = S$), by sending $T \setminus T'$ injectively into $[d] \setminus
S$, which has $d - \abs{S} \ge \abs{T} - \abs{T'}$ elements. Let $\pi$ be a
permutation of $[d]$ with $\pi(j) = t(j)$ for $j \in T$, and put $\nu' :=
\pi\nu$, $y' := \pi y$. Then $\nu'_{t(j)} = \nu_j$, $y'_{t(j)} = y_j$, $\nu'$
vanishes off $t(T)$, and by equivariance $w_{t(j)}(\nu') = w_j(\nu)$. By
(A), $f(y',\nu') = f(y,\nu)$, and $\Cl(y',\nu') = Q$. Every coordinate $i
\in [d]$ is now of one of four kinds: \emph{matched} ($i \in S'$, hosting
$(x_i, w_i(\mu))$ and $(y'_i, w_i(\nu')) = (y_{\sigma(i)},
w_{\sigma(i)}(\nu))$), \emph{leaving} ($i \in S \setminus S'$, weight
$w_i(\mu) > 0 = w_i(\nu')$), \emph{arriving} ($i \in t(T\setminus T')$,
weight $0 = w_i(\mu) < w_i(\nu')$), or \emph{empty}. By maximality, leaving
and arriving coordinates do not both occur.

\emph{The chain.} Split the matched coordinates into $I^{\uparrow} := \{i
\in S' : w_i(\mu) \le w_i(\nu')\}$ and $I^{\downarrow} := S' \setminus
I^{\uparrow}$. Define $z \in \cV^d$ by $z_i := y'_i$ for $i \in I^{\uparrow}$
and for arriving $i$, and $z_i := x_i$ otherwise. Then:
\begin{enumerate}
\item[Step 1.] $(x,\mu) \to (z,\mu)$: value moves at weights $\mu$ on
$I^{\uparrow}$ and on the arriving coordinates. By \Cref{lem:basic}(iii),
$\abs{f(x,\mu) - f(z,\mu)} \le L\sum_{i \in I^\uparrow} w_i(\mu)\,\rho(x_i,
y'_i)$, the arriving coordinates having $w_i(\mu) = 0$ and contributing
nothing.
\item[Step 2.] $(z,\mu) \to (z,\nu')$: one weight move. By (W),
$\abs{f(z,\mu) - f(z,\nu')} \le LR\,\norm{w(\mu) - w(\nu')}_1$.
\item[Step 3.] $(z,\nu') \to (y',\nu')$: value moves at weights $\nu'$ on
$I^{\downarrow}$ (on the leaving coordinates $y'$ may be taken equal to $x$,
since their $\nu'$-weight is zero and the lift does not see them). By
\Cref{lem:basic}(iii), $\abs{f(z,\nu') - f(y',\nu')} \le L\sum_{i \in
I^\downarrow} w_i(\nu')\,\rho(x_i, y'_i)$.
\end{enumerate}
On $I^\uparrow$ the weight in Step~1 is $w_i(\mu) = \min(w_i(\mu),
w_i(\nu'))$, and on $I^\downarrow$ the weight in Step~3 is $w_i(\nu') =
\min(w_i(\mu), w_i(\nu'))$, so the value costs of Steps~1 and~3 add up to
$\sum_{i \in S'} \min\bigl(w_i(\mu), w_{\sigma(i)}(\nu)\bigr)\,\rho(x_i,
y_{\sigma(i)})$, while the $\ell^1$ sum in Step~2 is $\sum_{i\in
S'}\abs{w_i(\mu) - w_{\sigma(i)}(\nu)}$ plus the weights $w_i(\mu)$ of the
leaving coordinates plus the weights $w_{t^{-1}(i)}(\nu)$ of the arriving
ones. Adding the three steps,
\[
  \abs{f(x,\mu) - f(y,\nu)} \;\le\; L\Bigl[\sum_{i \in S'}
  \vr\bigl((x_i,w_i(\mu)), (y_{\sigma(i)}, w_{\sigma(i)}(\nu))\bigr)
  + \sum_{\text{unmatched}} R\,w \Bigr] \;=\; L\,\cost(\sigma),
\]
which is $L\,\Mw(P,Q)$ by the optimality of $\sigma$.
\end{proof}

\begin{proof}[Proof of \Cref{thm:rep}(ii)]
The inequality $\le$ is the routing direction. For $\ge$ take $F :=
\Mw(\cdot, Q)$, which is $1$-Lipschitz for $\Mw$, so $F \circ \Cl \in
\cGw_1$ by the easy direction, and $F(P) - F(Q) = \Mw(P,Q)$. If a metric $d$
on $\cM_d$ makes every $F$ with $F \circ \Cl \in \cGw_1$ $1$-Lipschitz, then
$\Mw(P,Q) = \sup_f [f(z) - f(z')] \le d(P,Q)$. And if $\Lip_L(d)\circ\Cl =
\cGw_L$ for all $L$, the same argument with $d$ in place of $\Mw$ gives
$d(P,Q) = \sup_f[f(z) - f(z')] = \Mw(P,Q)$.
\end{proof}

In the language of persistence diagrams, part (ii) is the universality of
the Wasserstein distance on the metric pair $(\cK, \vr, \{o\})$:
\citet[Cor.~1.4 and \S5.3]{bubenik2022universality} prove that on any
metric space with a distinguished subset the $p$-Wasserstein distance on
finite multisets is the largest $p$-subadditive metric compatible with the
ground metric, and that at $p = 1$ it satisfies Kantorovich--Rubinstein
duality, which is \Cref{lem:metric}(iii), and \citet{bubenik2022virtual}
carry the construction to signed measures and their Banach spaces. What is specific here is the
pair, the cone with its apex, and that the class comes first.

The closure properties named after \Cref{thm:rep} are those of a Lipschitz
ball: the class is convex, and if $f_1, \dots, f_m$ are members with
constants $L_1, \dots, L_m$ and $h : \R^m \to \R$ is Lipschitz then $h(f_1,
\dots, f_m)$ is a member with constant $\sup_y \sum_a \abs{\partial_a
h(y)}\, L_a$, since $h(F_1, \dots, F_m)$ is Lipschitz for $\Mw$ with that
constant. The linear sums extend to the vector-valued case, $\sum_i
w_i(\mu)\varphi(x_i)$ with $\varphi : \cV \to \R^m$ $1$-Lipschitz and
bounded by $R$ in norm, which is $1$-Lipschitz from $(\cM_d, \Mw)$ into
$\R^m$ jointly (\Cref{lem:familycap}), so that attention pooling and the
flocking resultant are members by post-composition.

\subsection{The two instances}\label{app:instances}

\paragraph{Three levels, and the sum embeddings.} Three nested sets order
the class: the linear sums \eqref{eq:linearsums}, of which a weighted Deep
Sets network is a Lipschitz readout of finitely many
(\Cref{rem:gnncurrency}), the functionals of the lift, the dual ball of
the matching distance (\Cref{lem:metric}, \Cref{rem:dualball}), and the
class of \Cref{thm:rep}, every Lipschitz function of the lift. The
conservative linear sums learn at the rate of the value space alone, with
no count and no profile (\Cref{prop:countfree}), while the class sees the
law of the lift and pays the profile (\Cref{sec:complexity}).

The two classes of \Cref{sec:axioms} are the following.

\begin{definition}[The two classes]\label{def:twoclasses}
The \emph{conservative} class $\cG_L := \cG^{w^{\mathrm c}}_L$ is the class
of \Cref{def:class} for the cap weight $w^{\mathrm c}(\mu) = \mu$, the
weights themselves, and the \emph{accumulative} class $\cGt_L :=
\cG^{w^{\mathrm a}}_L$ is the class for the cap weight $w^{\mathrm a}(\mu)
= \mu/\mumax$, the vector of \emph{shares} $s(\mu)$, the leader at share
$1$.
\end{definition}

\Cref{tab:mechanisms} sorts the common aggregations by what they do to
influence. The same aggregations are debated in graph networks on
expressivity grounds, the sum separating multisets that the mean and the
maximum identify \citep{xu2019powerful} and a degree-dependent scaler
interpolating between the two \citep{corso2020principal}. In both classes
the weight bound sums over all $d$ coordinates, and $\norm{s(\mu) -
s(\nu)}_1$ can reach $d - 1$: that is the count of \Cref{tab:mechanisms}, a
member at constant $1$, and the size is intended.

\begin{table}[h]
\centering
\small
\begin{tabular}{@{}lllc@{}}
\toprule
mechanism & form & class & $L$ \\
\midrule
weighted mean & $h\bigl(\sum_i \mu_i x_i\bigr)$ & conservative & $\norm{h'}_\infty$ \\
value aggregation & $\sum_i \mu_i\,\varphi(x_i)$, $\varphi \in \cS$ & conservative & $1$ \\
attention pooling & $h\bigl(\sum_i \mu_i\, v(x_i)\bigr)$, $\Lip v \le 1$, $\abs{v} \le R$ & conservative & $\Lip(h)$ \\
top weight & $R\,\mumax$ & conservative & $1$ \\
flocking update & direction of $\sum_i \mu_i\, e^{\mathrm{i}\theta_i}$ & conservative & $1/\varpi_0$ at polarization $\varpi \ge \varpi_0$ \\
\midrule
count & $R/\mumax = R\sum_i s_i$ & accumulative & $1$ \\
unnormalized sum & $\sum_i s_i\,\varphi(x_i)$, $\varphi \in \cS$ & accumulative & $1$ \\
force sum & $\sum_i s_i\, F(x_i)$, $\Lip F \le 1$, $\abs{F} \le R$ & accumulative & $1$ \\
selected maximum & $\max_i s_i\, x_i$ & accumulative & $1$ \\
weighted mean$^{*}$ & $\sum_i \mu_i x_i$ & accumulative & $2$ \\
\bottomrule
\end{tabular}
\caption{Common aggregations, the class they belong to and the constant.
Values in $[-R, R]$ where a value bound is used.
The starred row is the mechanism of the first row read against the share cap.
The flocking row is \Cref{app:flocking}, the rest is
\Cref{prop:roster}. No constant depends on $d$ or on the number of selected
coordinates.}
\label{tab:mechanisms}
\end{table}

\begin{proposition}[Membership]\label{prop:roster}
Each mechanism of \Cref{tab:mechanisms} belongs to the class shown with the
constant shown. Moreover $\cG_L \subseteq \cGt_{2L}$.
\end{proposition}

\begin{proof}[Proof of \Cref{prop:roster}]
Each (V) line is the differential form of \Cref{rem:derivative} or a direct
two-point estimate, and each (W) line is the displayed one-line estimate.
\emph{Conservative rows}, on $\cV = [-R, R]$. Weighted means $h(\langle\mu,
x\rangle)$: (V) $\abs{h'}\mu_i\abs{x_i - y_i}$, (W) $\abs{h'}\,\abs{\langle
\mu - \nu, x\rangle} \le \norm{h'}_\infty R\norm{\mu - \nu}_1$. Value
aggregations: \Cref{lem:familycap}. Attention pooling: the vector-valued
linear sum of \Cref{lem:familycap} followed by the Lipschitz map $h$. The
top weight $R\,\mumax$: (V) $0$, (W) $R\abs{\mumax - \nu_{\max}} \le
R\norm{\mu - \nu}_\infty$. The flocking update is \Cref{app:flocking}.
\emph{Accumulative rows}, on $\cV = [-R, R]$, with $s = s(\mu)$, $s' = s(\nu)$ and $\abs{s} := \norm{s}_1$. The
count $R\sum_i s_i$: (V) $0$, (W) $R\,\abs{\abs{s} - \abs{s'}} \le R\norm{s
- s'}_1$. The sums and the force sums: \Cref{lem:familycap} with the share
weight. The selected maximum $\max_i s_ix_i$, on values in $[0, R]$ so that the coordinates off the support, at $s_ix_i = 0$, never attain it: (V) $\abs{s_ix_i - s_iy_i} \le s_i\abs{x_i -
y_i}$, (W) $\abs{\max_i s_ix_i - \max_i s'_ix_i} \le \max_i\abs{s_i -
s'_i}\,\abs{x_i} \le R\norm{s - s'}_\infty$. The weighted mean $\langle\mu,
x\rangle = \langle s, x\rangle/\abs{s}$: (V) $\mu_i \le s_i$, and for (W),
$\langle\mu,x\rangle - \langle\nu,x\rangle = \langle s - s', x\rangle/\abs{s}
+ \langle s',x\rangle\bigl(1/\abs{s} - 1/\abs{s'}\bigr)$ with
$\abs{\langle s',x\rangle} \le R\abs{s'}$ and $\abs{1/\abs{s} - 1/\abs{s'}}
\le \norm{s - s'}_1/(\abs{s}\abs{s'})$, so the difference is at most
$2R\norm{s-s'}_1/\abs{s} \le 2R\norm{s - s'}_1$ because $\abs{s} \ge 1$.
\emph{Containment.} $\mu_i \le s_i$ gives the accumulative (V) from the
conservative one, and for the weight axioms $\norm{\mu - \nu}_1 \le
2\,\mumax\,\norm{s - s'}_1 \le 2\norm{s-s'}_1$: indeed $\mu = s/\abs{s}$, so
$\mu - \nu = (s - s')/\abs{s} + s'(1/\abs{s} - 1/\abs{s'})$, of $\ell^1$ norm
at most $(\norm{s-s'}_1 + \abs{\abs{s'} - \abs{s}})/\abs{s} \le
2\norm{s-s'}_1/\abs{s}$, and $1/\abs{s} = \mumax$. Hence $\cG_L \subseteq
\cGt_{2L}$, and by \Cref{thm:rep}(ii) also $\Dl \le 2\,\Dt$ on $\cM_d$.
\end{proof}

The corollary read in \Cref{sec:representation} is the following. Reading
a sum as conservative inflates its constant by the count, the price of the
semantics.

\begin{corollary}[Fibers: the two semantics made exact]\label{cor:fiber}
Let $\mu_S$ be uniform on $S \subseteq [d]$ with $\abs{S} = k$. On the
fiber $\{\mu = \mu_S\}$ the accumulative distance is the matching distance
between the two $k$-point multisets of values,
\[
  \Dt\bigl(\Cl(x,\mu_S), \Cl(y,\mu_S)\bigr) \;=\;
  \min_{\tau}\, \sum_{i \in S} \rho\bigl(x_i, y_{\tau(i)}\bigr)
\]
over the bijections $\tau$ of $S$, and the restriction of $\cGt_L$ to the
fiber is exactly the set of symmetric functions $g$ of $x_S$ with
per-coordinate constant $L$, that is $\abs{g(x_S) - g(x_S^{i \to y_i})} \le
L\,\rho(x_i, y_i)$ with $x^{i \to y_i}$ the vector $x$ with coordinate $i$
replaced by $y_i$, every such $g$ extending to a member of $\cGt_L$. The
restriction of $\cG_L$ has per-coordinate constant $L/k$, and the factor $k$
is attained by the selected maximum.
\end{corollary}

\begin{proof}[Proof of \Cref{cor:fiber}]
On the fiber all shares equal $1$, so a matched pair costs
$\vr((x_i,1),(y_j,1)) = \rho(x_i,y_j) \le 2R$ while an unmatched atom costs
$R$. Leaving a pair unmatched on both sides costs $2R \ge \rho$, so a full
matching is optimal and the formula follows. A symmetric $g$ with
per-coordinate constant $L$ satisfies $\abs{g(x_S) - g(y_S)} \le L\sum_i
\rho(x_i, y_{\tau(i)})$ for every $\tau$, by symmetry and then the telescope
of \Cref{lem:basic}(iii) with all weights $1$, hence is $L$-Lipschitz for
$\Dt$ on the fiber. Extend it to $\cM_d$ with the same constant by McShane's
extension and apply the easy direction of \Cref{thm:rep}(i). Conversely the
accumulative (V) at $s_i = 1$ gives the per-coordinate constant $L$, and the
conservative (V) at $\mu_i = 1/k$ gives $L/k$. The selected maximum has
accumulative constant $1$ (\Cref{prop:roster}) and, on the fiber, changes by
$\abs{x_i - y_i}$ under a move of the leading coordinate, so its conservative
constant is at least $k$.
\end{proof}

\subsection{Strata, coverings, and the crop cost}\label{app:strata}

These are the inputs of \Cref{sec:complexity}: for each number of atoms, a
covering of the lifts carried by that many atoms, the cost of cropping a
configuration to its heaviest atoms, and the diameter of the space.

\begin{definition}[Strata and cells]\label{def:strata}
For $k \le d$, $\cA_k := \{P \in \cM_d : n_P \le k\}$, the lifts of
configurations with at most $k$ participants. For the accumulative instance
and $m \ge 0$, $\cA_k^{[m]} := \{P \in \cA_k : \abs{P} \le 2^m\} = \cA_k \cap
\{\mumax \ge 2^{-m}\}$. Since $\abs{P} \le n_P \le k$ on $\cA_k$, only $2^m <
k$ gives a proper subset. Covering numbers $N(\varepsilon, S, \Mw)$ use the
radius convention.
\end{definition}

\begin{theorem}[Covering of the strata, conservative]\label{thm:covcons}
For every $k \le d$ and $\varepsilon > 0$,
\[
  N\bigl(\varepsilon, \cA_k, \Dl\bigr) \;\le\;
  N\bigl(\tfrac{\varepsilon}{2}, \cV, \rho\bigr)^k \cdot
  N\bigl(\tfrac{\varepsilon}{2R}, \Delta_k, \ell^1\bigr)
  \;\le\; N\bigl(\tfrac{\varepsilon}{2}, \cV, \rho\bigr)^k \Bigl(1 + \frac{16R}{\varepsilon}
  \Bigr)^{k-1}.
\]
On $\cV = [-R,R]$ and $\varepsilon \le R$: $\log N(\varepsilon, \cA_k, \Dl) \le
(2k-1)\log(17R/\varepsilon)$.
\end{theorem}

\begin{proof}
Parametrize: for $v \in \cV^k$ and $a \in \Delta_k$ let $\Psi(v,a) := \Cl(z)$
where $z$ carries values $v$ and weights $a$ on the first $k$ coordinates and
weight $0$ elsewhere. By (A) every $P \in \cA_k$ is some $\Psi(v,a)$.
Matching coordinate $j$ to coordinate $j$, with the apex where one side has
$a_j = 0$,
\begin{align*}
  \Dl\bigl(\Psi(v,a), \Psi(v',a')\bigr) \;&\le\; \sum_{j=1}^k \Bigl[\min(a_j,a'_j)\,\rho(v_j,v'_j)
  + R\abs{a_j - a'_j}\Bigr] \\
  &\le\; \max_j \rho(v_j, v'_j) + R\norm{a - a'}_1 ,
\end{align*}
using $\sum_j\min(a_j,a'_j) \le 1$. Take an $\varepsilon/2$-net of $\cV$ in
each value slot and an $\varepsilon/(2R)$-net of $(\Delta_k, \ell^1)$. Each
term is then at most $\varepsilon/2$. For the simplex count: dropping the
last coordinate maps $\Delta_k$ into the unit $\ell^1$-ball of $\R^{k-1}$
with $\norm{a-a'}_1 \le 2\norm{b - b'}_1$. A $\delta/4$-net of that ball has
at most $(1 + 8/\delta)^{k-1}$ points, a ball in its own norm, and
re-centering the net at points of the image gives a $\delta/2$-net of the
image and hence a $\delta$-net of $\Delta_k$. With $\delta =
\varepsilon/(2R)$ this is $(1 + 16R/\varepsilon)^{k-1}$. On the interval,
$N(\varepsilon/2, [-R,R]) \le 4R/\varepsilon + 1 \le 5R/\varepsilon$ for
$\varepsilon \le R$, and $k\log(5R/\varepsilon) + (k-1)\log(17R/\varepsilon)
\le (2k-1)\log(17R/\varepsilon)$.
\end{proof}

\begin{theorem}[Covering of the cells, accumulative]\label{thm:covacc}
For every $k \le d$, $m \ge 0$ and $\varepsilon > 0$,
\[
  N\bigl(\varepsilon, \cA_k^{[m]}, \Dt\bigr) \;\le\;
  N\bigl(2^{-m-1}\varepsilon, \cV, \rho\bigr)^k \cdot
  k\Bigl(1 + \frac{2^{m+3}R}{\varepsilon}\Bigr)^{k-1},
\]
with net points in $\cA_k$. Equivalently, at the cell's own scale,
$N\bigl(2^{m-1}\varepsilon, \cA_k^{[m]}, \Dt\bigr) \le N(\varepsilon/4,
\cV, \rho)^k \cdot k(1 + 16R/\varepsilon)^{k-1}$: the count of a cell
resolved at $2^{m-1}\varepsilon$ does not depend on $m$. On $\cV = [-R,R]$
and $\varepsilon \le R$: $\log N(\varepsilon, \cA_k^{[m]}, \Dt) \le
(2k-1)\log(2^{m+5}R/\varepsilon)$, and $2^m \le 2k$ suffices on $\cA_k$.
\end{theorem}

\begin{proof}
Parametrize by the share vector: let $\Sigma_k := \{s \in [0,1]^k : \max_j
s_j = 1\}$ and, for $v \in \cV^k$, $s \in \Sigma_k$, let $\widetilde\Psi(v,s)
:= \Cl(z)$ with $z$ carrying values $v$ and weights $s/\norm{s}_1$ on the
first $k$ coordinates, so that $z$ has share vector $s$ there. Every $P \in
\cA_k^{[m]}$ is $\widetilde\Psi(v,s)$ with $\norm{s}_1 = \abs{P} \le 2^m$.
Matching coordinatewise,
\begin{align*}
  \Dt\bigl(\widetilde\Psi(v,s), \widetilde\Psi(v',s')\bigr) \;&\le\;
  \sum_{j}\min(s_j,s'_j)\,\rho(v_j,v'_j) + R\norm{s - s'}_1 \\
  &\le\; 2^m\max_j\rho(v_j,v'_j) + R\norm{s-s'}_1 .
\end{align*}
Take a $2^{-m-1}\varepsilon$-net of $\cV$ in each value slot and an
$\varepsilon/(2R)$-net of the share vectors of the cell in $\ell^1$. Each
term is then at most $\varepsilon/2$. For the share count: the share vectors
of the cell lie in the union over $j \le k$ of the faces $\{s_j = 1\}$, and
on the face $\{s_j = 1\}$ the remaining coordinates form a nonnegative vector
of $\ell^1$-norm at most $2^m - 1 < 2^m$, a point of the $\ell^1$-ball of
radius $2^m$ in $\R^{k-1}$. A $\delta/2$-net of that ball has at most $(1 +
2^{m+2}/\delta)^{k-1}$ points, and re-centering at points of the face gives
a $\delta$-net of the face inside the face. With $\delta = \varepsilon/(2R)$
this is $k(1 + 2^{m+3}R/\varepsilon)^{k-1}$. The net points
$\widetilde\Psi(v', s')$ are lifts in $\cA_k$. The second display is the
first with $2^{m-1}\varepsilon$ in place of $\varepsilon$. On the interval,
$N(2^{-m-1}\varepsilon,[-R,R]) \le 2^{m+2}R/\varepsilon + 1 \le
2^{m+3}R/\varepsilon$, and $k\log(2^{m+3}R/\varepsilon) + \log k +
(k-1)\log(2^{m+4}R/\varepsilon) \le (2k-1)\log(2^{m+5}R/\varepsilon)$, using
$\log k \le (2k-1)\log 2$. Finally $\abs{P} \le k$ on $\cA_k$, so
$\cA_k^{[m]} = \cA_k$ once $2^m \ge k$.
\end{proof}

Both bounds have exponent $k\dim_\varepsilon(\cV) + (k-1)$, which is $2k-1$
on intervals, with the ambient $d$ absent and the cell factor $2^m$ inside
the logarithms only, and not at all once the cell is resolved at its own
scale. On a \emph{selector design}, the law of $(\pi(\omega), c(\omega))$ for a
selector $c$, the weights are functions of the values and the simplex
directions are not there to cover. If the cropped weights depend on the $k$
selected values $x_S$ through a map that is $L_c$-Lipschitz from $(\cV^S,
\max_i\rho)$ to $(\Delta_k, \ell^1)$, or are uniform on the support, as
for a hard selector, then by \Cref{lem:compact}(i) the
map $x_S \mapsto \Cl(x_S, \hat\mu(x_S))$ is $(1 + RL_c)$-Lipschitz into
$(\cM_d, \Dl)$, and the cropped support of the stratum is covered by
$N(\varepsilon/(2(1 + RL_c)), \cV, \rho)^k$ points: exponent
$k\dim_\varepsilon(\cV)$. Both counts are used in \Cref{sec:complexity}.

\begin{proposition}[Crop cost]\label{prop:crop}
Let $\hat\mu$ be a $k$-crop of $\mu$ for $k \ge 1$, $k$ of its largest
weights kept and renormalized, by a fixed measurable rule where weights tie
(the cost below is the same for every choice), and $T_k = T_k(\mu)$ the
weight beyond them. Then
\[
  \Dl\bigl(\Cl(x,\mu), \Cl(x,\hat\mu)\bigr) \;\le\; 2R\,T_k,
  \qquad
  \Dt\bigl(\Cl(x,\mu), \Cl(x,\hat\mu)\bigr) \;\le\; \frac{R\,T_k}{\mumax(\mu)} .
\]
\end{proposition}

\begin{proof}
Match each kept atom to itself and send the tail to the apex. Conservative:
a kept atom moves in height from $\mu_i$ to $\mu_i/(1-T_k)$, in total
$R\sum_{\text{kept}}\mu_i T_k/(1-T_k) = R\,T_k$, and the tail costs $R\,T_k$.
Accumulative: $\hat\mu_{\max} = \mumax/(1-T_k)$, so the shares of kept atoms
are unchanged and they cost nothing, and the tail costs $R\sum_{\text{tail}}
s_i = R\,T_k/\mumax$.
\end{proof}

\begin{proposition}[Qualitative inputs]\label{prop:qual}
For both instances $(\cM_d, \Mw)$ is compact, $T_k$ and $\mumax$ are
continuous on it, the shells $\{T_k \le \theta < T_{k-1}\}$ and the cells
$\{\mumax \ge 2^{-m}\}$ are Borel, and Kantorovich--Rubinstein duality holds
for $\Wone(\Mw)$ on probability measures on $\cM_d$. Diameters:
$\diam(\cM_d, \Dl) \le 2R$, on $\{\abs{P} \le 2^m\}$ one has $\Dt(P,Q) \le
2^{m+1}R$, and globally $\diam(\cM_d, \Dt) \le 2Rd$.
\end{proposition}

\begin{proof}
\Cref{lem:compact}, and duality on a compact metric space is standard
\citep{peyre2019computational}.
\end{proof}

\section{Proofs for \Cref{sec:complexity}}\label{app:complexityproofs}

\paragraph{Inputs: strata, crop and counts.} The stratum $\cA_k \subseteq \cM_d$
of lifts carried by at most $k$ atoms, the crop of a configuration to its
$k$ heaviest coordinates and the covering numbers of the strata are the
inputs, all from \Cref{app:strata}. A configuration with tail $T_k$ lies
within $\kappa T_k$ of its crop, a point of $\cA_k$ (\Cref{prop:crop}),
and, when the value space has metric dimension $\dim\cV$, $N(\varepsilon, \cV,
\rho) \le (C_\cV R/\varepsilon)^{\dim\cV}$, the strata are covered at
\[
  N_k(\varepsilon) := N\bigl(\tfrac{\varepsilon}{2}, \cA_k, \Dl\bigr)
  \;\le\; \Bigl(\frac{C_1 R}{\varepsilon}\Bigr)^{e(k)}
  \qquad (\varepsilon \le \kappa),
\]
with $C_1$ depending on $C_\cV$ only and $e(k)$ the \emph{stratum
exponent}, one copy of the value space per atom and one simplex direction
per free weight: $e(k) = k\dim\cV + k - 1$ in general, which is $2k - 1$ on
interval values with $C_1 = 34$, and $e(k) = k\dim\cV$ on a selector design
whose cropped weights are Lipschitz in the selected values or constant
given the support, since its weights are not there to cover
(\Cref{thm:covcons}, and the paragraph after \Cref{thm:covacc} for
selector designs). The ball of \Cref{sec:flocking} is a selector design,
its weights being uniform on the support, and the graded kernels read there
are not, since their weights are set by positions that the mechanism does
not read (\Cref{def:vspace}). Nothing in these counts depends on $d$.

\subsection{The two-sample certificate}\label{app:twosample}

\begin{theorem}[The two-sample certificate, in full]\label{thm:twosamplefull}
Let $n$ be even and let $S$ be split, by a rule fixed in advance, into
halves $A$ and $B$ of $n/2$ points each, with empirical lifted designs
$\hat\Gamma_A$ and $\hat\Gamma_B$.
\begin{enumerate}
\item[(i)] $\E\,\Wone(\hat\Gamma_n, \Gamma) \le \E\,\Wone(\hat\Gamma_A,
\Gamma) \le \E\,\Wone(\hat\Gamma_A, \hat\Gamma_B) \le 2\,\E\,
\Wone(\hat\Gamma_A, \Gamma)$.
\item[(ii)] With probability at least $1 - \delta$,
\[
  \Wone(\hat\Gamma_A, \Gamma) \;\le\; \Wone(\hat\Gamma_A, \hat\Gamma_B)
  + \kappa\sqrt{\frac{\ln(1/\delta)}{n}} ,
\]
and with probability at least $1 - \delta$,
\[
  \Wone(\hat\Gamma_n, \Gamma) \;\le\; \Wone(\hat\Gamma_A, \hat\Gamma_B)
  + \kappa\sqrt{\frac{\ln(2/\delta)}{n}} .
\]
\item[(iii)] $\Wone(\hat\Gamma_A, \hat\Gamma_B) = \min_\pi\,\frac{2}{n}
\sum_{i \le n/2}\Dl(P_i, P'_{\pi(i)})$, the minimum over the bijections
$\pi$ between the lifted points $P_i$ of $A$ and $P'_j$ of $B$: an
assignment problem on $n/2$ points, each of whose costs is itself an
assignment between the atoms of two lifts (\Cref{lem:Dassign}).
\end{enumerate}
\end{theorem}

\begin{proof}[Proof of \Cref{thm:twosamplefull}, which contains \Cref{thm:twosample}]
\emph{(i).} $\hat\Gamma_n = \tfrac12(\hat\Gamma_A + \hat\Gamma_B)$ and
$\Wone(\cdot, \Gamma)$ is convex, so $\Wone(\hat\Gamma_n, \Gamma) \le
\tfrac12\Wone(\hat\Gamma_A, \Gamma) + \tfrac12\Wone(\hat\Gamma_B, \Gamma)$,
whose expectation is $\E\,\Wone(\hat\Gamma_A, \Gamma)$, the halves being
exchangeable. Next, by duality $\Wone(\mu, \nu) = \sup_{F \in \Lip_1}(\int
F\,d\mu - \int F\,d\nu)$ is a supremum of functionals affine in $\nu$, and
$\int F\,d\Gamma = \E_B\int F\,d\hat\Gamma_B$ for every $F$, so
$\Wone(\hat\Gamma_A, \Gamma) = \sup_F\E_B[\int F\,d\hat\Gamma_A - \int
F\,d\hat\Gamma_B] \le \E_B[\Wone(\hat\Gamma_A, \hat\Gamma_B) \mid A]$; take
expectations over $A$. The last inequality is the triangle inequality
through $\Gamma$ and exchangeability. \emph{(ii).} If $\nu$ and $\nu'$ are
empirical measures on $m$ points differing in one point then $\Wone(\nu,
\nu') \le \kappa/m$, one atom of mass $1/m$ moved by at most the diameter
$\kappa$ (\Cref{prop:qual}). So $B \mapsto \Wone(\hat\Gamma_A,
\hat\Gamma_B)$ has bounded differences $2\kappa/n$ over the $n/2$ points of
$B$, and McDiarmid's inequality \citep[Thm.~6.2]{boucheron2013concentration}
bounds its conditional mean, which dominates $\Wone(\hat\Gamma_A, \Gamma)$
by the argument of (i), by its value plus $\kappa\sqrt{\ln(1/\delta)/n}$
with probability $1 - \delta$: the first display. For the second, apply the
first display at $\delta/2$ to $A$ and, with the roles of the halves
exchanged, to $B$: on the intersection, of probability at least $1 -
\delta$, both $\Wone(\hat\Gamma_A, \Gamma)$ and $\Wone(\hat\Gamma_B,
\Gamma)$ are at most $\Wone(\hat\Gamma_A, \hat\Gamma_B) +
\kappa\sqrt{\ln(2/\delta)/n}$, and $\Wone(\hat\Gamma_n, \Gamma) \le
\tfrac12\Wone(\hat\Gamma_A, \Gamma) + \tfrac12\Wone(\hat\Gamma_B,
\Gamma)$ by the convexity used in (i). \emph{(iii).} Transport between
two uniform measures on $n/2$ atoms each is a linear program over doubly
stochastic matrices, whose extreme points are the permutation matrices
(Birkhoff), so the minimum is attained at a bijection.
\end{proof}

\begin{remark}[One split suffices, and averaging over splits]\label{rem:onesplit}
(a) The split must be a rule fixed in advance or drawn independently of
the data, and the second display of \Cref{thm:twosamplefull}(ii) certifies
the learner on the whole sample from the one observed split at the price of
$\delta/2$ in place of $\delta$, a factor $\sqrt{\ln(2/\delta)/\ln(1/\delta)}$,
$1.11$ at $\delta = 0.05$. (b) For a fixed sample, the average of
$\Wone(\hat\Gamma_A, \hat\Gamma_B)$ over splits drawn independently of the
data has, by exchangeability, the same expectation, and McDiarmid over all
$n$ points with differences $2\kappa/n$ puts the expected deviation of a
random half, $\E\,\Wone(\hat\Gamma_{n/2}, \Gamma)$, within
$\kappa\sqrt{2\ln(1/\delta)/n}$ above it. It does not certify the
realized half $A$: conditionally on $A$ the expectation of a mixed split
need not dominate $\Wone(\hat\Gamma_A, \Gamma)$, so the argument of (ii)
does not apply to it, and \Cref{sec:flocking} uses the one observed split. (c)
\Cref{sec:flocking} reports the certificate of the learner on half $A$ and
the supplementary material that of the learner on the whole sample.
\end{remark}

\begin{lemma}[$\Dl$ between two lifts is a rectangular assignment]\label{lem:Dassign}
Let $P$ carry atoms $(x_i, u_i)$, $i \le s$, and $P'$ atoms $(y_j, v_j)$,
$j \le s'$, with $s \le s'$. Then
\[
  \Dl(P, P') \;=\; R\sum_{j \le s'} v_j \;+\; \min_{\pi}\sum_{i \le s} c_{i\pi(i)},
  \qquad
  c_{ij} := \min(u_i, v_j)\,\rho(x_i, y_j) + R\abs{u_i - v_j} - R v_j ,
\]
the minimum over injections $\pi : [s] \to [s']$. With uniform weights
$u_i = 1/s$, $v_j = 1/s'$, $\Dl(P, P') = M/s' + \kappa(1 - s/s')$ with $M$
the cost of a minimum-cost injection of the $s$ values into the $s'$
under $\rho$.
\end{lemma}

\begin{proof}
By \eqref{eq:padding}, $\Dl(P, P')$ is the Kantorovich distance with the
cone cost between the two apex-padded counting measures, and a transport
between counting measures of unit atoms is a bijection between the padded
atom sets (Birkhoff). An atom matched to an apex point costs $R$ times its
height, and a matched real pair costs $\vr = \min(u,v)\rho + R\abs{u - v}
\le R(u + v)$ since $\rho \le 2R$. So if an atom of $P$ and an atom of $P'$
are both matched to apex points, matching them to each other and the two
apex points to each other, at cost $0$, does not increase the cost, and
iterating, some optimal bijection matches every atom of the smaller lift
to a real atom. Its cost is $\sum_{\text{matched}}\vr_{i\pi(i)} +
R\sum_{\text{unmatched } j}v_j = \sum_i(\vr_{i\pi(i)} - Rv_{\pi(i)}) +
R\sum_j v_j$. With uniform weights $c_{ij} = \rho_{ij}/s' + R/s - 2R/s'$,
and summing over the $s$ matched atoms gives $M/s' + R - 2Rs/s'$, which with
$R\sum_j v_j = R$ is $M/s' + 2R(1 - s/s')$.
\end{proof}

The bound is stated for a \emph{resolution profile}, a nondecreasing
sequence $0 < \varepsilon_1 \le \cdots \le \varepsilon_K$, $K \le d$, one
resolution per number of resolved atoms, with thresholds $\theta_k :=
\varepsilon_k/(2\kappa) < 1$. It is nondecreasing so that the sets of
\Cref{def:shells} nest: a configuration that needs more atoms is resolved
more coarsely, never more finely.

\begin{definition}[Shells]\label{def:shells}
For $\theta \in [0, 1)$ and $0 \le k \le d$ let $B_k(\theta) := \{P \in
\cM_d : T_k(P) \le \theta\}$, the lifts within $\kappa\theta$ of the
stratum $\cA_k$, closed sets that increase with $k$, with $B_0(\theta) =
\emptyset$ and $B_d(\theta) = \cM_d$. For a resolution profile the sets
$B_k(\theta_k)$ are nested, since $T_k \le T_{k-1}$ and $\theta_{k-1} \le
\theta_k$, and the \emph{shells}
\[
  \Sh_k := B_k(\theta_k) \setminus B_{k-1}(\theta_{k-1}),
  \qquad 1 \le k \le K, \quad \theta_0 := \theta_1 ,
\]
partition $B_K(\theta_K)$. Their masses are increments of the profile,
\[
  w_k := \Gamma(\Sh_k) = G_{k-1}(\theta_{k-1}) - G_k(\theta_k) \;\ge\; 0
  \qquad (G_0 \equiv 1),
\]
and the mass beyond the last shell, $\Gamma(B_K(\theta_K)^c) =
G_K(\theta_K)$, is the \emph{trim}.
\end{definition}

Shell $k$ holds the configurations that $k$ atoms resolve to within
$\varepsilon_k/2$ and $k - 1$ atoms do not resolve to within
$\varepsilon_{k-1}/2$: the sparse configurations sit in the low shells and
are resolved finely, the crowded ones in the high shells and coarsely,
each shell at its own mass. A \emph{celling} of the shells partitions each
$\Sh_k$ into cells, each contained in a closed $\Dl$-ball of radius
$\varepsilon_k$ (\Cref{def:celling}). One celling is the theorem's own:
the crop of a configuration of shell $k$ to its $k$ heaviest coordinates
is a point of the stratum $\cA_k$ within $\varepsilon_k/2$ of it
(\Cref{prop:crop}), so the pullback of an $\varepsilon_k/2$-net of $\cA_k$
cells the shell into at most $N_k(\varepsilon_k)$ cells
(\Cref{lem:tube}). Another, used in the computations, quantizes the crop
on a grid and reads the cell off the point. Both are determined by nothing
random, and the shell and the cell of a sample point are computable since
$T_k$ is.

\begin{proposition}[The shells bound the transport]\label{prop:sampleshells}
For every resolution profile and every celling of its shells, both of
which may be chosen after seeing the sample, write $\hat w^A_k$, $\hat
w^B_k$ for the fractions of the two halves in shell $k$ and $\hat\mu^A_j$,
$\hat\mu^B_j$ for their fractions in cell $j$, each point beyond the last
shell a cell of its own. Then
\[
  \Wone(\hat\Gamma_A, \hat\Gamma_B) \;\le\;
  \sum_{k=1}^K\bigl(\hat w^A_k + \hat w^B_k\bigr)\varepsilon_k
  \;+\; \frac{\kappa}{2}\sum_j\bigl|\hat\mu^A_j - \hat\mu^B_j\bigr| .
\]
\end{proposition}

The inequality is deterministic, one feasible transport between the two
halves: move each point to its cell's center, couple the common mass
within each cell for free, and the residuals at the diameter. The second
term is the observed discrepancy of the two halves' cell frequencies, the
points beyond the last shell included as cells of their own at $2/n$ each,
a number of the sample, and its expectation under i.i.d.\ halves is what
the fluctuation term of the next theorem bounds. With \Cref{thm:twosample}
the right-hand side is a certificate, with the profile and the celling
free to optimize on the data.

\begin{proof}[Proof of \Cref{prop:sampleshells}]
Move each point of $A$ to the center of its cell, at cost at most
$\sum_k\hat w^A_k\varepsilon_k$ in $\Wone$, and likewise $B$; a point
beyond the last shell is its own center and moves at no cost. Between the
two measures on the centers, couple mass $\min(\hat\mu^A_j, \hat\mu^B_j)$
within each center at cost $0$ and the residuals arbitrarily at cost at
most $\kappa$, whose total mass is $\tfrac12\sum_j\abs{\hat\mu^A_j -
\hat\mu^B_j}$. The triangle inequality adds the three.
\end{proof}

\begin{proposition}[The expectation of the sample bound]\label{prop:expsample}
For every resolution profile, with the theorem's own celling of the
shells, at most $N_k(\varepsilon_k)$ cells of radius $\varepsilon_k$ on
shell $k$, fixed before sampling,
\[
  \E\,\Wone(\hat\Gamma_A, \hat\Gamma_B) \;\le\; 2\sum_k w_k\varepsilon_k
  + \kappa\sum_k\sqrt{\frac{w_k N_k(\varepsilon_k)}{n}} + \kappa\,G_K(\theta_K),
\]
which is \eqref{eq:graded} with $\kappa$ in place of $\kappa/2$. Hence, by
\Cref{thm:twosamplefull}(i), $\Delta_n(Q)$ obeys the same bound, and every
corollary of \Cref{thm:graded} holds along the two-sample route with its
constants at most doubled.
\end{proposition}

\begin{proof}
Take expectations in \Cref{prop:sampleshells}. \emph{Floor:} $\E(\hat w^A_k
+ \hat w^B_k) = 2w_k$. \emph{Discrepancy:} for a cell of mass $\mu_j$,
$\hat\mu^A_j - \hat\mu^B_j$ has mean $0$ and variance $2\mu_j(1 -
\mu_j)/(n/2) \le 4\mu_j/n$, so $\E\abs{\hat\mu^A_j - \hat\mu^B_j} \le
2\sqrt{\mu_j/n}$, and over the at most $N_k$ cells of shell $k$, whose
masses sum to $w_k$, Cauchy--Schwarz gives $\sum_j\E\abs{\hat\mu^A_j -
\hat\mu^B_j} \le 2\sqrt{w_kN_k/n}$; times $\kappa/2$ this is
$\kappa\sqrt{w_kN_k/n}$. \emph{Trim:} a trimmed point is a singleton cell
of discrepancy $2/n$, so the trimmed points contribute
$(\kappa/2)\,\E(\hat w^A_{\mathrm{trim}} + \hat w^B_{\mathrm{trim}}) =
\kappa\,G_K(\theta_K)$. The last claim: $a + 2b + c \le 2(a + b + c)$ for
the three nonnegative terms, so the bound is at most twice
\eqref{eq:graded} at every profile, hence at most twice its infimum, which
is what the corollaries optimize.
\end{proof}

\begin{remark}[The Rademacher form]\label{rem:rademacher}
The empirical Rademacher complexity of the class on the sample, $\Rad_S :=
\E_\sigma\sup_F\frac1n\sum_i\sigma_iF(P_i)$ over the $1$-Lipschitz $F$
bounded by $b = \kappa/2$ and independent uniform signs $\sigma_i$, gives
the classical certificate $\Delta_S \le 2\Rad_S + 3\kappa\sqrt{\ln(2/\delta)
/(2n)}$ with probability $1 - \delta$
\citep{koltchinskii2000rademacher,bartlett2002rademacher}: McDiarmid on
$\Delta_S$, the symmetrization of \citet{vapnik1971uniform}, McDiarmid on
$\Rad_S$. For the Lipschitz ball, $\Rad_S$ is at most the expected cost of
an optimal partial matching of the plus points to the minus points in
$\Dl$ with every unmatched point sent to a reservoir at cost $b$, by
pairing the terms of $\sum_i\sigma_iF(P_i)$ along the matching, and
conditionally on the sizes of the sign classes the plus and minus points
are a random split of the sample, so the balanced part of $\Rad_S$ is half
the two-sample distance of a random unbalanced split and the rest is the
reservoir term, at most $b/\sqrt n$. Against \Cref{thm:twosample} the form
pays the factor $2$ and three concentration terms, and gains the whole
sample in one statement without a split. Its use for a Lipschitz class
over the sample's own geometry is the Rademacher theorem of \citet{gottlieb2013adaptive},
in the entropy form through the doubling dimension of the sample, and the
cell bound of \Cref{prop:sampleshells} in the Rademacher form is
\citet[Prop.~4]{carrasco2025rademacher}.
\end{remark}

\begin{remark}[The certificate, read in full]\label{rem:certread}
Of \Cref{thm:twosamplefull}, part (ii) is the certificate. The observed distance between the halves,
plus one concentration term, bounds the generalization gap of the
class for the learner trained on half $A$ and for the learner trained on
the whole sample, and by part (iii) it is computable: on the flock it is
evaluated at up to $32{,}000$ samples per half in \Cref{sec:flocking}.
Three remarks on the form. The split must
be fixed in advance or drawn independently of the data, and one split
suffices: the observed distance is within $\kappa\sqrt{2\ln(1/\delta)/n}$
of its mean by McDiarmid on the $n$ points, and averaging over independent
splits certifies the expected deviation of a half, not the realized one
(\Cref{rem:onesplit}). Any quantity
computed from the sample that dominates $\Wone(\hat\Gamma_A,
\hat\Gamma_B)$ is a certificate too, whatever was looked at to compute it,
a partition chosen on the sample or a profile optimized on it, since the
event of (ii) does not depend on that choice: this is what lets the
shells of the profile enter, in \Cref{app:coupling}, as a bound on the
transport rather than as a hypothesis on the design. And the object is
classical in its parts: the two-sample distance is the maximum discrepancy
of \citet{bartlett2002model} on a Lipschitz ball, where it is a transport
by duality \citep[eq.~(11)]{luxburg2004distance} and a two-sample test
statistic in its own right \citep{ramdas2017wasserstein}, and the certificate is
the symmetrization of \citet{vapnik1971uniform} in that form, while the
empirical Rademacher complexity of the class
\citep{koltchinskii2000rademacher,bartlett2002rademacher} is the same
object with a random sign pattern in place of a balanced one, and gives a
certificate with the constant $2$ and three concentration terms
(\Cref{rem:rademacher}).
Across two domains the same transport, between the empirical source and
target samples, bounds the target risk of a Lipschitz class in domain
adaptation \citep[Thm.~2]{redko2017theoretical}, \citep[Thm.~1]{shen2018wasserstein},
with the joint error of the two domains and a concentration term per
sample; here the halves share a law, and one term suffices by the
convexity step of (ii). What is not classical is that the class is
exactly the Lipschitz ball of a metric the axioms determine, so that the
transport is the class's own complexity and not an envelope of a model
(\Cref{sec:lit}).
\end{remark}

\subsection{The shells and the coupling}\label{app:coupling}

\begin{lemma}[The tails descend to the lift]\label{lem:descend}
For every $0 \le k \le d$ there is a unique function $T_k : \cM_d \to [0,
1]$ with $T_k \circ \Cl = T_k(\cdot, \cdot)$, and it is continuous on
$(\cM_d, \Dl)$ and on $(\cM_d, \Dt)$. In particular $T_k$ is computable
from the lifted point alone.
\end{lemma}

\begin{proof}
\emph{Well defined:} if $\Cl(x, \mu) = \Cl(y, \nu)$ then by
\Cref{lem:factor} the multisets of positive (value, weight) pairs
coincide, hence so do the multisets of positive weights; padding by zeros
does not change $T_k$ for $k \ge 1$, and $T_0 \equiv 1$. \emph{Continuous:}
$T_k$ is continuous on $\cC$ (\Cref{def:profile}), $\cC$ is compact, and
$\Cl$ is a continuous surjection onto the metric, hence Hausdorff, space
$\cM_d$ (\Cref{lem:compact}), so $\Cl$ is a closed map and therefore a
quotient map; a continuous function constant on fibers descends to a
continuous function through a quotient map.
\end{proof}

The sets $B_k(\theta) = \{T_k \le \theta\}$ of \Cref{def:shells} are
therefore closed, the shells are Borel, and their masses are the profile
increments stated there.

\begin{definition}[Celling]\label{def:celling}
Let $(\cX, \varrho)$ be a metric space and $S \subseteq \cX$ Borel. An
\emph{$(r, N)$-celling} of $S$ is a partition of $S$ into at most $N$
Borel sets, each contained in a closed $\varrho$-ball of radius $r$ about
a point of $\cX$, its \emph{center}.
\end{definition}

\begin{proposition}[Single-partition coupling, grouped]\label{prop:coupling}
Let $(\cX, \varrho)$ be a metric space with $\diam\cX \le \kappa_\cX$,
$\Gamma$ a Borel probability measure on $\cX$, and $\hat\Gamma_n$ the
empirical measure of $n$ i.i.d.\ draws from $\Gamma$. Let $B \subseteq
\cX$ be Borel, let $B = \bigsqcup_{g} S_g$ be a finite Borel partition,
and suppose each group $S_g$ admits an $(r_g, N_g)$-celling. Then, with
$w_g := \Gamma(S_g)$ and $W_1$ the Kantorovich distance with ground
metric $\varrho$,
\[
  \E\, W_1(\hat\Gamma_n, \Gamma) \;\le\;
  2\sum_{g} r_g\, w_g
  \;+\; \frac{\kappa_\cX}{2}\sum_{g}\sqrt{\frac{N_g\, w_g}{n}}
  \;+\; \kappa_\cX\, \Gamma(B^c) .
\]
\end{proposition}

\begin{proof}
Let $\{V_c\}_{c \in g}$ be the celling of $S_g$; together the cells
$V_c$, $c \in \bigcup_g g$, form a Borel partition of $B$, the cells of
group $g$ having $\varrho$-diameter at most $2r_g$. Write $a_c =
\Gamma(V_c)$, $b_c = \hat\Gamma_n(V_c)$, $m_c = a_c \wedge b_c$, and $M =
\sum_c m_c$.

\emph{One coupling.} Inside each cell pair mass $m_c$ of $\Gamma$ with
mass $m_c$ of $\hat\Gamma_n$, for instance the product of the two
normalized restrictions, at cost at most $2r_g$ per unit; pair the two
residual measures, each of total mass $1 - M$, arbitrarily, at cost at
most $\kappa_\cX$ per unit. From $a \wedge b = \tfrac{a+b}{2} -
\tfrac{\abs{a-b}}{2}$ and $\sum_c a_c = \Gamma(B)$, $\sum_c b_c =
\hat\Gamma_n(B)$,
\[
  1 - M \;=\; \tfrac12\sum_c \abs{a_c - b_c}
  + \tfrac12\bigl(\Gamma(B^c) + \hat\Gamma_n(B^c)\bigr),
\]
hence, using $m_c \le a_c$ and $\sum_{c \in g} a_c = w_g$,
\[
  W_1(\hat\Gamma_n, \Gamma) \;\le\; 2\sum_g r_g\, w_g
  + \kappa_\cX\Bigl[\tfrac12\sum_c\abs{a_c - b_c}
  + \tfrac12\bigl(\Gamma(B^c) + \hat\Gamma_n(B^c)\bigr)\Bigr].
\]

\emph{Binomial fluctuations, Cauchy--Schwarz per group.} The cells are
fixed before sampling, so $n\hat\Gamma_n(V_c) \sim \mathrm{Bin}(n, a_c)$
and $\E\abs{a_c - b_c} \le \sqrt{\E(a_c - b_c)^2} \le \sqrt{a_c/n}$.
Summing over the at most $N_g$ cells of group $g$, $\sum_{c \in
g}\E\abs{a_c - b_c} \le \sqrt{N_g\sum_{c \in g}a_c/n} = \sqrt{N_g
w_g/n}$, and $\E\hat\Gamma_n(B^c) = \Gamma(B^c)$. Take expectations in
the display.
\end{proof}

\begin{lemma}[Celling from a tube]\label{lem:tube}
Let $S \subseteq \cX$ be Borel, $A \subseteq \cX$, and $0 \le s < r$. If
every point of $S$ lies within $\varrho$-distance $s$ of $A$, and $A$ has
a net of $N$ points at radius $r - s$, then $S$ admits an
$(r, N)$-celling centered at the net points.
\end{lemma}

\begin{proof}
A point $p$ of $S$ is within $s + \eta$ of some point of $A$ for every
$\eta > 0$, which is within $r - s$ of a net point, so $p$ is within $r +
\eta$ of one of the $N$ net points, and one net point $z_j$ serves for a
sequence $\eta \downarrow 0$: $\varrho(p, z_j) \le r$, so the closed balls
$\Ball(z_j, r)$, $j \le N$, cover $S$. Disjointify: $V_j := S \cap \Ball(z_j, r) \setminus
\bigcup_{j' < j}\Ball(z_{j'}, r)$, Borel, a partition of $S$, and $V_j
\subseteq \Ball(z_j, r)$.
\end{proof}

\begin{proposition}[Single-partition coupling, reference-point form]
\label{prop:couplingref}
Let $(\cX, \varrho)$ be a metric space, $\Gamma$ a Borel probability
measure on $\cX$ and $\hat\Gamma_n$ its empirical measure, and fix a
reference point $z_0 \in \cX$ with $\varrho_0(\cdot) := \varrho(\cdot,
z_0)$. Let $B = \bigsqcup_g S_g$ be a finite Borel partition of a Borel
$B \subseteq \cX$, each $S_g$ with an $(r_g, N_g)$-celling and with
$\varrho_0 \le \varrho_g$ on $S_g$. Then, with $w_g := \Gamma(S_g)$,
\[
  \E\, W_1(\hat\Gamma_n, \Gamma) \;\le\;
  2\sum_{g} r_g\, w_g
  \;+\; \sum_{g}\varrho_g\sqrt{\frac{N_g\, w_g}{n}}
  \;+\; 2\int_{B^c}\varrho_0\, d\Gamma .
\]
The point of the form is that it needs no bound on the diameter: residual
mass is priced at its own distance to the reference point.
\end{proposition}

\begin{proof}
Couple within cells as in the proof of \Cref{prop:coupling}, at cost at
most $2r_g$ per unit. The two residual measures, $\Gamma$'s, consisting of
the mass $a_c - m_c$ in each cell $c$ and of $\Gamma|_{B^c}$, and
$\hat\Gamma_n$'s, consisting of $b_c - m_c$ in each cell and of
$\hat\Gamma_n|_{B^c}$, have equal total mass, and any coupling $\pi$ of
them costs $\int\varrho\,d\pi \le \int(\varrho_0(P) +
\varrho_0(Q))\,d\pi(P,Q)$ by the triangle inequality through $z_0$, that
is at most $\int\varrho_0$ against each residual. A residual unit in a
cell of group $g$ costs at most $\varrho_g$, and the two residuals of a
cell together have mass $(a_c - m_c) + (b_c - m_c) = \abs{a_c - b_c}$;
the residual mass on $B^c$ costs $\int_{B^c}\varrho_0\,d\Gamma +
\int_{B^c}\varrho_0\,d\hat\Gamma_n$. Hence
\[
  W_1(\hat\Gamma_n, \Gamma) \;\le\; 2\sum_g r_g w_g
  + \sum_g \varrho_g\sum_{c \in g}\abs{a_c - b_c}
  + \int_{B^c}\varrho_0\, d(\Gamma + \hat\Gamma_n),
\]
and taking expectations with $\sum_{c\in g}\E\abs{a_c - b_c} \le
\sqrt{N_g w_g/n}$ as in that proof and $\E\int_{B^c}\varrho_0\,
d\hat\Gamma_n = \int_{B^c}\varrho_0\,d\Gamma$ gives the display.
\end{proof}

\begin{theorem}[Generalization over the class]\label{thm:graded}
For every design $Q$ and every resolution profile,
\begin{equation}\label{eq:graded}
  \Delta_n(Q) \;\le\;
  2\sum_{k=1}^{K} w_k\,\varepsilon_k
  \;+\; \frac{\kappa}{2}\sum_{k=1}^{K}\sqrt{\frac{w_k\,N_k(\varepsilon_k)}{n}}
  \;+\; \kappa\,G_K(\theta_K) .
\end{equation}
The count $N_k(\varepsilon_k)$ may be replaced by the covering number at
radius $\varepsilon_k/2$ of any set within $\varepsilon_k/2$ of $\Sh_k$, in
particular of the crop of the shell's support, whose exponent on a
selector design is $k\dim\cV$. The same bound with $\kappa$ in place of
$\kappa/2$ is the expectation of \Cref{prop:sampleshells} with the
theorem's own celling (\Cref{prop:expsample}). The trim may be traded for
a crop: cropping every configuration beyond the last shell to its $K$
heaviest atoms, which lands it in some shell, \eqref{eq:graded} holds with
$\kappa\,G_K(\theta_K)$ replaced by $2\kappa\,\E_Q[T_K;\ T_K > \theta_K]$
and each $w_k$ by $w_k + g_k$, where $g_k$ is the mass of the cropped
configurations whose crop lies in $\Sh_k$, $\sum_k g_k = G_K(\theta_K)$.
\end{theorem}

\begin{remark}[Shells by the first threshold met]\label{rem:firstthreshold}
The monotonicity of the profile serves only to make the shells a
partition. Defining instead $\Sh'_k := \{T_k \le \theta_k,\ T_j > \theta_j
\text{ for } j < k\}$, the configurations whose first threshold met is the
$k$th, gives a partition of the configurations that meet some threshold,
for every sequence $\varepsilon_1, \dots, \varepsilon_K$, with $\Sh'_k
\subseteq B_k(\theta_k)$ still, and \Cref{thm:graded} and
\Cref{prop:sampleshells} hold verbatim with $w_k := \Gamma(\Sh'_k)$ and
the trim the mass that meets no threshold (\Cref{app:coupling}). For a
nondecreasing profile the two definitions agree. The nested form is the one
computed in \Cref{sec:flocking}.
\end{remark}

\begin{proof}[Proof of \Cref{thm:graded}]
By Kantorovich--Rubinstein duality on the compact space $(\cM_d, \Dl)$
(\Cref{prop:qual}), $W_1(\hat\Gamma_n, \Gamma) = \sup\{\int F\,
d(\hat\Gamma_n - \Gamma) : \Lip_{\Dl}(F) \le 1\}$, and by \Cref{thm:rep}
the functions $F \circ \Cl$ with $\Lip_{\Dl}(F) \le 1$ are exactly the
members of $\cG_1$, so that $\int F\, d\hat\Gamma_n = n^{-1}\sum_i
f(Z_i)$ and $\int F\, d\Gamma = \E_Q f$ for $f = F \circ \Cl$. The class
is symmetric under $f \mapsto -f$, so the supremum of the signed
deviation is the supremum of its absolute value, and $\Delta_n(Q) = \E\,
W_1(\hat\Gamma_n, \Gamma)$. Apply \Cref{prop:coupling} on $(\cX, \varrho)
= (\cM_d, \Dl)$ with $\kappa_\cX = \kappa$ (\Cref{prop:qual}), $B =
B_K(\theta_K)$, and the groups $S_k = \Sh_k$, $k \le K$, Borel by
\Cref{lem:descend}, of masses $w_k$, with $\Gamma(B^c) = G_K(\theta_K)$.
Each group admits an $(\varepsilon_k, N_k(\varepsilon_k))$-celling by
\Cref{lem:tube}: a point of $\Sh_k \subseteq B_k(\theta_k)$ has $T_k \le
\theta_k$, hence lies within $\kappa\theta_k = \varepsilon_k/2$ of its
crop, a point of $\cA_k$ (\Cref{prop:crop}), and $\cA_k$ has a net of
$N_k(\varepsilon_k) = N(\varepsilon_k/2, \cA_k, \Dl)$ points at radius
$\varepsilon_k/2$, finite since $\cA_k$ is compact (\Cref{lem:compact}).
So $r_k = \varepsilon_k$, and the proposition's three terms are the
display. For the replacement of the count: \Cref{lem:tube} takes any set
$A$ within $\varepsilon_k/2$ of the shell, in particular the image of
$\Sh_k \cap \supp\Gamma$ under the $k$-crop, which lies in $\cA_k$ and
within $\varepsilon_k/2$ of the shell by the same crop cost; on a selector
design its covering number has exponent $k\dim\cV$ by the paragraph after
\Cref{thm:covacc}. For \Cref{rem:firstthreshold}: the sets $\Sh'_k = \{T_k \le
\theta_k\} \cap \bigcap_{j < k}\{T_j > \theta_j\}$ are Borel by
\Cref{lem:descend}, they partition $\bigcup_{k \le K}B_k(\theta_k)$ for
every sequence of thresholds, and $\Sh'_k \subseteq B_k(\theta_k)$, so the
same cellings apply and \Cref{prop:coupling} gives the display with $w_k =
\Gamma(\Sh'_k)$ and the trim $\Gamma(\bigcap_{k \le K}\{T_k > \theta_k\})$;
\Cref{prop:sampleshells} needs only a partition into cells of the stated
radii.

\emph{The crop clause.} Let $\varphi$ be the identity on $B_K(\theta_K)$
and the $K$-crop $\mathrm{cr}_K$ on its complement, a measurable map into
$B_K(\theta_K)$ since $T_K(\mathrm{cr}_K P) = 0$, with $\Dl(P, \varphi(P))
\le \kappa\,T_K(P)\,\mathbf 1\{T_K(P) > \theta_K\}$ by \Cref{prop:crop}.
Then $\varphi_\#\hat\Gamma_n$ is the empirical measure of $n$ i.i.d.\
draws from $\Gamma' := \varphi_\#\Gamma$, and coupling along $\varphi$ on
both sides,
\[
  W_1(\hat\Gamma_n, \Gamma) \;\le\; \frac{\kappa}{n}\sum_{i \le n}
  T_K(P_i)\,\mathbf 1\{T_K(P_i) > \theta_K\}
  + W_1(\varphi_\#\hat\Gamma_n, \Gamma')
  + \kappa\,\E_Q[T_K;\ T_K > \theta_K],
\]
whose expectation is $2\kappa\,\E_Q[T_K; T_K > \theta_K] + \E\,
W_1(\varphi_\#\hat\Gamma_n, \Gamma')$. Apply \Cref{prop:coupling} to
$\Gamma'$ with the same $B$ and the same groups: the cellings are
properties of the sets, $\Gamma'(B^c) = 0$, and $\Gamma'(\Sh_k) = w_k +
g_k$ with $g_k = \Gamma(\{T_K > \theta_K\} \cap \mathrm{cr}_K^{-1}(\Sh_k))$.
\end{proof}

\begin{remark}[The three terms of \Cref{thm:graded}, read in full]\label{rem:gradedread}
The three terms are a resolution floor, a fluctuation and a trim, and the
profile enters each. The floor is mass weighted: shell $k$ pays its own
resolution at its own mass, so a coarsely resolved high shell does not tax
the finely resolved bulk. The fluctuation is the design's mixture of local
coverings: shell $k$ pays the count of the $k$-atom stratum, of exponent
$e(k)$ and not $e(d)$, weighted by the mass the design puts there, and the
term is the binomial fluctuation of the sample's cell counts, which is
where the measure enters, with no function class in sight. The trim is the
mass beyond the last shell, paid at the diameter. The crop distance
$\kappa\theta_k$ is absorbed into the radius of the cells and paid once,
in the floor, never as a bias and never in the counts. Under the crop
clause a crowded configuration pays its own tail in place of the
diameter, and \Cref{cor:sparse} below is the constant-profile case, with
the multiscale coupling of \Cref{prop:multiscale} on the cropped stratum in
place of the single scale. The constant
profile $\varepsilon_k \equiv \varepsilon$ is the single-scale form, one
resolution and one trim, the shape of the trimmed empirical-measure bounds
of \citet{weed2019sharp}; the stratified form, one resolution per shell, is what a profile needs. The stratified form is not the single-scale bound with
more indices: the scales index disjoint parts of the design, not refinements
of one set, and the mass per shell is what lets each part pay its own count,
the per-shell step of \Cref{prop:coupling}. Every profile gives a valid bound and the left side
does not depend on the profile, so the bound is the infimum over $K$ and
over profiles. The shells are observable: $T_k(P_i)$ is computable from
the sample, so an estimator can bin by shell (\Cref{rem:squareloss}). The profile is an assumption on one part of the design, its sparsity across and within
configurations, and on nothing else: within a shell the distribution is arbitrary and the bound
charges the stratum at its worst, which is where the certificate, which sees the sample's own
configurations, improves on it (\Cref{sec:flocking}).
\end{remark}

\subsection{Proof of \Cref{cor:sparse}}\label{app:sparseproof}

The upper bound is proved here, from a multiscale coupling on the cropped stratum (\Cref{prop:multiscale}), with $C = 51C_1$ for the constant $C_1$ of the stratum count, $1734$ on interval values. The lower bound of \Cref{cor:sparse} is \Cref{lem:massbound} in \Cref{app:lower}: its bound is the smaller of $c_{k_0}n^{-1/(2k_0-1)}$ and $\varepsilon_{\max}(k_0)/2$, and since $n^{-1/(2k_0-1)} \le 1$ the statement holds with the smaller of the two constants $c_{k_0}$ and $\varepsilon_{\max}(k_0)/2$ in place of $c_{k_0}$.

\begin{lemma}[Nested cellings from covers]\label{lem:nested}
Let $(\cX, \varrho)$ be a metric space and $S \subseteq \cX$ Borel with
$\diam S \le \varepsilon_0$ and finite covering numbers, and put
$\varepsilon_j := 3^{-j}\varepsilon_0$. For every $J \ge 0$ there are Borel
partitions $\mathcal{Q}_0 = \{S\}, \mathcal{Q}_1, \dots, \mathcal{Q}_J$ of $S$, each cell of $\mathcal{Q}_j$
of diameter at most $\varepsilon_j$ and contained in a cell of
$\mathcal{Q}_{j-1}$, with $\abs{\mathcal{Q}_j} \le N(\varepsilon_{j+1}/2, S, \varrho)$.
\end{lemma}

\begin{proof}
This is the construction of \citet[Lemma 2.1]{boissard2014mean} and
\citet[Prop.~3]{weed2019sharp}. For $1 \le j \le J$ let $B^j_1, \dots,
B^j_{N_j}$ be closed balls of radius $\varepsilon_{j+1}/2$ covering $S$,
$N_j := N(\varepsilon_{j+1}/2, S, \varrho)$. Disjointify the finest cover:
$Q^J_\ell := S \cap B^J_\ell \setminus \bigcup_{\ell' < \ell}B^J_{\ell'}$
gives a Borel partition $\mathcal{Q}_J$ of $S$ into at most $N_J$ cells of diameter
at most $\varepsilon_{J+1} \le \varepsilon_J$. Given $\mathcal{Q}_{j+1}$, let
$Q^j_\ell$ be the union of the cells of $\mathcal{Q}_{j+1}$ that meet $B^j_\ell$
and were not assigned to $Q^j_{\ell'}$ for some $\ell' < \ell$. Every cell
of $\mathcal{Q}_{j+1}$ meets some $B^j_\ell$, so the $Q^j_\ell$ form a Borel
partition of $S$ into at most $N_j$ cells, each a union of cells of
$\mathcal{Q}_{j+1}$; and two points of $Q^j_\ell$ lie in cells of $\mathcal{Q}_{j+1}$, of
diameter at most $\varepsilon_{j+1}$, that both meet the ball $B^j_\ell$, of
diameter $\varepsilon_{j+1}$, so their distance is at most
$3\varepsilon_{j+1} = \varepsilon_j$. Finally $\mathcal{Q}_0 := \{S\}$.
\end{proof}

\begin{proposition}[Multiscale coupling]\label{prop:multiscale}
Let $(\cX, \varrho)$, $S$ and $\varepsilon_j$ be as in \Cref{lem:nested},
$\Gamma$ a Borel probability measure on $S$ and $\hat\Gamma_n$ the
empirical measure of $n$ i.i.d.\ draws from $\Gamma$. If
$N(\varepsilon_{j+1}/2, S, \varrho) \le (A/\varepsilon_j)^{e}$ for every $j
\ge 0$, with $A \ge \varepsilon_0$ and $e > 0$, then
\[
  \E\, W_1(\hat\Gamma_n, \Gamma) \;\le\; 17\,A\, n^{-1/e}
  \quad\text{if } e \ge 3,
  \qquad
  \E\, W_1(\hat\Gamma_n, \Gamma) \;\le\; 4A\,(1 + \ln n)\, n^{-1/2}
  \quad\text{if } e \le 2 .
\]
For $2 < e < 3$ the first bound holds with a constant depending on $e$ in
place of $17$, by the same proof. The exponents used below are integers.
\end{proposition}

\begin{proof}
\emph{The coupling.} Fix $J \ge 0$ and the partitions of \Cref{lem:nested},
and for two probability measures $\mu, \nu$ on $S$ write $a_C := \mu(C)$,
$b_C := \nu(C)$ for a cell $C$. We claim
\begin{equation}\label{eq:multiscale}
  W_1(\mu, \nu) \;\le\; 2\varepsilon_J
  + \sum_{j=1}^{J}\varepsilon_{j-1}\cdot\tfrac12\sum_{C \in \mathcal{Q}_j}\abs{a_C - b_C} .
\end{equation}
Pick a point $x_C$ in each cell $C \in \mathcal{Q}_J$ and let $\mu_J := \sum_C
a_C\,\delta_{x_C}$, $\nu_J := \sum_C b_C\,\delta_{x_C}$: moving every point
to its cell's representative costs at most $\varepsilon_J$, so $W_1(\mu,
\nu) \le 2\varepsilon_J + W_1(\mu_J, \nu_J)$. Couple $\mu_J$ with $\nu_J$
from the finest level up. First pair, at each $x_C$, the mass $a_C \wedge
b_C$ of the two, at cost $0$; what remains in a cell $C \in \mathcal{Q}_J$ is $(a_C
- b_C)_+$ of $\mu_J$ and $(b_C - a_C)_+$ of $\nu_J$. Then, for $j = J,
J-1, \dots, 1$ in turn, inside each cell $P \in \mathcal{Q}_{j-1}$ pair as much as
possible of the remaining $\mu_J$-mass of the cells $C \in \mathcal{Q}_j$, $C
\subseteq P$, with the remaining $\nu_J$-mass of those cells, at cost at
most $\diam P \le \varepsilon_{j-1}$ per unit. By induction on the level,
before the step at level $j$ the mass remaining in a cell $C \in \mathcal{Q}_j$ is
$(a_C - b_C)_+$ of $\mu_J$ and $(b_C - a_C)_+$ of $\nu_J$: the two sums
$\sum_{C \subseteq P}(a_C - b_C)_+$ and $\sum_{C \subseteq P}(b_C -
a_C)_+$ differ by $a_P - b_P$ and add to $\sum_{C \subseteq P}\abs{a_C -
b_C}$, so the step at level $j$ inside $P$ pairs $\tfrac12\sum_{C \subseteq
P}\abs{a_C - b_C} - \tfrac12\abs{a_P - b_P}$ units and leaves $(a_P -
b_P)_+$ of $\mu_J$ and $(b_P - a_P)_+$ of $\nu_J$ in $P$. At the top, $P =
S$ has $a_S = b_S = 1$ and nothing remains, so the pairing is a coupling
of $\mu_J$ and $\nu_J$, of cost at most $\sum_{j=1}^{J}\varepsilon_{j-1}
\sum_{P \in \mathcal{Q}_{j-1}}\tfrac12\sum_{C \subseteq P}\abs{a_C - b_C}$, which
is \eqref{eq:multiscale}.

\emph{The fluctuations.} Take $\mu = \hat\Gamma_n$ and $\nu = \Gamma$. The
cells are fixed before sampling, so as in the proof of
\Cref{prop:coupling} $\E\sum_{C \in \mathcal{Q}_j}\abs{a_C - b_C} \le
\sqrt{\abs{\mathcal{Q}_j}/n} \le (A/\varepsilon_j)^{e/2}n^{-1/2}$, and with
$\varepsilon_{j-1} = 3\varepsilon_j$,
\[
  \E\, W_1(\hat\Gamma_n, \Gamma) \;\le\; 2\varepsilon_J
  + \tfrac32\,A^{e/2}\,n^{-1/2}\sum_{j=1}^{J}\varepsilon_j^{\,1 - e/2} .
\]

\emph{The rate.} Let $e \ge 3$, so that $\varepsilon_j^{1-e/2} =
\varepsilon_J^{1-e/2}\,3^{-(J-j)(e/2-1)} \le \varepsilon_J^{1-e/2}\,
3^{-(J-j)/2}$ and the sum is at most $\varepsilon_J^{1-e/2}/(1 - 3^{-1/2})
\le 2.37\,\varepsilon_J^{1-e/2}$: the bound reads $2\varepsilon_J +
3.56\,\varepsilon_J\,(A/\varepsilon_J)^{e/2}n^{-1/2}$. Put $\varepsilon^\ast
:= A\,n^{-1/e}$. If $\varepsilon^\ast \ge \varepsilon_0$ the claim holds
since $W_1 \le \diam S \le \varepsilon_0$. Otherwise take $J :=
\lfloor\log_3(\varepsilon_0/\varepsilon^\ast)\rfloor \ge 0$, so that
$\varepsilon^\ast \le \varepsilon_J < 3\varepsilon^\ast$; then
$(A/\varepsilon_J)^{e/2}n^{-1/2} \le (A/\varepsilon^\ast)^{e/2}n^{-1/2} =
1$ and the bound is at most $5.56\,\varepsilon_J < 16.7\,\varepsilon^\ast
\le 17A\,n^{-1/e}$. For $e \le 2$ take $J := \lceil\tfrac12\log_3
n\rceil$, so that $\varepsilon_J \le \varepsilon_0 n^{-1/2}$; each summand
is at most $\varepsilon_0^{1-e/2} \le A^{1-e/2}$, so the bound is at most
$2\varepsilon_0 n^{-1/2} + \tfrac32 A\,J\,n^{-1/2} \le A\,(\tfrac72 +
\tfrac34\log_3 n)\,n^{-1/2} \le 4A\,(1 + \ln n)\,n^{-1/2}$.
\end{proof}

\begin{proof}[Proof of \Cref{cor:sparse}]
\emph{A single level.} Let $\mathrm{cr}_K : \cM_d \to \cA_K$ be the
$K$-crop, which keeps the $K$ heaviest atoms of a lift and renormalizes
their weights, a measurable map with $\Dl(P, \mathrm{cr}_K P) \le \kappa
T_K(P)$ (\Cref{prop:crop}). The lifted design $\Gamma^K :=
(\mathrm{cr}_K)_\#\Gamma$ is that of an exactly $K$-sparse design, the
empirical lifted design of the cropped sample is
$(\mathrm{cr}_K)_\#\hat\Gamma_n$, and by the triangle inequality
\[
  \Wone(\hat\Gamma_n, \Gamma) \;\le\;
  \Wone\bigl(\hat\Gamma_n, (\mathrm{cr}_K)_\#\hat\Gamma_n\bigr)
  + \Wone\bigl((\mathrm{cr}_K)_\#\hat\Gamma_n, \Gamma^K\bigr)
  + \Wone(\Gamma^K, \Gamma) ,
\]
where the first term is at most $\kappa n^{-1}\sum_i T_K(P_i)$ and the
third at most $\kappa\,\E_Q T_K$, by the coupling along the crop; take
expectations and apply the exactly sparse case, proved next, to the middle
term.

\emph{Exact sparsity.} $\Gamma^K$ is a probability measure on $\cA_K$,
of diameter at most $\kappa$ (\Cref{prop:qual}), and $N(\varepsilon/2,
\cA_K, \Dl) \le (C_1R/\varepsilon)^{e}$ for $\varepsilon \le \kappa$ with
$e := e(K)$ (\Cref{thm:covcons}; on a selector design the cropped support
of the stratum replaces $\cA_K$, with the count of the paragraph after
\Cref{thm:covacc}). Apply \Cref{prop:multiscale} on $S = \cA_K$ with
$\varepsilon_0 := \kappa$: at the scales $\varepsilon_j = 3^{-j}\kappa$,
$N(\varepsilon_{j+1}/2, \cA_K, \Dl) \le (C_1R/\varepsilon_{j+1})^e =
(3C_1R/\varepsilon_j)^e$, so $A := 3C_1R \ge \kappa$ serves, and the
middle term is at most $51C_1R\,n^{-1/e}$ for $e \ge 3$ and $12C_1R\,(1 +
\ln n)\,n^{-1/2}$ for $e \le 2$. So $C := 51C_1$ serves in both cases,
$C = 1734$ on interval values.
\end{proof}

\begin{remark}[Graph-network bounds in this currency]\label{rem:gnncurrency}
The generalization bounds for graph networks bound a different quantity:
the gap between the training and the test loss of one trained network, of
fixed width and norm budget, over $m$ i.i.d.\ graphs. The rate in $m$ is
$m^{-1/2}$ whatever the graphs, and the graph enters the constant: the
maximum degree $d$ as $d^{\,l-1}$ for an $l$-step message-passing network,
through the Lipschitz constant of the sum over neighbors, and the width $h$
and the stable ranks of the weights as, in effect, the square root of the
parameter count \citep[Theorems 3.2 and 3.4]{liao2021pacbayes}, with the
spectral norm of the diffusion matrix in place of $d$ in
\citet[Theorem 3.1]{ju2023generalization}. The class behind these bounds is
a per-neighbor map shared across the neighbors, summed and read out, the
weighted Deep Sets form of \Cref{sec:intro}: parameter
sharing makes its entropy independent of how many neighbors there are, and
the sum makes the count a scale, so nothing in these bounds is a rate in
the count, and at one layer with a mean the degree is absent altogether.
To learn a member of $\cG$ with such a network the width cannot stay fixed:
a sum decomposition that represents every symmetric function of $k$ elements
needs latent dimension $k$ \citep[Theorem 20]{wagstaff2022universal}, and
the readout is then a function of that many coordinates. What no
parametrization changes is the lower bound: on the stratum design of
\Cref{app:lower} at $k$ atoms the generalization gap of the class is at
least $n^{-1/(2k-1)}$ (\Cref{cor:sparse}), and so is the risk of any
learner (\Cref{thm:learnlower}), so the count is in the exponent. In shape, asked to learn
the whole class, the graph-network route, one architecture serving every
configuration with a readout of $k_{\max}$ coordinates, sees the design
through its maximum count, \Cref{cor:sparse} at $k_0 = k_{\max}$, where
for its own parametric class the degree is a constant in front of the
parametric rate. On the flock, where $k_{\max} =
68$, that bound is trivial (\Cref{app:sparseproof}). Where the count is a scale
on both sides, in the accumulative class, ours is the mean effective count
$\bar\zeta$ of \Cref{cor:sparseacc}, $12$ on the flock, against the maximum
degree of the sample, $68$. The conservative linear sums,
the members a sum network with a linear
readout learns without the count, are count-free at the rate of
\Cref{prop:countfree}, and the profile is the price of the rest.
\end{remark}

\paragraph{The graph-network bounds, and \Cref{cor:sparse} at the maximum
count.}
The PAC-Bayes bound of \citet[Theorem 3.2]{liao2021pacbayes} for an
$l$-layer graph convolutional network of width $h$ trained on $m$ graphs
has the form
\[
  L_{\mathcal D}(f_w) \;\le\; L_{S,\gamma}(f_w) +
  O\Biggl(\sqrt{\frac{B^2 d^{\,l-1} l^2 h\log(lh)\,
  \prod_{i}\norm{W_i}_2^2\,\sum_i \norm{W_i}_F^2/\norm{W_i}_2^2
  + \log(ml/\delta)}{\gamma^2 m}}\Biggr),
\]
with $d - 1$ the maximum degree of the training graphs, and their Theorem
3.4 for the message-passing network carries the geometric series
$((dC)^{l-1} - 1)/(dC - 1)$ through a different weight term.
\citet[Theorem 3.1]{ju2023generalization} replace the degree by the
spectral norm of the diffusion matrix, $\norm{P_G}^{2(l-1)}$, at most
$d^{2(l-1)}$ for a sum and at most one for a normalized mean, under a root
that carries the width $d_i$ of each layer and the stable ranks
$r_i^2\prod_j s_j^2$ explicitly, summed over the layers. Three features are
common to the line. The rate is $m^{-1/2}$. The count enters as a scale, a
power of the degree per layer that comes from the norm of a sum of $d$
bounded vectors, and it is absent at one layer with a mean. And the width
enters as roughly the square root of the parameter count, $h$ times the
stable rank, so that at a fixed architecture the bound is a constant over
$\sqrt m$. These are bounds on the gap of one network, not on the risk of
learning a class, and the class they are bounds for reads one neighbor at
a time through a shared map and sums. Its entropy does not see the count
because the parameters do not, which is the mechanism that makes the node
count drop out (\Cref{sec:lit}), and the same mechanism keeps the count out
of the exponent.

To learn a member of $\cG$ with such a network on a design of maximum
count $k$, the width has to grow with the accuracy. By
\citet[Theorem 20]{wagstaff2022universal} a continuous sum decomposition
$\rho(\sum_i \varphi(x_i))$ that approximates every continuous symmetric
function of $k$ scalar elements needs latent dimension at least $k$, and
below it there is a Lipschitz symmetric function, affine in the sorted
vector, that no decomposition approximates better than a constant does. So the readout $\rho$ is a function of $p \ge k$ coordinates, $k\dim\cV$ for
vector elements. What the count costs does not depend on the
parametrization: on the stratum design of \Cref{app:lower} at $k$ atoms
the generalization gap of the class is at least $n^{-1/(2k-1)}$
(\Cref{cor:sparse}, \Cref{lem:massbound}), and the risk of learning
likewise (\Cref{thm:learnlower}), so for the class the count is in the exponent
whatever the architecture, while the conservative linear sums are count-free, at the rate of
\Cref{prop:countfree}. One architecture serves every configuration, so the
count in the exponent is the maximum over the design, where
\Cref{cor:sparse} charges the count of the shell and \Cref{thm:graded} the
profile. On the flock the maximum support is $68$, and the bound at $k_0 =
68$ is the trivial value $\kappa$ on the flock, under either
exponent convention.

\subsection{Proof of \Cref{prop:countfree}}\label{app:countfreeproof}

\begin{proposition}[The linear sums are count-free]\label{prop:countfree}
Let $w = \mu$, and let $\Delta^{\mathrm{lin}}_n(Q) := \E\sup_{\varphi \in
\cS}\bigl(\E_Q f_\varphi - n^{-1}\sum_{i \le n} f_\varphi(Z_i)\bigr)$ be
the generalization process of the linear sums \eqref{eq:linearsums}. Write
$\nu_z := \sum_i \mu_i\,\delta_{x_i}$ for the weighted value measure of a
configuration $z = (x, \mu)$, a probability measure on $\cV$. Then
\[
  \Delta^{\mathrm{lin}}_n(Q) \;=\; \E\,\Wone^{\rho}\Bigl(\frac1n\sum_{i \le
  n}\nu_{Z_i},\ \E_Q\,\nu_Z\Bigr)
  \;\le\; \inf_{\varepsilon > 0}\Bigl[\,2\varepsilon +
  R\sqrt{\frac{N(\varepsilon, \cV, \rho)}{n}}\,\Bigr]
  \;\le\; C_{\cV}\,R\,n^{-\frac{1}{\dim\cV + 2}} ,
\]
with $\dim\cV$ the covering dimension of the value space and $C_\cV$ a
constant of its covering numbers alone. Neither $d$, nor the count, nor the
profile enters. In one dimension, $\cV$ an interval or the circle,
$\Delta^{\mathrm{lin}}_n(Q) \le 2R/\sqrt n$.
\end{proposition}

\begin{remark}[The two regimes: the count-free members and the rest]\label{rem:tworegimes}
The gap of the linear sums is a transport on the value space, between the
mean of $n$ weighted value measures and its expectation, one level below
the transport on the lifted space of \Cref{sec:complexity}, and the
count is absent because a weighted mean of $k$ atoms is one point of
$\cP(\cV)$ whatever $k$: crowding lowers the variance of the cell masses
and helps. On the flock $\dim\cV = 1$ and the rate is $n^{-1/2}$ at every
crowding, the one-dimensional case being exact (\Cref{app:countfreeproof}). This is the mechanism by which the parameter sharing of a sum
network keeps the count out of its entropy (\Cref{rem:gnncurrency}); a
Lipschitz readout of $p$ linear sums sees the law of a $p$-dimensional
feature of the mean lift and stays count-free at the price of $p$, the
bottleneck. The class pays the profile for everything beyond, the members
that are nonlinear across atoms: the selected maximum, the count, and
every functional whose score reads the height nonlinearly
(\Cref{rem:dualball}). The proof follows.
\end{remark}

\begin{proof}[Proof of \Cref{prop:countfree}]
For $z = (x, \mu)$ and $\varphi \in \cS$, $f_\varphi(z) = \int\varphi\,
d\nu_z$, so $\E_Q f_\varphi - n^{-1}\sum_i f_\varphi(Z_i) = \int\varphi\,
d(\bar\nu - \bar\nu_n)$ with $\bar\nu := \E_Q\,\nu_Z$ and $\bar\nu_n :=
n^{-1}\sum_i \nu_{Z_i}$, both probability measures on $\cV$. The score
ball $\cS$ contains a translate of every $1$-Lipschitz function on $\cV$
and $\bar\nu - \bar\nu_n$ has mass zero, so the supremum over $\cS$ is the
supremum over the $1$-Lipschitz functions, which is $\Wone^\rho(\bar\nu_n,
\bar\nu)$ by Kantorovich--Rubinstein duality on the compact space $(\cV,
\rho)$; this is the identity. For the bound let $c_1, \dots, c_N$ be an
$\varepsilon$-net of $\cV$, $N = N(\varepsilon, \cV, \rho)$, and $A_1,
\dots, A_N$ its Voronoi cells, every point within $\varepsilon$ of its
center. Moving a measure to the centers of its cells costs at most
$\varepsilon$ in $\Wone$, and between two measures on the centers $\Wone
\le \diam\cV \cdot \tfrac12\sum_j\abs{\bar\nu_n(A_j) - \bar\nu(A_j)}$, so
by the triangle inequality
\[
  \Wone^\rho(\bar\nu_n, \bar\nu) \;\le\; 2\varepsilon + R\sum_{j \le N}
  \abs{\bar\nu_n(A_j) - \bar\nu(A_j)} .
\]
The cell masses $\nu_{Z_i}(A_j)$, $i \le n$, are i.i.d.\ in $[0, 1]$ with
mean $\bar\nu(A_j)$, so $\E\abs{\bar\nu_n(A_j) - \bar\nu(A_j)} \le
\sqrt{\operatorname{Var}\,\nu_Z(A_j)/n} \le \sqrt{\bar\nu(A_j)/n}$, using $\nu_Z(A_j)^2
\le \nu_Z(A_j)$, and by Cauchy--Schwarz $\sum_j\sqrt{\bar\nu(A_j)/n} \le
\sqrt{N/n}$. Taking expectations gives the first bound. With
$N(\varepsilon, \cV, \rho) \le (C R/\varepsilon)^{\dim\cV}$, the choice
$\varepsilon = R\,C^{\dim\cV/(\dim\cV + 2)}\,n^{-1/(\dim\cV + 2)}$ balances
the two terms and gives the rate, with a constant depending on $C$ and
$\dim\cV$ alone. Nothing in the argument looks at how many atoms carry
$\nu_z$, which is where the count and the profile would enter.

In one dimension the net is not needed. On an interval $I$ of length $\ell$
with its length metric, $\Wone(\bar\nu_n, \bar\nu) = \int_I\abs{\bar F_n(t) -
\bar F(t)}\,dt$ with $\bar F_n(t) = n^{-1}\sum_i \nu_{Z_i}((-\infty, t])$ an
average of $n$ i.i.d.\ variables in $[0, 1]$ of mean $\bar F(t)$, so
$\E\abs{\bar F_n(t) - \bar F(t)} \le 1/(2\sqrt n)$ and $\E\,\Wone \le
\ell/(2\sqrt n) = R/\sqrt n$. On the circle cut at a point, the geodesic
distance is at most the distance on the cut interval, of length $2\pi =
4R$, so $\Wone^{S^1} \le \Wone^{I}$ and $\E\,\Wone \le 2R/\sqrt n$. On the
flock this is the rate of the linear sums at every crowding.
\end{proof}

\subsection{Regression}\label{app:genproof}

\begin{remark}[Regression with the absolute loss]\label{rem:regression}
The two-sample distance is the maximum discrepancy of the unit ball on
the halves, $\Wone(\hat\Gamma_A, \hat\Gamma_B) = \sup_{f \in
\cG_1}\abs{\E_A f - \E_B f}$ (\Cref{sec:certificate}), which
\Cref{thm:twosample} carries to the population with one concentration
term. For regression with the absolute
loss on a target $f^* \in \cG_1$, the loss $\abs{f - f^*}$ of a predictor
$f \in \cG_1$ is itself a member, of $\cG_2$, so with $R_A$ and $R_B$ the
empirical risks on the halves, $\sup_{f \in \cG_1}\abs{R_A(f) - R_B(f)}
\le 2\,\Wone(\hat\Gamma_A, \hat\Gamma_B)$: the certificate bounds the
generalization gap of every predictor at once, at twice its value, and of
the empirical risk minimizer in particular.
For a target with $m$ components that is $1$-Lipschitz into $\ell^2$,
under the Euclidean loss, a predictor whose components are each in $\cG_1$
is $\sqrt m$-Lipschitz into $\ell^2$, so the loss $\abs{f - f^\ast}_2$ has
constant $\sqrt m + 1$ and the factor is $1 + \sqrt m$ in place of $2$ (and
$2\sqrt m$ for a target that is only componentwise in $\cG_1$), which is
$1 + \sqrt2$ for the flock's update direction, jointly Lipschitz by
\Cref{prop:direction} (\Cref{app:flocking}). For labels in a range of
length $b$ the loss
$\abs{f(z) - y}$ is $1$-Lipschitz in the pair $(z, y)$ under $\Dl +
\abs{y - y'}$, so the same holds with the transport between the labelled
halves in that metric, the same matching with the label distance added to
each pair's cost, and the concentration term at the diameter $\kappa + b$.
A loss that is $\ell_1$-Lipschitz in the prediction
costs the factor $\ell_1$ in place of $1$, by Talagrand's contraction
lemma \citep{mohri2018foundations}.
\end{remark}
\begin{remark}[The square loss]\label{rem:squareloss}
For the square loss the rate is a different one: least squares has excess
risk at the square of a rate, and what governs it is the entropy of the
class under the design rather than the transport. The same shells give
that entropy and the same fibers give the matching lower bounds, so we
sketch the route and omit the statements, the regression that the
paper's number concerns being the absolute-loss one
(\Cref{rem:regression}). \emph{Entropy.} Cover each shell of
\Cref{def:shells} by the celling of \Cref{thm:graded} at its own
resolution $\varepsilon_k$, assign to every cell a value from a grid of
step $\varepsilon_k$ on $[-b, b]$, and the value $0$ beyond the cutoff
shell: this is a cover of the members of $\cG_1$ bounded by $b$ in
$L_2(\Gamma)$, of log-cardinality $\sum_{k \le K}
N_k(\varepsilon_k)\ln(3b/\varepsilon_k)$ at a radius $\eta$ with $\eta^2
\le \tfrac94\sum_{k \le K} w_k\varepsilon_k^2 + b^2G_K(\theta_K)$, on
interval values the entropy of a Lipschitz class
\citep{kolmogorov1959epsilon} in the dimension $2K - 1$ of the cutoff
shell, with the mass beyond it charged at $b$ and no dependence on $d$.
\emph{Upper bound.} Least squares over the cover is the shell-binning
regressogram, on each cell the grid value nearest the cell's sample mean
of $Y$, the cells being observable since $T_k$ of a sample is computable
from the sample. Bernstein's inequality with a union bound over the cover
gives an excess risk of order $\eta^2 + b^2\ln N/n$ over the best member,
and the resolution that balances the two terms gives $(\ln
n/n)^{2/(2k_0+1)}$ at exact sparsity, the nonparametric rate of a
Lipschitz class in the dimension $2k_0 - 1$ of the cutoff shell, and
$e^{-2\sqrt{\beta\ln n/7}}$ on an exponential profile, the shape of
\Cref{cor:exp}'s rate with another constant in the exponent. \emph{Lower bound.}
On a fiber of fixed weights the metric is
bi-Lipschitz to a weighted supremum distance with the weights as
semi-axes (\Cref{lem:fiber}), so bumps of height $h$ with disjoint
supports in the cells of the stratum design of \Cref{def:region} or of
the fiber design of \Cref{thm:lowerexp} are $1$-Lipschitz for $\Dl$, and
Assouad's hypercube \citep[\S2.7]{tsybakov2009introduction} on Gaussian
regression over them gives a minimax risk of $n^{-2/(2k+1)}$ on the
stratum design, $n^{-2/(k+2)}$ on a selector design at count $k$ and
$e^{-\sqrt{8\beta\ln n}}$ on the fiber design, with constants polynomial
in $k$; on a mixture of separated populations the hypercubes combine, and
the bound is the mass-weighted sum of the populations' risks at their
effective sample sizes $np_j$. By this sketch the
square loss learns at the nonparametric rate of the cutoff shell's
dimension, matched from below in the exponent at exact sparsity and in
shape on an exponential profile.
\end{remark}

\subsection{Exponential profiles: proof of \Cref{cor:exp}}\label{app:expproof}

The upper bound is proved here, with its constant $2e^{3\beta/2}\Lambda$, $\Lambda := CR \vee 2\kappa A$ for the constant $C$ of \Cref{cor:sparse} ($1734$ on interval values), and $2e^{3\beta}\Lambda$ with the selector count's constant in place of $C_1$ on a selector design. The lower bound of \Cref{cor:exp} is \Cref{thm:lowerexp} in \Cref{app:lower}, with the constant $\gamma^2R/16$, $\gamma := 1 - e^{-\beta}$.

\begin{proof}
Let $n \ge 2$ and $d \ge 2$: the bound at $n = 1$ exceeds the diameter
$\kappa$, and at $d = 1$ the lift is the interval, where $\Delta_n(Q) \le
2R/\sqrt n \le 2Re^{\beta/4}e^{-\sqrt{\beta\ln n/2}}$ by the one-dimensional
bound of \Cref{prop:countfree}. Put $x := \sqrt{2\ln n/\beta}$ and $K := d
\wedge (\lceil x/2 \rceil \vee 2)$, so that $e(K) = 2K - 1 \ge 3$, that
$\beta K \ge \sqrt{\beta\ln n/2}$ unless $K = d$, and that $2K - 1 \le x +
3$ in every case. \Cref{cor:sparse} at this $K$, with its constant $C$,
gives
\[
  \Delta_n(Q) \;\le\; 2\kappa\,\E_Q T_K + CR\,n^{-1/(2K-1)}
  \;\le\; 2\kappa A\,e^{-\beta K} + CR\,e^{-\ln n/(x+3)} ,
\]
where the first term vanishes at $K = d$ and is at most $2\kappa A\,
e^{-\sqrt{\beta\ln n/2}}$ otherwise, and $\ln n/(x + 3) \ge \ln n/x -
3\ln n/x^2 = \sqrt{\beta\ln n/2} - 3\beta/2$. So $\Delta_n(Q) \le (2\kappa
A + CR\,e^{3\beta/2})\,e^{-\sqrt{\beta\ln n/2}} \le 2e^{3\beta/2}
\Lambda\,e^{-\sqrt{\beta\ln n/2}}$. On a selector design, $e(K) = K$: take
$x := \sqrt{\ln n/\beta}$ and $K := d \wedge (\lceil x\rceil \vee 3)$, so
that $e(K) \ge 3$ once $d \ge 3$, $K \le x + 3$ and $\ln n/(x+3) \ge
\sqrt{\beta\ln n} - 3\beta$; for $d \le 2$ the design is exactly $d$-sparse
and the clause $e \le 2$ of \Cref{cor:sparse} at $K = d$, $CR(1 + \ln
n)n^{-1/2}$, is below the claimed rate at a constant depending on $\beta$.
\end{proof}

\paragraph{The shells route.}  The same rate with the constant
$\tfrac13$, at $n \ge n_0$, follows from \Cref{thm:graded} at a constant
profile with the trim at the diameter, and the computation serves the
accumulative class (\Cref{cor:expacc}), whose theorem has no crop clause.
Write $L := \ln n$, let $n_0 := \exp(48\beta \vee 9\ln^2\Lambda'/\beta)$,
and let $C_1 := 136$, the constant of the accumulative count in
\Cref{app:accumulative}, which dominates the $34$ of \Cref{thm:covcons} and
lets one $\Lambda'$ serve both classes, so that $N_k(\varepsilon) \le
(C_1R/\varepsilon)^{2k-1}$ for $\varepsilon \le \kappa$ (\Cref{thm:covcons})
and $\Lambda' := C_1R \vee 2\kappa A \vee e$, a different constant from the crop route's $\Lambda$. Apply
\Cref{thm:graded} with the constant profile at
\[
  \varepsilon := \Lambda' e^{-u}, \qquad \theta := \varepsilon/(2\kappa),
  \qquad
  u := \tfrac13\sqrt{\beta L},
\]
so that $\ln(\Lambda'/\varepsilon) = u$ and $N_k(\varepsilon) \le
e^{(2k-1)u}$. Three consequences of $n \ge n_0$, that is $48\beta \le L$
and $9\ln^2\Lambda' \le \beta L$, are used repeatedly:
\begin{enumerate}
\item[(a)] $u \le L/20$, from $\beta \le L/48$;
\item[(b)] $\beta \le u$ and $1 \le \ln\Lambda' \le u$, from $u \ge
\sqrt{\beta \cdot 48\beta}/3 \ge 2\beta$ and $u \ge \sqrt{9\ln^2\Lambda'}/3
= \ln\Lambda'$;
\item[(c)] $\kappa \le \Lambda'$, since $C_1 \ge 2$, and $L \ge \sqrt{48
\cdot 9}\,\ln\Lambda' \ge 20$.
\end{enumerate}
In particular $\varepsilon = \Lambda' e^{-u} \le 1 \le \kappa$, using $R \ge \tfrac12$ (so $\kappa = 2R \ge 1$), so the
count bound applies and $\theta \le \tfrac12$.

Markov's inequality at the working threshold
converts the moment profile into the quantile the theorem consumes:
\[
  G_k(\theta) \;\le\; 1 \wedge \frac{A e^{-\beta k}}{\theta}
  \;=\; 1 \wedge e^{-\beta(k - h)},
  \qquad
  h := \beta^{-1}\ln(2\kappa A/\varepsilon) \in [0,\, u/\beta].
\]
Set $m := \lceil h \rceil + 1 \le u/\beta + 2$ and apply the theorem with
$K := (\lceil 3u/\beta \rceil \vee 1) \wedge d$. \emph{Trim:} either $K =
d$, or $\beta K \ge 3u$ and $\kappa\, G_K(\theta) \le \kappa\,
e^{-\beta(K - h)} \le \kappa\, e^{-2u} \le \Lambda' e^{-u}$. \emph{Shells:}
$w_k \le G_{k-1}(\theta) \le e^{-\beta(k - m)_+}$, as $k - 1 - h \ge k -
m$, so the $k$-th shell term is at most $n^{-1/2} e^{\psi(k)}$ with
$\psi(k) := (2k-1)\tfrac{u}{2} - \tfrac{\beta}{2}(k - m)_+$. Split the sum
at $m$: below $m$ it is geometric with ratio $e^u \ge e$; above $m$,
geometric with ratio $e^{u - \beta/2} \ge e^{u/2} \ge e^{1/2}$ by (b),
and we extend it upward to $\tilde K := \lceil 3u/\beta \rceil \ge K$.
Hence
\begin{gather*}
  \sum_{k=1}^{K} \sqrt{\frac{w_k N_k}{n}} \;\le\;
  \frac{2\, e^{\psi(m)} + 3\, e^{\psi(\tilde K)}}{\sqrt{n}}, \\
  \psi(m) \;\le\; \tfrac{u^2}{\beta} + \tfrac{3u}{2}
  \;=\; \tfrac{L}{9} + \tfrac{3u}{2},
  \qquad
  \psi(\tilde K) \;\le\; \tfrac{3u^2}{\beta} + \tfrac{u}{2}
  - \bigl(u - \beta\bigr)
  \;=\; \tfrac{L}{3} - \tfrac{u}{2} + \beta,
\end{gather*}
using $m \le u/\beta + 2$, $\tilde K \le 3u/\beta + 1$, and $\tilde K - m
\ge 2u/\beta - 2 \ge 0$ by (b). The middle term of the theorem is thus at
most $\Lambda'\bigl(e^{\psi(m)} + \tfrac32 e^{\psi(\tilde K)}\bigr)/\sqrt{n}$,
and each piece is at most $e^{-u}$:
\begin{gather*}
  \psi(m) - \tfrac{L}{2}
  \;\le\; \tfrac{L}{9} + \tfrac{3u}{2} - \tfrac{L}{2}
  \;\le\; -\tfrac{L}{4} \;\le\; -u, \\
  \ln\tfrac32 + \psi(\tilde K) - \tfrac{L}{2}
  \;\le\; \tfrac{L}{48} + \tfrac{L}{3} + \tfrac{L}{48}
  - \tfrac{u}{2} - \tfrac{L}{2} \;=\; -\tfrac{L}{8} - \tfrac{u}{2}
  \;\le\; -u,
\end{gather*}
using $\tfrac{3u}{2} \le \tfrac{3L}{40}$, $\beta \le L/48$, $\ln\tfrac32
\le L/48$, and $u \le L/4$, all from (a)--(c). So the middle term is at
most $2\Lambda' e^{-u}$, and the total is $2\Lambda' e^{-u} + 2\Lambda'
e^{-u} + \Lambda' e^{-u} = 5\Lambda' e^{-u}$.

\begin{remark}[The exponential rate, read in full]\label{rem:expread}
The rate $e^{-c\sqrt{\beta\ln n}}$ is slower than
every polynomial in $n$ and faster than every power of $\ln n$. The
optimizing cutoff is $K_n \asymp \sqrt{\ln n/\beta}$, so the dimension
the rate sees at sample size $n$ grows without bound, like $\sqrt{\ln
n}$, and no fixed dimension explains the answer. In sample-complexity
form, accuracy $\varepsilon$ costs $n(\varepsilon) = \exp\bigl(c'
\ln^2(1/\varepsilon)/\beta\bigr)$, quasi-polynomial in $1/\varepsilon$,
between every finite-dimensional class and generic nonparametrics, and
free of $d$. The profile is a statement about the mean tail and costs
nothing for being one: the crop pays each configuration its own tail and
the bound takes the mean (\Cref{rem:expconstants}). Since every design is
exactly $d$-sparse,
\Cref{cor:sparse} at $K = d$ holds unconditionally and the bound is the
minimum of the two branches, which cross at $\ln n = \beta(2d-1)^2/2$; on
the measured flock, with $d$ in the thousands, the
crossover is unreachable at any rate the profiles of \Cref{sec:flocking} show. Two qualifiers. The content of these bounds is the exponent and the
absence of $d$, and the number at hand is the certificate's. And the rate
is a single level's, \Cref{cor:sparse} at its best $K$ with the mean tail as
bias: what the shells add is the across-configuration axis, where each
shell pays its own rate at its own mass.
\end{remark}

\begin{remark}[The constants in the exponent]\label{rem:expconstants}
The constant $1/\sqrt2$ is the crop's. A trim at the diameter, whether
\Cref{thm:graded} without its crop clause at a constant profile or the
finite-sample dimension of \citet[Prop.~10]{weed2019sharp} on the trimmed
coverings of \Cref{app:strata}, needs the tail to
fall below the working threshold on all but the trimmed mass, which a
mean tail gives only through Markov's inequality, at twice the cutoff:
that route returns $\tfrac12$ at best, and $\tfrac13$ as computed (the shells route of \Cref{app:expproof}, at $n \ge n_0$ and $R \ge \tfrac12$),
and the factor is the trim's price, not the profile's; a pointwise bound
$T_k \le Ae^{-\beta k}$ at every configuration would hand the trim its
$1/\sqrt2$ back, and nothing here needs one. The lower bound of
\Cref{thm:lowerexp} has $\sqrt2$.
\end{remark}

\subsection{The accumulative class}\label{app:accumulative}

The theorems above are instances of \Cref{prop:coupling}, and everything
specific to the metric entered through four inputs: the stratum covering,
the crop cost, the diameter of $(\cM_d, \Dl)$, and the qualitative facts
behind \Cref{lem:descend}. This subsection instantiates the argument for
the accumulative class $\cGt_L$, the Lipschitz ball of $\Dt$, the matching
distance with the shares $s_i = \mu_i/\mumax$ as heights. The lifts of the
two instances are bijections from the same quotient of $\cC$, so we regard
$\Dt$ as a second metric on the same set $\cM_d$, with $T_k$, $\mumax$ and
the shells the same functions and sets, and write $\abs{P} := \sum_i s_i =
1/\mumax(P)$ for the total height of $P$, the effective count, with the
leader at height $1$. Three inputs are imported from \Cref{app:strata}.
The crop cost: for $\hat\mu$ the $k$-crop, $\Dt(\Cl(x,\mu), \Cl(x,\hat\mu))
\le R\,T_k/\mumax = \kappa T_k/(2\mumax)$, and $\mumax(\hat\mu) =
\mumax/(1 - T_k) \ge \mumax$, the crop does not decrease the top weight
(\Cref{prop:crop} and its proof). The qualitative facts: $(\cM_d, \Dt)$ is
compact, $T_k$ and $\mumax$ are continuous on it, and duality holds for
$\Wt$, the Kantorovich distance on measures on $\cM_d$ with ground metric
$\Dt$ (\Cref{prop:qual}); moreover $\Dt(P, Q) \le R(\abs{P} + \abs{Q})$,
by sending every atom of both to the apex, so for a fixed one-atom lift
$P_0$,
\[
  \varrho_0(P) := \Dt(P, P_0) \;\le\; R\Bigl(\frac{1}{\mumax(P)} + 1\Bigr)
  \;\le\; \frac{\kappa}{\mumax(P)} ,
\]
while $\diam(\cM_d, \Dt) \le 2Rd$ and no better in general, the count
being $1$-Lipschitz for $\Dt$ and ranging over $[R, Rd]$. The celled
covering: for $m \ge 0$ let $\cA_k^{[m]} := \{P \in \cA_k : \mumax(P) \ge
2^{-m}\} = \{P \in \cA_k : \abs{P} \le 2^m\}$; then for every $k \le d$,
$m \ge 0$ and $\varepsilon > 0$, by the second display of
\Cref{thm:covacc},
\[
  N\bigl(2^{m-1}\varepsilon,\; \cA_k^{[m]},\; \Dt\bigr) \;\le\;
  \Nt_k(\varepsilon) \;:=\; N\bigl(\tfrac{\varepsilon}{4}, \cV, \rho\bigr)^k
  \cdot k\Bigl(1 + \frac{16R}{\varepsilon}\Bigr)^{k-1},
\]
with net points in $\cA_k$: resolving the cell $\{\abs{P} \le 2^m\}$ at
scale $2^{m-1}\varepsilon$ costs the conservative count up to the factor
$k$ and a constant, independently of $m$. On interval values and
$\varepsilon \le \kappa$, $\Nt_k(\varepsilon) \le (68R/\varepsilon)^{2k-1}$.

\begin{definition}[Concentration index, cells, concentration weights]
\label{def:cells}
For $P \in \cM_d$ let $m(P) := \lceil \log_2(1/\mumax(P)) \rceil \in \{0,
1, \dots\}$, so that $1/\mumax(P) \le 2^{m(P)} < 2/\mumax(P)$; $m(P) = 0$
exactly on point masses. For the stratified shells $\Sh_k$ of
\Cref{def:shells} set
\[
  \Sh_{k,m} := \Sh_k \cap \{m(P) = m\}, \qquad
  w_{k,m} := \Gamma(\Sh_{k,m}),
\]
the first and second \emph{concentration weights} of shell $k$,
\[
  \zeta_k := \sum_{m \ge 0} 2^m\, w_{k,m}
  = \int_{\Sh_k} 2^{m(P)}\, d\Gamma ,
  \qquad
  \xi_k := \sum_{m \ge 0} 4^m\, w_{k,m} ,
\]
and the \emph{count-weighted trim}
\[
  H_K := \int_{B_K(\theta_K)^c} \frac{d\Gamma}{\mumax}
  \;\in\; \bigl[\, G_K(\theta_K),\; d\, G_K(\theta_K) \,\bigr].
\]
Since $\mumax \ge (1 - \theta_k)/k$ on $\Sh_k$, only indices $m <
\log_2(2k/(1-\theta_k))$ occur there, at most $M_k := \lceil
\log_2(2k/(1-\theta_k)) \rceil$ cells per shell, and, from $1/\mumax \le
2^{m(P)} < 2/\mumax$,
\[
  w_k \;\le\; \zeta_k \;\le\; 2\int_{\Sh_k}\frac{d\Gamma}{\mumax}
  \;\le\; \frac{2k}{1-\theta_k}\, w_k ,
  \qquad
  \xi_k \;\le\; 4\int_{\Sh_k}\frac{d\Gamma}{\mumax^2}
  \;\le\; \Bigl(\frac{2k}{1-\theta_k}\Bigr)^2 w_k .
\]
The design's \emph{concentration moments} are $\bar\zeta :=
\E_Q[\mumax^{-1}]$ and $\bar\xi := \E_Q[\mumax^{-2}]$.
\end{definition}

\begin{theorem}[Generalization over the accumulative class]\label{thm:gradedacc}
For every design $Q$ and every resolution profile, with the shells and
masses of \Cref{def:shells} and the cells, weights and trim of
\Cref{def:cells},
\[
  \tilde\Delta_n(Q) \;=\; \E\, \Wt(\hat\Gamma_n, \Gamma) \;\le\;
  2\sum_{k=1}^{K} \zeta_k\,\varepsilon_k
  \;+\; \kappa \sum_{k=1}^{K} \sum_{m < M_k} 2^m
  \sqrt{\frac{w_{k,m}\, \Nt_k(\varepsilon_k)}{n}}
  \;+\; 2\kappa\, H_K .
\]
Supplementary bounds: the floor is at most $4\sum_k \varepsilon_k
\int_{\Sh_k}\mumax^{-1}\,d\Gamma$, hence at most $4\varepsilon\bar\zeta$
for the constant profile; the inner sum is at most $\sqrt{M_k\,\xi_k}$;
and $H_K \le \sqrt{\bar\xi\, G_K(\theta_K)}$.
\end{theorem}

\begin{proof}
The identity is duality for $\Wt$ and \Cref{thm:rep} for the accumulative
instance, as in the proof of \Cref{thm:graded}. Apply
\Cref{prop:couplingref} on $(\cX, \varrho) = (\cM_d, \Dt)$ with reference
point $P_0$ a point mass, $B = B_K(\theta_K)$, and the groups $S_{k,m} =
\Sh_{k,m}$, $k \le K$, $m < M_k$, Borel, of masses $w_{k,m}$.
\emph{Reference costs:} on $S_{k,m}$, $1/\mumax \le 2^m$, so $\varrho_0
\le \kappa/\mumax \le 2^m\kappa =: \varrho_{k,m}$; on $B^c$, $\varrho_0
\le \kappa/\mumax$, so $\int_{B^c}\varrho_0\,d\Gamma \le \kappa\, H_K$.
\emph{Cellings:} a point $P \in \Sh_{k,m}$ has $T_k(P) \le \theta_k$, so
its $k$-crop $\hat P$ satisfies $\Dt(P,\hat P) \le \kappa\theta_k/(2\mumax(P))
\le 2^{m-1}\kappa\theta_k = 2^{m-2}\varepsilon_k$, with $\hat P \in
\cA_k$ and $\mumax(\hat P) \ge \mumax(P) \ge 2^{-m}$, that is $\hat P \in
\cA_k^{[m]}$; and $\cA_k^{[m]}$ has a net of $\Nt_k(\varepsilon_k)$
points at radius $2^{m-1}\varepsilon_k \le 2^m\varepsilon_k -
2^{m-2}\varepsilon_k$, so \Cref{lem:tube} gives a $(2^m\varepsilon_k,
\Nt_k(\varepsilon_k))$-celling of $S_{k,m}$. With $r_{k,m} =
2^m\varepsilon_k$, $N_{k,m} = \Nt_k(\varepsilon_k)$ and $\varrho_{k,m} =
2^m\kappa$ the proposition's three terms are the display:
$2\sum_{k,m}2^m\varepsilon_k\,w_{k,m} = 2\sum_k\zeta_k\varepsilon_k$, then
$\sum_{k,m}2^m\kappa\sqrt{w_{k,m}\Nt_k/n}$, then $2\kappa H_K$. The
supplementary bounds: $2^{m(P)} < 2/\mumax(P)$ for the floor;
Cauchy--Schwarz over $m$ for the inner sum, $\sum_m 2^m\sqrt{w_{k,m}} \le
\sqrt{M_k\sum_m 4^m w_{k,m}}$; and Cauchy--Schwarz for $H_K \le
\sqrt{\int\mumax^{-2}\,d\Gamma}\,\sqrt{\Gamma(B_K^c)}$.
\end{proof}

Against \Cref{thm:graded} on the same design and the same shells, the
profile, the shells and the masses are identical: the accumulative class
does not induce a different partition, it is the same partition with cell
$(k, m)$ resolved at scale $2^m\varepsilon_k$ instead of $\varepsilon_k$.
The counts are identical in shape. The price of the semantics is a
count-weighting of all three terms, the floor by the mean inverse top
weight on the bulk, the fluctuation by the second concentration moment,
the trim by the count of the trimmed configurations. This is forced, not
slack: the count $R/\mumax$ is a member of $\cGt_1$ (\Cref{prop:roster}),
so two configurations of counts $c$ and $c'$ are at $\Dt$-distance at
least $R\abs{c - c'}$, and mass misplaced across counts must pay for the
difference. In the conservative class the total height is $1$ everywhere
and all three weights are $1$.

\begin{corollary}[Exact sparsity, accumulative class, the single-partition route]\label{cor:sparseacc}
If $Q(T_{k_0} = 0) = 1$ then $1/\mumax \le k_0$ $Q$-a.s., $\bar\zeta \le
k_0$, and with $K = k_0$ and the constant profile, for every $\varepsilon
> 0$,
\[
  \tilde\Delta_n(Q) \;\le\; 4\varepsilon\,\bar\zeta
  + 2k_0\kappa\sqrt{\frac{k_0\, M\, \Nt_{k_0}(\varepsilon)}{n}},
  \qquad M := \lceil\log_2(2k_0)\rceil .
\]
On interval values, optimizing $\varepsilon$, for $n \ge 68^{2k_0+1}k_0^3$,
so that the optimizing $\varepsilon$ is at most $\kappa$,
\[
  \tilde\Delta_n(Q) \;\le\; C\,R\;\bar\zeta\; n^{-\frac{1}{2k_0+1}}
  \;\le\; C\,R\,k_0\; n^{-\frac{1}{2k_0+1}}
\]
with $C$ absolute.
\end{corollary}

\begin{proof}
$G_{k_0}(\theta) \le G_{k_0}(0) = 0$ kills the trim, $H_{k_0} = 0$. On
the support of $Q$ at most $k_0$ weights are positive, so $1/\mumax \le
k_0$ there, $2^{m} < 2k_0$ on every occupied cell, at most $M$ values of
$m$ occur, and the fluctuation term is at most
$2k_0\kappa\sum_{k,m}\sqrt{w_{k,m}\Nt_k/n} \le 2k_0\kappa\sqrt{k_0 M\,
\Nt_{k_0}(\varepsilon)/n}$, by Cauchy--Schwarz over the at most $k_0 M$
occupied cells, $\sum_{k,m} w_{k,m} = 1$, and monotonicity of the counts
in $k$. For the rate insert $M\,\Nt_{k_0}(\varepsilon) \le
(C_1R/\varepsilon)^{2k_0-1}$ with $C_1 = 136$, from
$\Nt_{k_0}(\varepsilon) \le (68R/\varepsilon)^{2k_0-1}$ and $M \le 2k_0
\le 2^{2k_0-1}$, and choose
\[
  \varepsilon \;:=\; C_1R\,\Bigl(\frac{k_0^3}{\bar\zeta^2 n}\Bigr)^{\frac{1}{2k_0+1}}
  \;\le\; \kappa ;
\]
both terms are then at most $4(C_1+1)R\,(k_0^3)^{1/(2k_0+1)}\,
\bar\zeta^{(2k_0-1)/(2k_0+1)}\, n^{-1/(2k_0+1)}$, using $\kappa = 2R$;
conclude with $\bar\zeta \ge 1$ and $\sup_{k \ge 1} (k^3)^{1/(2k+1)} \le 2$.
\end{proof}

\begin{lemma}[Concentration moments from the profile]\label{lem:concmoment}
Let $A \ge 1$, $\beta > 0$, and $k_1 := \lceil \beta^{-1}\ln(2A)\rceil
\vee 1$.
\begin{enumerate}
\item[(i)] If $\E_Q[T_k] \le Ae^{-\beta k}$ for all $k$, then $\bar\zeta
\le 2k_1 + 2 + 2/\beta$, and more generally $\bar\zeta \le 2k +
4Ae^{-\beta k}/(1 - e^{-\beta})$ for every $k \ge 1$.
\item[(ii)] Under the same hypothesis, $\bar\xi \le 4k_1^2 + 4(2k_1+1)(1
+ 1/\beta) + 8/\beta^2$, so $\sqrt{\bar\xi} \le 4k_1 + 2 + 4/\beta$; and
for every $K \ge 1$ and $\theta$,
\[
  \int_{B_K(\theta)^c}\frac{d\Gamma}{\mumax} \;\le\; 2K\,G_K(\theta)
  + \frac{4Ae^{-\beta K}}{1 - e^{-\beta}} .
\]
\end{enumerate}
The concentration moments are constants of order
$\beta^{-1}(1 + \ln 2A)$, the effective neighbor count, and its square,
independent of $d$ and of $n$.
\end{lemma}

\begin{proof}
For any $k \ge 1$, $\mumax \ge (1 - T_k)/k$, so $\{1/\mumax > 2k\}
\subseteq \{T_k > \tfrac12\}$, and Markov gives
$\Gamma(1/\mumax > 2k) \le 2Ae^{-\beta k}$. (i) Layer cake: $\E[\mumax^{-1}] \le 2k + \sum_{j
\ge k} 2\Pr(\mumax^{-1} > 2j) \le 2k + 4A\sum_{j \ge k} e^{-\beta j}$. At
$k = k_1$, $4Ae^{-\beta k_1} \le 2$ and $(1 - e^{-\beta})^{-1} \le 1 +
1/\beta$. (ii) For $t := 1/\mumax$ and any integer $j_0 \ge 1$, pointwise
$t \le 2j_0 + 2\sum_{j \ge j_0}\mathbf 1\{t > 2j\}$ and $t^2 \le 4j_0^2 +
\sum_{j \ge j_0}(8j+4)\,\mathbf 1\{t > 2j\}$. Integrating the second
against $\Gamma$ at $j_0 = k_1$: $\bar\xi \le 4k_1^2 + 8A\sum_{j \ge
k_1}(2j+1)e^{-\beta j} \le 4k_1^2 + 8Ae^{-\beta k_1}\bigl[(2k_1+1)(1 +
1/\beta) + 2/\beta^2\bigr]$, using $\sum_{i \ge 0}e^{-\beta i} \le 1 +
1/\beta$ and $\sum_{i \ge 0} i e^{-\beta i} = e^{-\beta}/(1-e^{-\beta})^2
\le 1/\beta^2$, and $8Ae^{-\beta k_1} \le 4$; the square root by
$2\sqrt{ab} \le a + b$. Integrating the first against $\Gamma$ restricted
to $B_K(\theta)^c$ at $j_0 = K$ gives the last display.
\end{proof}

\begin{corollary}[Exponential profiles, accumulative class, the single-partition route]\label{cor:expacc}
Let $\cV = [-R,R]$ with $R \ge \tfrac12$, let $A$ and $\beta$ be as in
\Cref{cor:exp}, $\Lambda' := 136R \vee 2\kappa A \vee e$ as in the shells
route of \Cref{app:expproof}, $k_1$ as in \Cref{lem:concmoment}, and let
$n_0 := \exp(48\beta \vee 9\ln^2\Lambda'/\beta)$. Then for $n \ge n_0$,
under an exponential profile at rate $\beta$,
$\tilde\Delta_n(Q) \le 32\,(k_1 + 1 + 1/\beta)\,\Lambda'\,
e^{-\frac13\sqrt{\beta\ln n}}$.
\end{corollary}

\begin{proof}
Run the shells route of \Cref{app:expproof} on
\Cref{thm:gradedacc} with the constant profile: $\varepsilon :=
\Lambda' e^{-u}$, $\theta := \varepsilon/(2\kappa) \le \tfrac12$ (as
$\varepsilon \le 1 \le \kappa$), $u = \tfrac13\sqrt{\beta L}$, $L = \ln
n$; the consequences (a)--(c) hold verbatim, and $M_k\,\Nt_k(\varepsilon)
\le (C_1R/\varepsilon)^{2k-1} \le e^{(2k-1)u}$ with $C_1 = 136$, since
$M_k \le \lceil\log_2 4k\rceil \le 2k$ for $\theta \le \tfrac12$.

Take $K := (\lceil 3u/\beta\rceil \vee 1) \wedge
d$ as there, so that either $K = d$ or $\beta K \ge 3u$ and $G_K(\theta)
\le e^{-2u}$. \emph{Floor:} at most $4\varepsilon\bar\zeta =
4\Lambda'\bar\zeta e^{-u}$. \emph{Fluctuation:} by the supplementary bound
of the theorem and $\xi_k \le 4\bar\xi$ it is at most
$2\kappa\sqrt{\bar\xi}\sum_{k \le K}\sqrt{M_k\Nt_k(\varepsilon)/n} \le
2\kappa\sqrt{\bar\xi}\,n^{-1/2}\sum_{k \le K}e^{(k - \frac12)u} \le
4\kappa\sqrt{\bar\xi}\,e^{(K - \frac12)u - L/2}$, a geometric sum with
ratio $e^u \ge e$, and $(K - \tfrac12)u \le 3u^2/\beta + u/2 = L/3 +
u/2$, so the exponent is at most $-L/6 + u/2 \le -u$ as $u \le L/20$: the
term is at most $4\Lambda'\sqrt{\bar\xi}\,e^{-u}$. \emph{Trim:} if $K <
d$, \Cref{lem:concmoment}(ii) gives $2\kappa H_K \le 4\kappa K e^{-2u} +
8\kappa A(1 + 1/\beta)e^{-\beta K} \le 4\Lambda'(3u/\beta + 1)e^{-2u} +
4\Lambda'(1 + 1/\beta)e^{-3u}$, using $\kappa \le \Lambda'$, $2\kappa A \le
\Lambda'$, $K \le 3u/\beta + 1$ and $\beta K \ge 3u$; since $ue^{-u} \le
e^{-1}$ and $e^{-u} \le e^{-1}$, this is at most $4\Lambda'
e^{-u}\bigl[(3e^{-1} + e^{-2})/\beta + 2e^{-1}\bigr] \le 8\Lambda'(1 +
1/\beta)e^{-u}$. If $K = d$ the trim is zero. \emph{Total:} at most
$4\Lambda' e^{-u}\,[\bar\zeta + \sqrt{\bar\xi} + 2 + 2/\beta]$, and
\Cref{lem:concmoment}(i)--(ii) give $\bar\zeta + \sqrt{\bar\xi} \le 6k_1
+ 4 + 6/\beta$, so the bracket is at most $8(k_1 + 1 + 1/\beta)$.
\end{proof}

\begin{corollary}[Exact sparsity, accumulative class, the multiscale route]\label{cor:sparseacc2}
Let $\cV = [-R, R]$. If $Q(T_{k_0} = 0) = 1$ with $k_0 \ge 2$, then
\[
  \tilde\Delta_n(Q) \;\le\; C\,R\,k_0\; n^{-\frac{1}{2k_0-1}}, \qquad C := 17 \cdot 408 = 6936 .
\]
Conversely $\tilde\Delta_n(Q) \ge \tfrac12\Delta_n(Q)$ on every design, since
$\cG_1 \subseteq \cGt_2$ (\Cref{prop:roster}), so on the stratum design
$Q_{k_0}$ of \Cref{lem:massbound} the exponent $1/(2k_0-1)$ holds from both
sides.
\end{corollary}

\begin{proof}
On the support of $Q$ at most $k_0$ weights are positive, so $1/\mumax \le
k_0$ there and every lift lies in $S := \cA_{k_0}^{[M]}$ with $M :=
\lceil\log_2 k_0\rceil$, so that $k_0 \le 2^M < 2k_0$. Its diameter:
$\Dt(P, P') \le R(\abs{P} + \abs{P'}) \le 2Rk_0 =: \varepsilon_0$. Its
covering: at a radius $\varrho \le 2^{M-1}\kappa$ put $\varepsilon :=
\varrho/2^{M-1} \le \kappa$, so that $N(\varrho, S, \Dt) \le
\Nt_{k_0}(\varepsilon) \le (68R/\varepsilon)^{2k_0-1} =
(68R\,2^{M-1}/\varrho)^{2k_0-1} \le (68Rk_0/\varrho)^{2k_0-1}$. Apply
\Cref{prop:multiscale} on $S$ with $\varepsilon_j := 3^{-j}\varepsilon_0$:
the radii $\varepsilon_{j+1}/2 = \varepsilon_j/6 \le \varepsilon_0/6 =
Rk_0/3 \le 2^{M-1}\kappa$ are admissible, and $N(\varepsilon_{j+1}/2, S,
\Dt) \le (6 \cdot 68Rk_0/\varepsilon_j)^{2k_0-1} = (A/\varepsilon_j)^{e}$
with $A := 408Rk_0 \ge \varepsilon_0$ and $e := 2k_0 - 1 \ge 3$. The
proposition gives $\E\,\Wt(\hat\Gamma_n, \Gamma) \le 17A\,n^{-1/e}$. The
converse: a member of $\cG_1$ is a member of $\cGt_2$, so
$\sup_{\cGt_1}(\E_Q f - \E_n f) \ge \tfrac12\sup_{\cG_1}(\E_Q f - \E_n f)$
pointwise in the sample, and \Cref{lem:massbound} bounds the expectation of
the right side from below on $Q_{k_0}$.
\end{proof}

\begin{corollary}[Exponential profiles, accumulative class, the crop route]\label{cor:expacc2}
Let $\cV = [-R, R]$ and let $Q$ have an exponential profile at rate
$\beta$ with constant $A \ge 1$. Then for every $n \ge 2$ and $d \ge 2$,
\[
  \tilde\Delta_n(Q) \;\le\; C\,R\,\bigl[A(1 + 1/\beta) + e^{3\beta/2}\bigr]
  \Bigl(\sqrt{\tfrac{\ln n}{2\beta}} + 3\Bigr)\, e^{-\sqrt{\beta\ln n/2}} ,
  \qquad C = 6936 :
\]
the exponent of \Cref{cor:exp}, with the effective count at the cutoff,
$K_n \asymp \sqrt{\ln n/\beta}$, as a factor.
\end{corollary}

\begin{proof}
\emph{The crop.} For $K \le d$ let $\mathrm{cr}_K$ be the $K$-crop, with
$\Dt(P, \mathrm{cr}_K P) \le R\,T_K(P)/\mumax(P)$ (\Cref{prop:crop}), and
$\Gamma^K := (\mathrm{cr}_K)_\#\Gamma$, the lift of an exactly $K$-sparse
design. As in the proof of \Cref{cor:sparse}, coupling along the crop on
both sides,
\[
  \tilde\Delta_n(Q) \;\le\; 2R\,\E_Q\Bigl[\frac{T_K}{\mumax}\Bigr]
  + \E\,\Wt\bigl((\mathrm{cr}_K)_\#\hat\Gamma_n, \Gamma^K\bigr) ,
\]
and once $K \ge 2$ the second term is at most $6936\,RK\,n^{-1/(2K-1)}$ by
\Cref{cor:sparseacc2}, the cropped lifts having $K$ atoms of share at most
$1$. \emph{The crop cost.} On $\{T_K \le \tfrac12\}$, $\mumax \ge (1 -
T_K)/K \ge 1/(2K)$, so pointwise $T_K/\mumax \le 2K\,T_K +
\mumax^{-1}\mathbf 1\{T_K > \tfrac12\}$, and \Cref{lem:concmoment}(ii) at
$\theta = \tfrac12$, with $G_K(\tfrac12) \le 2\,\E_Q T_K \le 2Ae^{-\beta K}$
by Markov, gives
\[
  \E_Q\Bigl[\frac{T_K}{\mumax}\Bigr] \;\le\; 2KAe^{-\beta K} + 4KAe^{-\beta K}
  + \frac{4Ae^{-\beta K}}{1 - e^{-\beta}}
  \;\le\; Ae^{-\beta K}\,(6K + 4 + 4/\beta) ,
\]
and $6K + 4 + 4/\beta \le 6(1 + 1/\beta)(K + 1)$.
\emph{The cutoff.} As in the proof of \Cref{cor:exp}, put $x :=
\sqrt{2\ln n/\beta}$ and $K := d \wedge (\lceil x/2\rceil \vee 2)$, so that
$\beta K \ge \sqrt{\beta\ln n/2}$ unless $K = d$, where $T_d = 0$ and the
crop costs nothing, that $2K - 1 \le x + 3$, whence $n^{-1/(2K-1)} \le
e^{-\ln n/(x+3)} \le e^{3\beta/2}e^{-\sqrt{\beta\ln n/2}}$, and that $K + 1
\le x/2 + 3$. The two terms are then at most $12RA(1 + 1/\beta)(K+1)\,
e^{-\sqrt{\beta\ln n/2}}$ and $6936\,R\,e^{3\beta/2}(K+1)\,e^{-\sqrt{\beta\ln
n/2}}$, and $12 \le 6936$.
\end{proof}

\subsection{The two caps compared, and the accumulative certificate}\label{app:capprice}

Against \Cref{cor:sparse} and \Cref{cor:exp}, the multiscale and crop
routes give the accumulative class the same exponents
(\Cref{cor:sparseacc2,cor:expacc2}), and the accumulative price is a factor
of the count, paid once: the maximal count $k_0$ at exact sparsity, and the
effective count at the cutoff, $\sqrt{\ln n/(2\beta)} + 3$, on exponential
profiles. The single-partition route of \Cref{thm:gradedacc} gives the
weaker exponents $1/(2k_0+1)$ and $1/3$ (\Cref{cor:sparseacc,cor:expacc}),
with the mean count $\bar\zeta$ and a constant of order $\beta^{-1}(1 +
\ln 2A)$ in place of the maximal count, the profile bounding the
concentration moments (\Cref{lem:concmoment}). This is the two semantics of
\Cref{sec:axioms} as arithmetic. On a hard selection of $k$ coordinates
the accumulative class is the conservative class at scale $kL$
(\Cref{cor:fiber}), and the design pays that scale.

\begin{remark}[The accumulative certificate at a truncated diameter]\label{rem:accertificate}
\Cref{thm:twosample} needs a diameter for its concentration step, and
$(\cM_d, \Dt)$ has $2Rd$ and no better, the count being a member of the
class. Let $h(P) := \sum_i s_i = 1/\mumax(P)$ be the total height, so that
$\Dt(P, P') \le R(h(P) + h(P'))$ through the reservoir. Fix $\zeta > 0$,
let $\tau := Q(h > \zeta)$ and $\Gamma_\zeta := \Gamma(\cdot \mid h \le
\zeta)$. On the event that all $n$ sample points have $h \le \zeta$, of
probability at least $1 - n\tau$, the points are i.i.d.\ from
$\Gamma_\zeta$, whose support has $\Dt$-diameter at most $2R\zeta$, so
\Cref{thm:twosample} applies with $\kappa_\zeta :=
2R\zeta$ in place of $\kappa$ and $\Gamma_\zeta$ in place of $\Gamma$; and
$\Wt(\Gamma_\zeta, \Gamma) \le R\tau\,(\zeta + \E[h \mid h > \zeta])$, by
coupling the heavy mass to light points through the reservoir. Hence, with
probability at least $1 - \delta - n\tau$,
\[
  \Wt(\hat\Gamma_A, \Gamma) \;\le\; \Wt(\hat\Gamma_A, \hat\Gamma_B)
  + 2R\zeta\sqrt{\frac{\ln(1/\delta)}{n}}
  + R\tau\bigl(\zeta + \E[h \mid h > \zeta]\bigr).
\]
The level comes from the profile, by the first line of the proof of
\Cref{lem:concmoment}: $\mumax < 1/(2k)$ forces $T_k > \tfrac12$, so
$Q(h > 2k) \le Q(T_k > \tfrac12)$, and Markov
gives $Q(h > 2k) \le 2Ae^{-\beta k}$, and $\zeta_n :=
2\lceil\beta^{-1}\ln(2An/\delta)\rceil$, at most $(2/\beta)\ln(2An/\delta) + 2$, makes $n\tau \le \delta$, the display then holding with probability at least $1 - 2\delta$, with a tail term of
order $\delta\zeta_n/n$; the price is the factor $\zeta_n$ on the
concentration term, the effective count at the sample size, logarithmic
in $n$, the quantity the bounds in expectation carry as $\bar\zeta$. The
trim that works is on the heavy configurations, not on the light atoms:
cropping a crowded configuration to $K$ atoms costs $R$ times the removed
height, cheap in the conservative class where the height is the weight
and expensive here where it is the share. The cell bound of
\Cref{prop:sampleshells} and the exact matching need no diameter, a
trimmed sample point being priced at its own observed height.
\end{remark}

\section{Examples}\label{app:examples}

\subsection{Flocking}\label{app:flocking}

The protocol below stores a sample of \Cref{sec:flocking}, one focal
particle at one instant, as a \emph{view}, the term of the legend of
\Cref{fig:count}.

\paragraph{The dynamics.} Particles $i = 1, \dots, d$ carry positions $p_i
\in \T^2_\ell$, the torus of side $\ell$, and headings $\theta_i \in S^1$.
Writing $u(\theta) = (\cos\theta, \sin\theta)$, one step of the drift of the
dynamics is
\[
  \theta'_i \;=\; \arg \sum_{j} c_{ij}\, e^{\mathrm{i}\theta_j},
  \qquad
  p'_i \;=\; p_i + v_0\, u(\theta'_i),
\]
where $c_{i\cdot} \in \Delta_d$ are the neighborhood weights of particle $i$,
the selector, computed from the distances $\abs{p_j - p_i}$. The classical
model \citep{vicsek1995novel} takes $c_{i\cdot}$ uniform on $\{j :
\abs{p_j - p_i} \le r\}$, the metric interaction on a hard ball. Two
variants are in the literature, the second in our roster: the
\emph{topological} interaction, uniform on the particle itself and its $k$
nearest neighbors, the form supported by the starling observations of
\citet{ballerini2008interaction} with $k \approx 7$, and \emph{soft kernels},
$c_{ij} \propto K(\abs{p_j - p_i})$.

\begin{convention}[Noise]\label{conv:noise}
The stochastic term, angular noise added to $\theta'_i$, is exogenous: the
mechanism is the drift map above, the noise is added afterwards, and all
statements are about the drift.
\end{convention}

\paragraph{The four flocks.} Every flock has $d = 3200$ particles on the
torus of side $\ell = 40$, density $2$, at speed $v_0 = 0.5$, and the
flocks differ in the rule. The ball rule is the one above at $r = 1$. The
three graded rules have the same drift with soft weights $c_{ij} \propto
K(\abs{p_j - p_i})$: the bump $K(t) = (1 - t^2/r_b^2)_+^2$ at $r_b = 1.4$,
the Gaussian $K(t) = e^{-t^2/2h^2}$ at $h = 0.5$ truncated at radius $3.5$,
and the Cauchy $K(t) = (1 + t^2/h^2)^{-3}$ at $h = 0.9$ truncated at radius
$4$, beyond which lies $0.2\%$ of the untruncated kernel's weight.
The range of each graded
rule is the one at which its mean tail matches the ball's at ten atoms,
$T_{10}$ between $0.16$ and $0.19$ for every flock in both regimes, so the
four rules have the same effective number of partners. Each flock is read
through the selector of its own rule at the rule's parameters: the lift of
a sample is the weights the flock's own update puts on it, and the
mechanism whose class the certificate bounds is that update. The two
regimes are $\eta = 1.2$, the main text's, and $\eta = 2.4$, both in the
ordered phase \citep{chate2008collective}.

\paragraph{The protocol.} The protocol is the same for every flock and
regime: a flock, a list of seeds, one replica per seed, and every replica
viewed at the same steps. A replica draws uniformly random positions and
headings from its seed and runs the drift with its noise, $64{,}000$
replicas per flock and regime. A view is one focal particle, drawn
uniformly at random at a fixed step of its replica, with its neighborhood
within radius $3.5$ ($4$ on the Cauchy flock) stored as positions relative
to the focal particle and headings rotated by an independent uniformly
random angle, the self atom included at position $0$: the design certified
is the law of the rotated view, a design of the same mechanism, which is
equivariant under the rotation.
Every replica is viewed at step $3000$, after the burn-in, by which the
flock has ordered (the polarization below), and the focal particle and
the rotation of a view are drawn from the replica's generator after its
last step.
One view per replica: the views are independent samples
of the flock's law at that age, the i.i.d.\ hypothesis of
\Cref{thm:twosample} holds as stated, and nothing is assumed about mixing
within a run.

\paragraph{The split, the lift and the certificate.} Half $A$ is the replicas of the first $32{,}000$ seeds and
half $B$ those of the last $32{,}000$, a split fixed before any data were
seen, and the nested sample sizes are the prefixes of the seed order, $n =
500$ to $64{,}000$ samples in doublings, both halves counted. The lift of a
view is the selector's weights on its stored atoms, normalized to sum to
one, and on the Gaussian and the Cauchy flock the lift is cropped to the
heaviest atoms carrying all but $10^{-3}$ of the weight, at the cost of
\Cref{prop:crop} per sample, paid in the certificate by the triangle
inequality: $0.0017\kappa$ and $0.0019\kappa$ at the largest $n$. Every
distance between a lift of $A$ and a lift of $B$ is the assignment of
\Cref{lem:Dassign}, and the transport is one exact assignment on the
$32{,}000 \times 32{,}000$ matrix of them, by the Jonker--Volgenant
algorithm, at every prefix. The certificate is that of \Cref{thm:twosample}
for half $A$ at $\delta = 0.05$, the transport plus the crop
cost plus $\kappa\sqrt{\ln(1/\delta)/n}$, and the whole-sample certificate
of \Cref{thm:twosamplefull}(ii), with its one term at $\delta/2$, is larger
by $0.0008\kappa$ at the largest $n$.

\paragraph{The simulator and the seeds.} The simulator is a compiled
implementation of the drift above, exact against the reference
implementation step by step, and a matching takes from an hour for the ball
flock to five hours for the Cauchy flock on a many-core machine, the
assignment on the full matrix half an hour of it. The seeds are consecutive
integers, one range per flock and regime, stored with every replica and
listed in the supplementary material, which regenerates every table and
figure of this appendix and of \Cref{sec:flocking} from them.

\begin{figure}[t]
\centering
\includegraphics[width=0.42\textwidth]{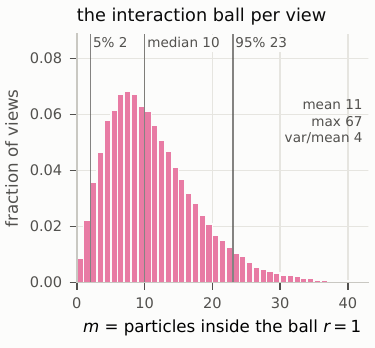}\hspace{0.03\textwidth}%
\includegraphics[width=0.42\textwidth]{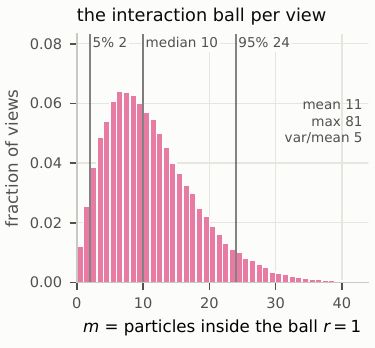}
\caption{The count law of the ball flock after the burn-in: the number $m$
of particles inside the ball $r = 1$ other than the self, per sample, on the
$64{,}000$ replicas at \emph{(left)} $\eta = 1.2$, those of
\Cref{fig:certdim}, and \emph{(right)} $\eta = 2.4$, those of
\Cref{fig:twin}.}
\label{fig:count}
\end{figure}

\paragraph{The count law and the profile.} \Cref{fig:count} is the count
law of the ball flock after the burn-in in both regimes, the number $m$ of
particles inside the ball other than the self. At $\eta = 1.2$ the ball
holds $10.7$ particles on average and $67$ at most, from $2$ to $23$
between the $5\%$ and $95\%$ samples, with a variance $4.3$ times the mean,
and at $\eta = 2.4$ it holds $10.9$ on average and $81$ at most, from $2$
to $24$, with a variance $4.6$ times the mean. The survival
$P(m \ge k)$ is
geometric beyond the knee, at $0.20$ per atom on $15 \le k \le 50$ in both
regimes. Every sample of the ball flock is exactly sparse with uniform
weights on its $m + 1$ atoms and tail $T_k = (1 - k/(m+1))_+$, so the
ball's mean tail in \Cref{fig:certdim}(a) is the functional $\E_Q T_k =
\sum_m P(m)\,(1 - k/(m+1))_+$ of the count law, and it is geometric beyond
the knee because the count law is, at a decay rate of $0.23$ per atom. The graded rules are the other case, a
decay inside each sample at a rate set by the local crowding. Their
profiles are \Cref{fig:certdim}(a) at $\eta = 1.2$ and \Cref{fig:twin}(a)
at $\eta = 2.4$: matched to the ball's at ten atoms, the bump's tail is the
ball's shape on a wider support with the boundary atoms discounted, the
Gaussian's is geometric and the Cauchy's polynomial, above the Gaussian's
from the seventh atom on. At the one-percent level the selectors resolve
$24$ (ball), $27$ (bump), $32$ (Gaussian) and $57$ (Cauchy) atoms at $\eta
= 1.2$, and $24$, $29$, $35$ and $66$ at $\eta = 2.4$. The profiles of the
two regimes are close, the high-noise flock's tails a little heavier.

\begin{table}[t]
\centering
\caption{The two-sample transport $\Wone(\hat\Gamma_A, \hat\Gamma_B)$ and
the certificate of \Cref{thm:twosample} for half $A$ at
$\delta = 0.05$, in units of $\kappa$, at $n = 64{,}000$ samples
($32{,}000$ per half), each flock through the selector of its own rule,
after the burn-in of $3000$ steps, in both regimes, with the cell
certificate of \Cref{prop:sampleshells} on the same halves at the same
$\delta$, its shells, resolution and cutoff optimized on the data (the
covering script of the supplementary material).
The trivial value is $\kappa$, and the mean distance $\Dl$ between two
random samples is $0.56$ to $0.68\kappa$.}
\label{tab:cert}
\small
\begin{tabular}{l ccc ccc}
\toprule
& \multicolumn{3}{c}{$\eta = 1.2$} & \multicolumn{3}{c}{$\eta = 2.4$} \\
\cmidrule(lr){2-4} \cmidrule(lr){5-7}
flock & $\Wone$ & cert. & cell cert. & $\Wone$ & cert. & cell cert. \\
\midrule
ball, $r = 1$              & 0.0272 & 0.0340 & 0.681 & 0.0435 & 0.0503 & 0.881 \\
bump, $r_b = 1.4$          & 0.1057 & 0.1125 & 1.005 & 0.1386 & 0.1454 & 1.004 \\
Gaussian, $h = 0.5$        & 0.1152 & 0.1238 & 1.006 & 0.1498 & 0.1583 & 1.004 \\
Cauchy, $h = 0.9$, $a = 3$ & 0.1205 & 0.1292 & 1.006 & 0.1555 & 0.1643 & 1.005 \\
\bottomrule
\end{tabular}
\end{table}

\paragraph{The certificate.} \Cref{tab:cert} is the certificate of
\Cref{thm:twosample} at the largest $n$ on every flock, in both regimes,
after the burn-in, with the transport it is built
on, and \Cref{fig:certdim}(b) and \Cref{fig:twin}(b) are the curves
against $n$. Four readings. \emph{The bound is non-vacuous by a wide
margin} on every flock and regime: the largest value in the table is
below $0.17\kappa$, against the trivial value $\kappa$ and a mean distance
between two random samples of $0.56$ to $0.68\kappa$, and on the ball at $\eta = 1.2$
the certificate is $0.034\kappa$, a thirtieth of the trivial value.
\emph{The graded flocks sit together}, at four times the ball's transport
at $\eta = 1.2$ and three times at $\eta = 2.4$: the
weight mismatch between two samples is a cost the ball's uniform weights
never pay, and it depends little on how the grading is done, the three
graded certificates lying within $0.02\kappa$ of one another in both
regimes. \emph{The matching is close to its floor.} The mean distance
from a sample of $A$ to its nearest sample of $B$ is $0.95$ to $0.96$ of
the matching cost on the graded flocks and $0.69$ ($0.79$ at $\eta = 2.4$)
on the ball, $0.0188\kappa$ at $\eta = 1.2$, so what remains is the sample
size and not slack in the transport. \emph{The covering certificates sit
at the trivial value.} The cell certificate of \Cref{prop:sampleshells}
on the cells of the quantized $k$-crops of the two halves (on the ball,
where the weights tie, the crop takes the $k$ headings nearest the view's
circular mean), the resolution floor, the observed discrepancy of the halves' cell frequencies and the
trim, with the concentration term of \Cref{thm:twosample}, at the shells,
the resolution and the cutoff $K \le 24$ that are best on the data, is
$0.68\kappa$ on the ball at $\eta = 1.2$ and $0.88\kappa$ at $\eta =
2.4$, twenty and eighteen times the matching certificate, and at or above
$\kappa$ on every graded kernel at every $n$, where the continuous
weights make almost every occupied cell a singleton, and the stratified
bound of \Cref{thm:graded} at a fixed covering constant, from the sample's
own tails with the shells optimized, is above $0.98\kappa$ on every flock
and regime. The comparison is between forms of the same bound on the same
class, the same halves and the same $\delta$: every cell bound is one
feasible transport and the matching is below every one of them
(\Cref{prop:sampleshells}), and a celling optimized on the sample is a
certificate on the same terms (\Cref{rem:certread}).
Not computed is a
chained form, the entropy integral over the scales, or the
doubling-dimension covers of \citet{gottlieb2017efficient}, which could
sit between the two.

\begin{figure}[t]
\centering
\raisebox{-\height}{\includegraphics[width=0.306\textwidth]{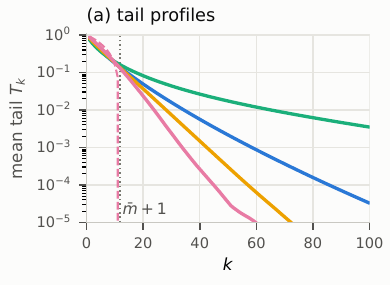}}\hspace{0.018\textwidth}%
\raisebox{-\height}{\includegraphics[width=0.295\textwidth]{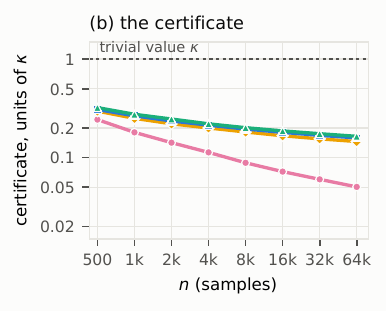}}\\[-2pt]
\includegraphics[width=0.56\textwidth]{figures/certificate_legend}
\caption{The flock at $\eta = 2.4$, the twin of \Cref{fig:certdim}:
$64{,}000$ independent replicas per flock, one sample each. \emph{(a)} The
profile, the mean tail $T_k$ under the four selectors. \emph{(b)} The
certificate of \Cref{thm:twosample} for half $A$ at $\delta
= 0.05$, in units of $\kappa$, after the burn-in of $3000$ steps.}
\label{fig:twin}
\end{figure}

\paragraph{The implied dimension.} The \emph{implied dimension} of a
measured transport is $-1/(\mathrm d\ln\Wone/\mathrm d\ln n)$, the
exponent it shows at the sample sizes attained, read here as the reciprocal
of the least-squares slope of $\ln\Wone(\hat\Gamma_A, \hat\Gamma_B)$
against $\ln n$ over a window of three consecutive doublings, at the middle
$n$. No constant enters: the transport is computed exactly and the slope is
free of units. On interval values an exactly $k$-sparse design with $k \ge
2$ shows at most $2k - 1$ as $n$ grows and the stratum design of
\Cref{app:lower} at least $2k - 1$
(\Cref{cor:sparse}), as the bipartite matching on the $d$-cube decays at
$n^{-1/d}$ \citep[Thm.~15]{luxburg2004distance}. In general the number is
the dimension of the lifted design at the scales the sample resolves.
\Cref{fig:dimension} draws it for the four selectors after the burn-in in
both regimes. Over the last three doublings the
ball's transport decays with the exponent of a $3$-dimensional sample at
$\eta = 1.2$ and of a $5$-dimensional one at $\eta = 2.4$, the graded
flocks' with $13$ to $15$: about ten dimensions of weights at the sample's
scale, against the $k - 1 \approx 10$ that the stratum count charges at the
mean support. The dimensions are local. On the graded flocks the exponent
rises with $n$, from $5$ to $8$ at $n = 1000$ to $13$ at $n = 32{,}000$ at
$\eta = 1.2$, about one dimension per doubling and decelerating: at a finer
scale the sample resolves more atoms of the weight vector, the cutoff $K_n
\asymp \sqrt{\ln n/\beta}$ of \Cref{cor:exp} read off the data, and the
number is the dimension of the lifted design at the scales the sample
resolves, not a limit. On the ball it is between $2$ and $3.5$ at every $n$
at $\eta = 1.2$, where the support has up to $68$ atoms, and the more
ordered flock is the lower-dimensional for every selector.

\begin{figure}[t]
\centering
\includegraphics[width=0.295\textwidth]{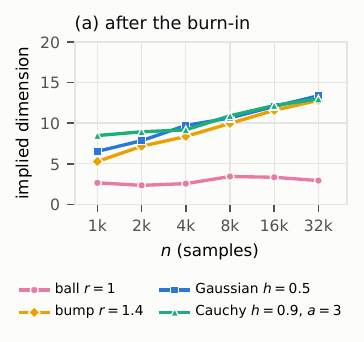}\hspace{0.03\textwidth}%
\includegraphics[width=0.295\textwidth]{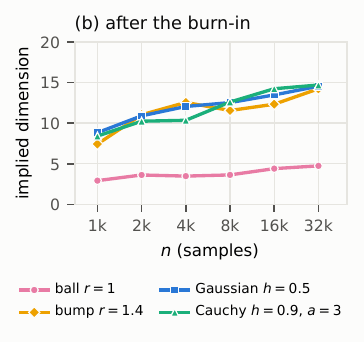}
\caption{The implied dimension, the reciprocal of the local slope of
$\ln\Wone(\hat\Gamma_A, \hat\Gamma_B)$ against $\ln n$ over three
consecutive doublings, at the middle $n$, after the burn-in: \emph{(a)} on
the replicas of \Cref{fig:certdim} at $\eta = 1.2$ and \emph{(b)} on those
of \Cref{fig:twin} at $\eta = 2.4$. The curves are the transports behind
panel (b) of the two figures.}
\label{fig:dimension}
\end{figure}

\paragraph{The two regimes.} \Cref{fig:twin} is \Cref{fig:certdim} at
$\eta = 2.4$, on the replicas of that regime. The profile is nearly the
same, the transport is larger, $0.044\kappa$ on the ball against
$0.027\kappa$: the high-noise flock is less ordered, with a local
polarization of $0.78$ against $0.94$ (below), and the bound sees that too.

\paragraph{The certificate on a member with zero training risk.} For a
target $f^\ast \in \cG_L$ observed on half $A$, the lower Lipschitz
extension $f_-(z) := \max_{i \in A}(f^\ast(z_i) - L\,\Dl(z, z_i))$ agrees
with $f^\ast$ on $A$ and is $L$-Lipschitz \citep[Lemma 7]{luxburg2004distance},
an empirical risk minimizer of $\cG_L$ with zero training risk, and on the
target $f^\ast = L\,\Dl(\cdot, A)$, which is zero on $A$, it is
$-L\,\Dl(\cdot, A)$: its error at $z$ is $2L\,\Dl(z, A)$ and its gap on
half $B$ is exactly $2L\,\overline{\mathrm{NN}}$, twice the mean
nearest-neighbor distance from $B$ to $A$. By \Cref{rem:regression} the
certificate bounds the gap of every empirical risk minimizer of $\cG_L$ on
every target in $\cG_L$ by $2L$ times its value, so on this member the
ratio of the bound to the gap is the certificate over
$\overline{\mathrm{NN}}$: $1.8$ on the ball at $\eta = 1.2$ ($0.034$ over
$0.0188$) and $1.5$ at $\eta = 2.4$, and $1.10$ to $1.13$ on the graded
kernels in both regimes, where the matching is at its nearest-neighbor
floor (\Cref{tab:cert} and the floor above). The bound uniform over the
class is within that factor of the truth on this member, and a member is
what a uniform bound is judged by. For one chosen learner on one target
the gap can be far smaller, the target's own smoothness at the neighbor
scale, which only the labels see \citep{ashlagi2021functions}, and a
holdout estimate on half $B$ reads it: the certificate competes with
nothing there. What it gives is the bound before any learner is chosen or
label seen, at once for every target and predictor in the class, so that
a choice made after seeing the data costs nothing
(\Cref{thm:twosamplefull}(ii)).

\paragraph{The bandwidth sweep.} The graded kernels of \Cref{sec:flocking}
are at the ball's own scale: the Gaussian's $h = 0.5$ is the second-moment
match to the disc of radius $1$ (the disc has mean squared distance
$\tfrac12$, the kernel $2h^2$), and the bump's radius and the Cauchy's
bandwidth match the ball's tail at ten atoms. Their certificates lie
within $0.02\kappa$ of one another, four times the ball's: the price of a
design of dimension thirteen against one of dimension three.
What the bound reads is the profile, and the bandwidth sets the
profile, so we sweep it rather than choose it: the Gaussian rule at $h
\in \{0.125, 0.18, 0.25, 0.35, 0.5\}$, each flock simulated under its own
rule as in the protocol, on $8{,}000$ replicas per bandwidth read at
$4{,}000$ samples per half and on $32{,}000$ replicas read at $16{,}000$
per half, and the flock of \Cref{sec:flocking} at the same prefixes
(\Cref{tab:sweep}). The certificate follows the profile,
slowly: halving the bandwidth sharpens the profile by a factor of nearly
three, from $32$ to $12$ atoms carrying all but one percent of the
weight, lowers the implied dimension from $10$ to $7$, and lowers the
certificate by a tenth at this $n$, and a quarter of the bandwidth, at
$9$ atoms and dimension $6$, lowers it by three tenths, monotone in the
bandwidth over the five values. At $16{,}000$ samples per half every
certificate is lower by 15 to 26\%, $0.133\kappa$, $0.128\kappa$,
$0.113\kappa$, $0.097\kappa$ and $0.081\kappa$ at $h = 0.5$, $0.35$,
$0.25$, $0.18$ and $0.125$, against the ball's $0.043\kappa$ at that $n$.
The weights the nearest atoms
carry vary continuously with their positions whatever the bandwidth, and
those are the coordinates the transport pays for; the gain is in the
rate, $n^{-1/6}$ against $n^{-1/10}$, so it widens with $n$. The narrow
kernels also order the flock less, the polarization falling from $0.94$
to $0.78$, since each particle then follows fewer neighbors.

\begin{table}[h]
\centering
\caption{The bandwidth sweep of the Gaussian rule at $\eta = 1.2$, each
flock under its own rule, read at $4{,}000$ samples per half (the pilot
sets of $8{,}000$ replicas) and at $16{,}000$ per half (the full sets of
$32{,}000$ replicas), $h = 0.5$ the flock of \Cref{sec:flocking} at the
same prefixes: the atoms
carrying all but $1\%$ and $0.1\%$ of the mean weight and the decay rate
per atom between them, the mean local polarization within $r = 1$, the
implied dimension of the transport over the last three doublings, and
the certificate of \Cref{thm:twosample} in units of $\kappa$.}
\label{tab:sweep}
\small
\begin{tabular}{c cc c c c c}
\toprule
$h$ & $k(1\%)$ & $k(0.1\%)$ & decay & $\varpi$ & dim & cert. \\
\midrule
\multicolumn{7}{l}{\emph{$4{,}000$ samples per half}} \\
$0.125$ & 9 & 16 & 0.33 & 0.78 & 5.8 & 0.109 \\
$0.18$ & 11 & 19 & 0.29 & 0.85 & 5.9 & 0.127 \\
$0.25$ & 12 & 21 & 0.26 & 0.90 & 7.2 & 0.142 \\
$0.35$ & 20 & 33 & 0.18 & 0.93 & 7.2 & 0.155 \\
$0.5$ & 32 & 53 & 0.11 & 0.94 & 9.7 & 0.157 \\
\midrule
\multicolumn{7}{l}{\emph{$16{,}000$ samples per half}} \\
$0.125$ & 9 & 16 & 0.33 & 0.78 & 5.8 & 0.081 \\
$0.18$ & 11 & 19 & 0.29 & 0.86 & 6.7 & 0.097 \\
$0.25$ & 13 & 21 & 0.29 & 0.90 & 8.0 & 0.113 \\
$0.35$ & 20 & 33 & 0.18 & 0.93 & 9.9 & 0.128 \\
$0.5$ & 32 & 53 & 0.11 & 0.94 & 12.0 & 0.133 \\
\bottomrule
\end{tabular}
\end{table}

\paragraph{The noise sweep.} Both regimes of the paper are in the ordered
phase. To see where the certificate stops being small we simulated the
ball rule above the transition, at $\eta = 3$, $4$ and $5$, on $8{,}000$
replicas each under the protocol of \Cref{sec:flocking}, read at $4{,}000$
samples per half, with the flocks of the paper at the same prefix
(\Cref{tab:noise}). The certificate grows slowly, from $0.070\kappa$ to
$0.107\kappa$ at $\eta = 3$, and then holds, as the crowding falls with
the order, from $10.7$ to $6.5$ neighbors, and the transport's implied
dimension rises from $2.5$ to $4.9$: a disordered lift is a count and a
spread of headings with no common direction, a few dimensions more than an
ordered one. The polarization gate $\varpi_0 = 0.6$ of the $\eta = 2.4$
regime, below which the update leaves the class (\Cref{prop:direction}),
is met on $69\%$, $27\%$ and $13\%$ of the samples at $\eta = 3$, $4$ and
$5$: above the transition the update is a member on a minority of the
samples, while the certificate is the design's and does not need the
order.

\begin{table}[h]
\centering
\caption{The noise sweep of the ball rule at $4{,}000$ samples per half:
the paper's regimes at that prefix and the pilots above the transition
($8{,}000$ replicas each). The columns of \Cref{tab:sweep}, with the mean
count of neighbors other than the self and the share of samples at the
polarization gate $\varpi_0$ of \Cref{prop:direction}.}
\label{tab:noise}
\footnotesize\setlength{\tabcolsep}{3.5pt}
\begin{tabular}{l ccc c cc c c c}
\toprule
$\eta$ & $k(1\%)$ & $k(0.1\%)$ & $\beta$ & count & pol. & gate & dim. & cert$_A$ & $\varpi_0$ \\
\midrule
$1.2$ & 24 & 34 & 0.23 & 10.7 & 0.94 & 1.00 & 2.5 & 0.070 & 0.8 \\
$2.4$ & 24 & 34 & 0.23 & 10.9 & 0.78 & 0.94 & 3.5 & 0.088 & 0.6 \\
$3$ & 26 & 39 & 0.18 & 10.8 & 0.66 & 0.69 & 4.0 & 0.107 & 0.6 \\
$4$ & 15 & 20 & 0.46 & 7.5 & 0.47 & 0.27 & 4.4 & 0.101 & 0.6 \\
$5$ & 12 & 15 & 0.77 & 6.5 & 0.37 & 0.13 & 4.9 & 0.106 & 0.6 \\
\bottomrule
\end{tabular}
\end{table}

\paragraph{The fully symmetric form.} The drift is invariant
under torus shifts, and depends on particle $i$ only through the
configuration relative to it, with no internal memory. Frame by the focal
particle with the shift $T_i : p \mapsto p - p_i$:
\[
  \tilde p_j := p_j - p_i \;\in\; \T^2_\ell, \qquad
  x_j := \theta_j \;\in\; \cV := S^1 ,
\]
for all $j$, including $j = i$: the framed position $\tilde p_j$ is what the
selector reads, the heading $x_j$ is the value the mechanism reads, and
the self atom sits at position $\tilde p_i = 0$ with its own heading. Framed positions keep the tilde, and $t_j :=
\abs{\tilde p_j}$ is the distance from the focal particle. Headings are
untouched (shift-only framing). Then the selector becomes an ordinary
anonymous selector $c : (\T^2_\ell)^d \to \Delta_d$ of the framed positions
with per-coordinate scores: with $i$ fixed by the framing, $c = c(\tilde p)
\in \Delta_d$ is the weight vector, its $j$-th coordinate $c_j \propto
K(t_j)$ depends on $\tilde p_j$ alone, and $j = i$ contributes $K(0)$, the
maximal weight. The aggregation
is fully anonymous, self included,
\[
  m \;=\; \sum_j c_j\, e^{\mathrm{i}\theta_j}, \qquad
  \theta'_i = \arg m, \qquad p'_i = p_i + v_0\, m/\abs{m},
\]
and the focal state appears only in the framing and un-framing shifts,
which are isometries, never as an argument of the aggregation. The
composite is therefore $\theta'_i = g(x, \tilde p) := f(x, c(\tilde p))$
with $f$ anonymous, the mechanism reading the headings and the selector
the positions, and the design the law of $(x, c(\tilde p))$: a move of
particle $j \ne i$ moves the framed positions in coordinate $j$ only, by
exactly the world distance, and a move of the focal particle moves the
whole frame, an ordinary telescope priced with no new axiom.

\begin{convention}[The self atom stays in]\label{conv:self}
The self atom is always part of the configuration. Beyond physical fidelity
(the Vicsek average includes the particle itself) it pins the normalization
from below: the sum $Z = \sum_l K(t_l)$ contains the self term $K(0)$, so $Z
\ge K(0)$ however small every neighbor term is, which keeps the selector
well defined for an isolated particle and carries the additive locality of
compact windows.
\end{convention}

\begin{remark}[What the framing removes, and what stays]\label{rem:selfin}
Removed: the self as a distinguished argument. $f$ has no query slot, is
symmetric in all $d$ atoms, and cannot tell which atom is the self; it sees
a multiset of weighted headings, and the selector sees an atom at
position $0$ carrying the maximal weight, so a particle exactly on top of
the focal one is rightly treated identically. Kept: the
self data, as an ordinary atom, necessarily so, because shift-only framing
pins the self position but not the self heading, and $\theta_i$ is a
variable the alignment average reads. Deleting the self atom would force $f$
to remember $\theta_i$, a query slot again. Full $SE(2)$ framing, rotating
by $-\theta_i$, removes the self entirely but is wanted only for
heading-dependent interactions (vision cones, blind spots) or to make a
learned mechanism equivariant by construction; rotations are not isometries
of the torus, so that variant lives on a planar relative chart cropped to a
ball of a few interaction ranges, at a crop cost that is exactly zero for
compact windows and exponentially small for Gaussian kernels. Default:
shift-only, self atom in. From here on $(x, \mu)$ denotes the framed
configuration and neighbor moves are moves of coordinates $j \ne i$.
\end{remark}

\paragraph{The value space and its constants.}
\begin{definition}[Value space]\label{def:vspace}
$\cV = S^1$, the heading, with $\rho = d_{S^1}$, the angle in radians.
The framed positions are not values: they are read by the selector alone,
and the design is the law of the headings and the weights.
\end{definition}
$\cV$ is compact and geodesic, $\diam\cV = \pi$, so $R = \pi/2$ and
$\kappa = \pi$, and $N(\varepsilon, \cV, \rho) = \lceil \pi/\varepsilon
\rceil$: metric dimension $1$. A learner who took the positions as values
would pay their dimension in the exponents of \Cref{sec:complexity}, which
carry $\dim\cV$. Every statement of \Cref{sec:complexity}
made on an interval $[-R, R]$ holds on $S^1$ with the same constants. With
$R = \pi/2$, $N(\varepsilon/2, S^1) = \lceil 2\pi/\varepsilon \rceil \le
5R/\varepsilon$ for $\varepsilon \le R$, the bound the proof of
\Cref{thm:covcons} uses for the interval, so $C_1 = 34$ and the constant of
\Cref{cor:exp} are unchanged, and the lower bounds of \Cref{cor:sparse,cor:exp}
transfer by placing the design on a half-circle, which is isometric to an
interval of length $\pi$. By \Cref{thm:covcons} the stratum of $k$-sparse lifted configurations has
\[
  \log N(\varepsilon, \cA_k, \Dl) \;\lesssim\; k\, \log(C/\varepsilon) +
  (k-1)\log(C'/\varepsilon) \;=\; (2k - 1)\,\log(C''/\varepsilon) ,
\]
exponent $2k - 1$ when the weights are free, as they are for
the graded kernels, whose weights are functions of the positions, and $k$
for the ball, whose weights are uniform on its support.

\begin{remark}[The domain size is in the selector]\label{rem:diam}
The torus enters only through the selector's scores. The class's
constants are those of the circle: the reservoir price $R = \pi/2$ is the
worst change of a heading average per unit of weight removed, the crop
constant $\kappa = \pi$ is a half turn, and neither the side $\ell$ nor the
interaction range $r$ appears in them. The framing shift and the position
update $p'_i = p_i + v_0 u'$ are an isometry and a Lipschitz image of the
heading output, and never enter the aggregation.
\end{remark}

\paragraph{Membership: the Vicsek statistics are in the class.} The update
factors through two statistics of the configuration,
\begin{equation}\label{eq:mR}
  \mathbf{m}(x, \mu) \;=\; \sum_{j} \mu_j\, e^{\mathrm{i}\theta_j} \;\in\; \R^2,
  \qquad
  \varpi(x, \mu) \;=\; \abs{\mathbf{m}(x, \mu)} \;\in\; [0, 1],
\end{equation}
the local mean velocity, with headings read from the values and weights
from the selector slot, and its resultant length, the local polarization:
$\varpi = 1$ at perfect alignment, $\varpi \approx 0$ at cancellation, and with
uniform weights over the whole flock $\varpi$ is exactly Vicsek's global order
parameter. Membership comes in two layers of different regularity, and the
ordered-region hypothesis of the second is a lower bound $\varpi \ge \varpi_0$ on the
first.

\begin{proposition}[Unconditional layer]\label{prop:uncond}
For $\mathbf{m}$ and $\varpi$ of \eqref{eq:mR}: (i) each component of $\mathbf{m}$ is in
$\cG_{1}$, and jointly $\abs{\mathbf{m}(P) - \mathbf{m}(Q)}_2 \le \Dl(P, Q)$
for all lifted $P, Q$; (ii) $\varpi$ is in $\cG_{1}$, so the quantity
gating the conditional layer is itself a member, unconditionally.
\end{proposition}

\begin{proof}
(i) $\mathbf{m} = \sum_j \mu_j e^{\mathrm{i}\theta_j} = \int u\,\varphi(s)\, d\Cl$ with $\varphi(s) = e^{\mathrm{i}\theta(s)}$, the
linear sum of \Cref{lem:familycap} with a value-only score. For values $s, t$ with headings $\theta, \theta'$,
$\abs{e^{\mathrm{i}\theta} - e^{\mathrm{i}\theta'}}_2 = 2\sin(d_{S^1}/2) \le
d_{S^1}(\theta, \theta') = \rho(s, t)$. \Cref{lem:familycap}, in its vector-valued form, gives the joint bound. Anonymity is clear. (ii) Post-compose with the
$1$-Lipschitz map $\abs{\cdot}_2$.
\end{proof}

$\mathbf{m}$ is position-blind: value moves in position cost nothing and change
nothing, consistently with positions entering only through the selector.
The update $\theta'_i = \arg \mathbf{m}$ factors as configuration $\mapsto \mathbf{m}$
(\Cref{prop:uncond}) followed by the fixed normalization $\mathbf{m} \mapsto u' =
\mathbf{m}/\abs{\mathbf{m}}$ of $\R^2 \to S^1$, which discards the magnitude and keeps the
direction. Both outputs pass through it: the position update $p'_i = p_i +
v_0\, u'$ consumes the same direction at constant speed, and the magnitude
$\varpi$ never enters the dynamics.

\begin{proposition}[Direction layer]\label{prop:direction}
On the ordered region $\{\varpi \ge \varpi_0\}$ the normalization $u' = \mathbf{m}/\abs{\mathbf{m}}$
satisfies $\abs{u'(\mathbf{m}) - u'(\mathbf{m}')}_2 \le (1/\varpi_0)\abs{\mathbf{m} - \mathbf{m}'}_2$ and $\norm{Du'(\mathbf{m})}
= 1/\abs{\mathbf{m}} \le 1/\varpi_0$. Hence the heading-update components lie in the class
at scale $1/\varpi_0$ on the ordered region, and the position
displacement $v_0\, \mathbf{m}/\abs{\mathbf{m}}$ at scale $v_0/\varpi_0$.
\end{proposition}

\begin{proof}
For unit vectors $u, v$ and $r, s \ge \varpi_0$, $\abs{ru - sv}_2^2 = (r - s)^2 + rs\abs{u - v}_2^2$, so
$\abs{u - v}_2 \le \abs{ru - sv}_2/\sqrt{rs} \le \abs{ru - sv}_2/\varpi_0$: take $u = u'(\mathbf{m})$, $r = \abs{\mathbf{m}}$, $v = u'(\mathbf{m}')$, $s = \abs{\mathbf{m}'}$. The differential form is the Jacobian $(I - u'u'^{\top})
/\abs{\mathbf{m}}$. Compose with \Cref{prop:uncond}(i).
\end{proof}

\begin{remark}[The $\varpi_0$ caveat is the phase transition]\label{rem:phase}
$\varpi \ge \varpi_0$ fails where neighbors cancel, the rule in the
disordered phase and the exception in the ordered one, half a percent of
the samples at $\eta = 1.2$.
There the true update is genuinely unstable, the direction of a near-zero
resultant, so the divergence of the class constant as $\varpi_0 \downarrow 0$ is
the constant pricing the physics rather than slack: $L = 1/\varpi_0$
diverges at the order--disorder transition, where the order-parameter
fluctuations peak \citep{chate2008collective}. The resultant $\varpi$ is
a member unconditionally, and only the direction readout pays $1/\varpi_0$. The
price is attached to the constant-speed rule, which forces the output onto
the unit circle, not to alignment as such: a linear velocity-averaging
dynamics of Cucker--Smale type \citep{cucker2007emergent}, $v'_i = \sum_j c_{ij} v_j$ with no
renormalization, stops at the first layer and is a fully unconditional
member.
\end{remark}

On the replicas after the burn-in the local polarization under the ball
selector has mean $0.94$ at $\eta = 1.2$, with $\varpi \ge 0.8$ on $99.5\%$
of the samples and $\varpi \ge 0.6$ on $99.96\%$, and mean $0.78$ at $\eta =
2.4$, with $\varpi \ge 0.6$ on $94\%$ and $\varpi \ge 0.5$ on $97\%$
($64{,}000$ samples per regime, the resultant of the atoms within $r = 1$,
self included).

The gate $\varpi \ge \varpi_0$ is a condition on $(x, \mu)$ jointly, hence with $\mu =
c(\omega)$ apparently a restriction on the selector as well as on $x$. The next
lemma decouples it: weighted averages of unit vectors respect cones, so a
heading-coherence condition on $x$ alone yields one $\varpi_0$ uniformly over all
selectors with a given tail bound.

\begin{lemma}[Coherence decouples the gate from the selector]\label{lem:coherence}
Fix a range $r_{\mathrm c} > 0$ and $\alpha < \pi/2$, and call the state
$(\alpha, r_{\mathrm c})$-coherent at the focal particle if all headings of
the particles within distance $r_{\mathrm c}$ of it, self included, lie in a
closed arc of half-width $\alpha$, a condition on the headings and the
positions together. Then for every $\mu \in \Delta_d$ placing mass at least
$1 - t$ on the $r_{\mathrm c}$-ball,
\[
  \varpi(x, \mu) \;\ge\; (1 - t)\cos\alpha - t \;\ge\; \cos\alpha - 2t .
\]
Consequently, for any selector whose tail beyond $r_{\mathrm c}$ is at most $t$ on
the design region, the ordered-region hypothesis holds with $\varpi_0 =
\cos\alpha - 2t$, positive once $\cos\alpha > 2t$, uniformly over the
selector family: the restriction is on the state, the standard locally
ordered hypothesis of the flocking literature, and
on the selector only through its standing tail bound.
\end{lemma}

\begin{proof}
Let $e$ be the arc axis. Each in-ball heading vector has $\langle
e^{\mathrm{i}\theta_j}, e\rangle \ge \cos\alpha$, each out-of-ball vector at
least $-1$, so $\langle \mathbf{m}, e \rangle \ge (1-t)\cos\alpha - t$, and $\varpi \ge
\langle \mathbf{m}, e\rangle$.
\end{proof}

\begin{remark}[Region restrictions are benign]\label{rem:region}
$\varpi$ is a functional of the lift, $\varpi = \abs{\int u\, e^{\mathrm{i}\theta(s)}\, dP}$,
so $\{\varpi \ge \varpi_0\}$ is a lift-defined region of the same kind as a stratum
hypothesis: in the learning setup both are conditions on the lifted design.
Membership on the region is direct. $\mathbf{m}$ is $1$-Lipschitz on all lifts
(\Cref{prop:uncond}) and $u'$ is $(1/\varpi_0)$-Lipschitz in $\mathbf{m}$ on the
region (\Cref{prop:direction}), so each component of $u' \circ \mathbf{m}$, restricted
to the lifts of the region, is $(1/\varpi_0)$-Lipschitz there, and McShane's
extension, clipped to $[-1, 1]$, is a member of $\cG_{1/\varpi_0}$ that agrees
with the update on the region. The characterization itself is not the route:
the proof of \Cref{thm:rep} passes through intermediate configurations that
can leave the region. The region is stable under the crop up to a margin:
a $t$-crop moves $\varpi$ by at most $2t/(1-t)$, so after the crop the gate
holds at $\varpi_0 - 2t/(1-t)$.
\end{remark}

\paragraph{What the number bounds.} The certificate bounds the uniform
deviation of the unit ball $\cG_1$ in the class's own units, against the
trivial $\kappa$, and it translates directly into the flock's own
statistics. The components of the local mean velocity are members of
$\cG_1$ with range $2$ (\Cref{prop:uncond}), so on the ball at $\eta = 1.2$
their deviation between the population and the sample is at most
$0.034\kappa = 0.11$ against $2$. The heading update is in $\cG_L$ at $L =
1/\varpi_0$ on the samples whose local polarization is at least $\varpi_0$
(\Cref{prop:direction}), and its extension from that region, clipped to
$[-1, 1]$, is a member of $\cG_{1/\varpi_0}$ on every sample
(\Cref{rem:region}) that differs from the update only on the region $B :=
\{\varpi < \varpi_0\}$, by at most the range $2$. So for a component $u$ of
the update and its extension $\bar u$,
\[
  |\Gamma u - \hat\Gamma_A u|
  \le |\Gamma \bar u - \hat\Gamma_A \bar u| + 2\,\Gamma(B) + 2\,\hat\Gamma_A(B)
  \le \frac{1}{\varpi_0}\,\Wone(\hat\Gamma_A, \Gamma) + 2\,\Gamma(B) + 2\,\hat\Gamma_A(B),
\]
where $\hat\Gamma_A(B)$ is observed and $\Gamma(B)$ is at most the share of
$B$ in the $64{,}000$ samples plus $\sqrt{\ln(1/\delta)/(2 \cdot 64{,}000)} =
0.005$ at $\delta = 0.05$, by Hoeffding's inequality on the indicator of
$B$. At $\eta = 1.2$ and $\varpi_0 = 0.8$ the region holds on $99.5\%$ of the
samples of either half, so the deviation of the update between the
population and the sample is at most $0.13 + 0.03 = 0.16$ against a range
of $2$. At $\eta = 2.4$ and $\varpi_0 = 0.6$ it holds on $94\%$, and the
deviation is at most $0.26 + 0.25 = 0.51$ against $2$, the samples outside
the gate costing as much as the certificate at that noise. For regression with
the absolute loss on a target in the class the gap of every predictor is at
most twice the number (\Cref{rem:regression}), and other losses cost the
factor of their Lipschitz constant.

\section{Lower bounds from the profile}\label{app:lower}

The upper bounds of \Cref{sec:complexity} are functionals of the profile.
This appendix shows that the profile is also what bounds the problem from
below: at a fixed stratum the uniform law is at the $(2k-1)$-dimensional
rate, and so is the risk of learning from noiseless labels
(\Cref{thm:learnlower}), and on a design with an exponential profile it
is quasi-polynomial in $n$. The technique is classical, a
bound on the mass that the lifted design gives a ball, so that $n$ atoms
cover at most half of the design at the scale where the bound is
$1/(2n)$, and the content is in the object: the metric $\Dl$ restricted
to a fiber of fixed weights is bi-Lipschitz to a weighted supremum
distance whose weights are the design's (\Cref{lem:fiber}), so the lifted
design is locally an anisotropic box with semi-axes the weights, and on
an exponential profile its mass at scale $\varepsilon$ is
$\exp(-\ln^2(1/\varepsilon)/(2\beta))$ up to constants. The square loss
has lower bounds of the same shape from the same fibers, by Assouad's
hypercube, sketched in \Cref{rem:squareloss}.

\subsection{The device: fibers of fixed weights}
\label{app:lowerdevice}

\begin{lemma}[Fibers of fixed weights]\label{lem:fiber}
Let $\mu \in \simplex$ have support $S := \{i : \mu_i > 0\}$ with pairwise
distinct entries on $S$, and set, with $\gamma(\mu) := 1$ when $S$ is a
single atom,
\[
  \gamma(\mu) \;:=\; \min_{i \ne j \in S}
  \frac{\abs{\mu_i - \mu_j}}{\max(\mu_i, \mu_j)} \;\in\; (0, 1] .
\]
\begin{enumerate}
\item[(i)] $\Cl$ is injective on $\cV^S \times \{\mu\}$, values off $S$
being forgotten, and for all $x, y \in \cV^d$,
\[
  \frac{\gamma(\mu)}{2}\, \max_{i \in S} \mu_i\, \rho(x_i, y_i)
  \;\le\; \Dl\bigl(\Cl(x,\mu), \Cl(y,\mu)\bigr)
  \;\le\; \sum_{i \in S} \mu_i\, \rho(x_i, y_i) .
\]
\item[(ii)] Write $F_\mu := \Cl(\cV^d \times \{\mu\})$ for the fiber over
$\mu$. If $\mu, \nu$ differ as multisets of positive weights then
$\Dl(F_\mu, F_\nu) \ge R\, m(\mu, \nu) > 0$, where $m(\mu,\nu)$ is the
minimum, over partial matchings of the positive entries of $\mu$ with
those of $\nu$, of the sum of $\abs{\mu_i - \nu_j}$ over matched pairs
plus the sum of the unmatched entries.
\end{enumerate}
\end{lemma}

\begin{proof}
(i) Injectivity: $\Cl(x,\mu) = \Cl(y,\mu)$ forces equal multisets of
positive pairs (\Cref{lem:factor}); the weights being distinct, the
pairing is the identity on $S$. Upper bound: match atom $i$ to atom $i$,
at cone cost $\mu_i\rho(x_i,y_i)$ each. Lower bound: fix $i \in S$
attaining the maximum and follow the atom $(x_i,\mu_i)$ in an optimal
partial matching. Matched to $(y_i,\mu_i)$ it costs $\mu_i\rho(x_i,y_i)$.
Matched to $(y_j,\mu_j)$, $j \ne i$, it costs at least the height term
$R\abs{\mu_i - \mu_j} \ge R\gamma(\mu)\mu_i \ge \tfrac12\gamma(\mu)\,
\mu_i\rho(x_i,y_i)$, as $\rho \le 2R$. Unmatched it costs $R\mu_i \ge
\tfrac12\mu_i\rho(x_i,y_i)$. All other terms are nonnegative. (ii) Value
terms are nonnegative, so every partial matching between $\Cl(x,\mu)$ and
$\Cl(y,\nu)$ costs at least $R$ times its height cost, which is at least
$Rm(\mu,\nu)$; the minimum is over finitely many matchings and vanishes
only if one of them pairs every entry with an equal entry, that is if the
multisets coincide.
\end{proof}

\subsection{Fixed stratum: the calibration from below}
\label{app:lowerstratum}

\begin{definition}[A region of separated weights]\label{def:region}
Let $\cV = [-R, R]$, $k \ge 2$, and put
\[
  \mu^0_i := \frac{2i}{k(k+1)}, \qquad
  \delta_0 := \frac{2}{k(k+1)}, \qquad
  r := \frac{\delta_0}{4k}, \qquad
  m_0 := \frac{3\delta_0}{4},
\]
so $\delta_0$ is both the gap between consecutive coordinates of $\mu^0
\in \Delta_k$ and its minimum, and let
\[
  U_k := \bigl\{a \in \Delta_k : \abs{a_i - \mu^0_i} \le r,\ i \le k-1
  \bigr\},
\]
a cube of side $2r$ in the chart $(a_1, \dots, a_{k-1})$. On $U_k$ every
coordinate is within $\delta_0/4$ of $\mu^0$, the last one because
$(k-1)r \le \delta_0/4$, hence $a_i \ge m_0$ for all $i$, the coordinates
of a single $a \in U_k$ are pairwise at distance at least $\delta_0/2$,
and $\abs{a_i - a'_j} \ge \delta_0/2$ for all $a, a' \in U_k$, $i \ne j$.
The \emph{stratum design} is $Q_k := \mathrm{Unif}([-R,R]^k) \otimes
\mathrm{Unif}(U_k)$ on the first $k$ coordinates, weight $0$ on the
others, normalized Lebesgue measure in the chart: an exactly $k$-sparse
design, $T_k = 0$ a.s., with lift $\Gamma_k$ supported in $S_k :=
\Cl([-R,R]^k \times U_k) \subseteq \cA_k$.
\end{definition}

At $k = 1$ the stratum is the interval at height $1$, and $\Wone(\nu_n,
\Gamma_1) \ge c/n$ for every $n$-atom $\nu_n$ is the quantization bound of
the uniform law on an interval, so the converse of \Cref{cor:sparse} holds
at $k_0 = 1$ as well.

\begin{lemma}[Local bi-Lipschitz bound on the separated region]
\label{lem:bilip}
For $(x,a), (y,a') \in [-R,R]^k \times U_k$ with $\Dl(\Cl(x,a),
\Cl(y,a')) \le \varepsilon < R\delta_0/2$,
\[
  m_0\,\norm{x - y}_1 + R\,\norm{a - a'}_1 \;\le\; \varepsilon .
\]
\end{lemma}

\begin{proof}
Both lifts have $k$ atoms, one per coordinate. A partial matching that
pairs coordinate $i$ of the first with coordinate $j \ne i$ of the second
pays $R\abs{a_i - a'_j} \ge R\delta_0/2$ on that pair, and one that leaves
an atom unmatched pays $Ra_i \ge Rm_0 \ge R\delta_0/2$; either exceeds
$\varepsilon$. So the optimal matching pairs each coordinate with itself,
and $\Dl = \sum_i\bigl[\min(a_i,a'_i)\,\abs{x_i - y_i} + R\abs{a_i -
a'_i}\bigr] \ge m_0\norm{x - y}_1 + R\norm{a - a'}_1$.
\end{proof}

\begin{lemma}[Mass bound, and the uniform law at a stratum]\label{lem:massbound}
Let $\varepsilon_{\max}(k) := R\delta_0/5$,
\[
  C_k := \frac{2^{2k-1}}{m_0^{k}\,R^{2k-1}\,r^{k-1}\, k!\,(k-1)!}, \qquad
  c_k := \tfrac12\,(2C_k)^{-1/(2k-1)} .
\]
For every $P \in \cM_d$ and $\varepsilon \le \varepsilon_{\max}(k)$,
$\Gamma_k(\Ball(P, \varepsilon)) \le C_k\,\varepsilon^{2k-1}$, and for every
$n$ and every $n$-atom probability measure $\nu_n$ on $\cM_d$,
\[
  W_1(\nu_n, \Gamma_k) \;\ge\;
  \min\bigl\{\, c_k\, n^{-1/(2k-1)},\ \varepsilon_{\max}(k)/2 \,\bigr\};
\]
in particular this bounds $\E\, W_1(\hat\Gamma_n, \Gamma_k)$ from below.
\end{lemma}

\begin{proof}
\emph{Mass bound.} If $\Ball(P, \varepsilon)$ misses $S_k$ there is
nothing to prove. Otherwise pick $P' = \Cl(x', a') \in S_k \cap \Ball(P,
\varepsilon)$, so that $\Ball(P, \varepsilon) \subseteq \Ball(P',
2\varepsilon)$. Every $(y, a) \in [-R,R]^k \times U_k$ with $\Cl(y,a)
\in \Ball(P', 2\varepsilon)$ satisfies, by \Cref{lem:bilip} as
$2\varepsilon < R\delta_0/2$, $\norm{x' - y}_1 \le t := 2\varepsilon/m_0$
and $\norm{a' - a}_1 \le 2\varepsilon/R$. Since $Q_k$ is a product
measure,
\[
  \Gamma_k\bigl(\Ball(P', 2\varepsilon)\bigr) \;\le\;
  \frac{(2t)^k}{k!\,(2R)^k}\cdot\frac{(4\varepsilon/R)^{k-1}}{(k-1)!\,(2r)^{k-1}}
  \;=\; C_k\,\varepsilon^{2k-1},
\]
using the volume $(2t)^m/m!$ of the $\ell^1$-ball of radius $t$ in
$\R^m$, in the chart for the second factor, where the $\ell^1$-norm of
the first $k-1$ coordinates is at most the full one. \emph{Uniform law.}
Let $\nu_n = \sum_{i \le n} a_i\,\delta_{P_i}$ and set $\varepsilon :=
\min\{(2nC_k)^{-1/(2k-1)}, \varepsilon_{\max}(k)\}$, so that
$nC_k\varepsilon^{2k-1} \le \tfrac12$. By the mass bound,
$\Gamma_k\bigl(\bigcup_i \Ball(P_i, \varepsilon)\bigr) \le \tfrac12$, and
any coupling of $\nu_n$ with $\Gamma_k$ transports the $\Gamma_k$-mass
outside this union, at least $\tfrac12$, across distance $> \varepsilon$.
Hence $W_1(\nu_n, \Gamma_k) \ge \varepsilon/2$.
\end{proof}

\begin{theorem}[Learning from below at a stratum]\label{thm:learnlower}
Let $k \ge 2$ and let $Q_k$ be the stratum design of \Cref{def:region}. A
\emph{learner} is a measurable map from a sample $(Z_i,
f^\ast(Z_i))_{i \le n}$, $Z_i$ i.i.d.\ $\sim Q_k$, to a function $\hat f$
on $\cC$, and its risk on the target $f^\ast$ is $\E_{Q_k}\abs{\hat f -
f^\ast}$. There are $c'_k > 0$ and $n_k$, depending on $k$ and $R$ only,
such that for every $n \ge n_k$ and every learner
\[
  \sup_{f^\ast \in \cG_1}\; \E\,\E_{Q_k}\abs{\hat f - f^\ast}
  \;\ge\; c'_k\, n^{-\frac{1}{2k-1}} :
\]
at a fixed stratum the labels, however many and noiseless, do not lift the
exponent of \Cref{cor:sparse}.
\end{theorem}

\begin{proof}
Write $\Gamma := \Gamma_k$, $S := S_k$, and fix $\varepsilon \le
\varepsilon'_{\max}(k) := R\delta_0\min\{1/20, 2/k\}$, so that $2\varepsilon
\le \varepsilon_{\max}(k)$ and $\varepsilon/(8R) \le r$.
\emph{A packing.} Let $P_1, \dots, P_N$ be a maximal $2\varepsilon$-separated
subset of $S$, finite by compactness. By maximality the balls $\Ball(P_j,
2\varepsilon)$ cover $S$, so $1 \le \sum_j \Gamma(\Ball(P_j, 2\varepsilon))
\le N\,C_k(2\varepsilon)^{2k-1}$ by the mass bound of \Cref{lem:massbound},
that is $N \ge 1/(C_k(2\varepsilon)^{2k-1})$.
\emph{The mass near a center.} Write $P_j = \Cl(x^j, a^j)$ with $(x^j, a^j)
\in [-R,R]^k \times U_k$. Matching coordinate $i$ with coordinate $i$
(\Cref{eq:matching}), $\Dl(\Cl(y,a), P_j) \le \norm{y - x^j}_1 +
R\norm{a - a^j}_1$, and the last coordinate of $a - a^j$ is bounded by the
sum of the others, so $\{\norm{y - x^j}_1 \le \varepsilon/4\} \times
\{\norm{(a - a^j)_{\le k-1}}_1 \le \varepsilon/(8R)\}$ lifts into
$\Ball(P_j, \varepsilon/2)$. The $\ell^1$-ball of radius $t$ in $\R^m$ has
volume $(2t)^m/m!$, and its intersection with a cube containing the center
contains one of its $2^m$ orthant pieces, wherever the center sits, once
$t$ is at most half the cube's side, which holds for both factors
($\varepsilon/4 \le R$ and $\varepsilon/(8R) \le r$).
As $Q_k$ is the product of the uniform
laws on the cube $[-R,R]^k$ and on the cube $U_k$ of side $2r$ in the chart,
\[
  \Gamma\bigl(\Ball(P_j, \varepsilon/2)\bigr) \;\ge\;
  \frac{(\varepsilon/2)^k}{2^k\,k!\,(2R)^k}\cdot
  \frac{(\varepsilon/(4R))^{k-1}}{2^{k-1}\,(k-1)!\,(2r)^{k-1}}
  \;=:\; c_2\,\varepsilon^{2k-1} .
\]
\emph{The bumps.} Let $\varphi_j(P) := (\varepsilon - \Dl(P, P_j))_+$, a
$1$-Lipschitz function of the lift supported in $\Ball(P_j, \varepsilon)$
and at least $\varepsilon/2$ on $\Ball(P_j, \varepsilon/2)$, so that
$\int\varphi_j\,d\Gamma \ge \tfrac12 c_2\varepsilon^{2k}$. For $\sigma \in
\{-1, 1\}^N$ put $f_\sigma := \sum_j \sigma_j\,\varphi_j \circ \Cl$. It is
in $\cG_1$: the supports are disjoint, the centers being $2\varepsilon$
apart, and for $P$ in the $j$-th support and $P'$ in the $l$-th, $l \ne
j$, $\abs{f_\sigma(P) - f_\sigma(P')} \le \varphi_j(P) + \varphi_l(P') =
2\varepsilon - \Dl(P, P_j) - \Dl(P', P_l) \le \Dl(P, P')$ by the triangle
inequality through $\Dl(P_j, P_l) \ge 2\varepsilon$, while for $P'$ in no
support $\abs{f_\sigma(P) - f_\sigma(P')} = \varphi_j(P) \le \varepsilon -
\Dl(P, P_j) \le \Dl(P', P_j) - \Dl(P, P_j) \le \Dl(P, P')$.
\emph{The signs.} Let $\sigma$ be uniform on $\{-1,1\}^N$, independent of
the sample, and $J$ the set of $j$ whose ball $\Ball(P_j, \varepsilon)$
contains no sample point. The labels $f_\sigma(Z_i)$ are functions of the
sample and of $(\sigma_j)_{j \notin J}$ alone, so conditionally on both,
$(\sigma_j)_{j \in J}$ is still uniform, and for $Z' \in \Ball(P_j,
\varepsilon)$ with $j \in J$, where $f_\sigma = \sigma_j\varphi_j$,
$\E[\abs{\hat f(Z') - \sigma_j\varphi_j(Z')} \mid \text{sample},
(\sigma_j)_{j \notin J}] \ge \varphi_j(Z')$, as $\tfrac12\abs{a - b} +
\tfrac12\abs{a + b} \ge \abs{b}$. Integrating over $Z'$ and summing over $j
\in J$, the risk averaged over $\sigma$, which the supremum over $f^\ast$
dominates, is at least $\sum_j \Pr(j \in J)\int\varphi_j\,d\Gamma$, and
$\Pr(j \in J) = (1 - \Gamma(\Ball(P_j,\varepsilon)))^n \ge 1 -
nC_k\varepsilon^{2k-1}$ by the mass bound.
\emph{The scale.} Take $\varepsilon := (2nC_k)^{-1/(2k-1)} = 2c_k n^{-1/(2k-1)}$,
which is at most $\varepsilon'_{\max}(k)$ for $n \ge n_k :=
(2c_k/\varepsilon'_{\max}(k))^{2k-1}$. Then $\Pr(j \in J) \ge \tfrac12$ and
\[
  \sup_{f^\ast \in \cG_1}\E\,\E_{Q_k}\abs{\hat f - f^\ast}
  \;\ge\; N \cdot \tfrac12 \cdot \tfrac12 c_2\varepsilon^{2k}
  \;\ge\; \frac{c_2\,\varepsilon}{4\,C_k\,2^{2k-1}}
  \;=\; \frac{c_2\,c_k}{4^k\,C_k}\; n^{-\frac{1}{2k-1}} \;=:\; c'_k\,n^{-\frac{1}{2k-1}} .
  \qedhere
\]
\end{proof}

\subsection{Exponential profiles: quasi-polynomial from below}
\label{app:lowerexp}

\begin{theorem}[Exponential profile: the uniform law from below]
\label{thm:lowerexp}
Let $\cV = [-R,R]$, $\beta > 0$, $q := e^{-\beta}$, $\gamma := 1 - q$, and
$d \ge 2$. Let $\mu^\beta \in \simplex$ have entries $\mu^\beta_j :=
\gamma q^{j-1}/(1 - q^d)$, and let $Q^\beta := \mathrm{Unif}([-R,R]^d)
\otimes \delta_{\mu^\beta}$, with lift $\Gamma^\beta$ supported in the
fiber $F^\beta := F_{\mu^\beta}$. Then $T_k \le e^{-\beta k}$ for all
$k$ at every configuration: an exponential profile at rate $\beta$ with
$A = 1$, to which \Cref{cor:exp} applies. For every $n \ge 1$ with $d \ge
\sqrt{2\ln(2n)/\beta} + 1$ and every $n$-atom probability measure $\nu_n$
on $\cM_d$,
\[
  W_1(\nu_n, \Gamma^\beta) \;\ge\; \frac{\gamma^2 R}{16}\,
  e^{-\sqrt{2\beta\ln(2n)}} ;
\]
in particular this bounds $\E W_1(\hat\Gamma_n, \Gamma^\beta) =
\Delta_n(Q^\beta)$ from below.
\end{theorem}

\begin{proof}
\emph{Profile and gaps.} $T_k(\mu^\beta) = \sum_{j>k}\mu^\beta_j = (q^k -
q^d)/(1 - q^d) \le q^k$, as $q^{k+d} \le q^d$. For $i < j$, $\mu^\beta_i - \mu^\beta_j =
\mu^\beta_i(1 - q^{j-i}) \ge \gamma\mu^\beta_i$, so $\gamma(\mu^\beta)
\ge \gamma$ in \Cref{lem:fiber}(i):
\begin{equation}\label{eq:fiberexp}
  \Dl\bigl(\Cl(x,\mu^\beta), \Cl(y,\mu^\beta)\bigr) \;\ge\;
  \frac{\gamma}{2}\max_j \mu^\beta_j\abs{x_j - y_j} \qquad (x, y \in [-R,R]^d),
\end{equation}
and $\Cl$ is injective on $[-R,R]^d \times \{\mu^\beta\}$. Also
$\mu^\beta_j \ge \gamma q^{j-1}$.

\emph{The anisotropic box.} For $t \in (0, \gamma^2R/4]$ put $a(t) :=
\ln(\gamma^2R/(4t)) \ge 0$, $K(t) := \lfloor a(t)/\beta\rfloor + 1 \ge 1$,
and $y_j(t) := \gamma\mu^\beta_jR/(2t) \ge \gamma^2q^{j-1}R/(2t) =
2e^{a(t) - \beta(j-1)}$, so that $y_j(t) \ge 2$ for $j \le K(t)$, as
$\beta(K(t)-1) \le a(t)$. Consequently, for $K = K(t) \le d$ and $a =
a(t)$,
\begin{equation}\label{eq:boxmass}
  \prod_{j \le K}\frac1{y_j(t)} \;\le\; 2^{-K} e^{-aK + \beta K(K-1)/2}
  \;\le\; 2^{-K} e^{-aK/2} \;\le\; e^{-a^2/(2\beta)} ,
\end{equation}
using $\beta(K-1) \le a$ and then $K > a/\beta$.

\emph{The bound.} Fix $P \in \cM_d$ and $\varepsilon$ with $2\varepsilon \le
\gamma^2R/4$ and $K' := K(2\varepsilon) \le d$, and suppose
$\Ball(P,\varepsilon)$ meets $F^\beta$ at $P' = \Cl(x', \mu^\beta)$. Then
$\Ball(P,\varepsilon) \subseteq \Ball(P', 2\varepsilon)$, and
by~\eqref{eq:fiberexp} every $\Cl(y, \mu^\beta) \in \Ball(P',
2\varepsilon)$ has $\abs{y_j - x'_j} \le 4\varepsilon/(\gamma\mu^\beta_j)
= R/y_j(2\varepsilon)$ for every $j$. As $Q^\beta$ is uniform on the
cube, $\Gamma^\beta(\Ball(P,\varepsilon)) \le \prod_{j \le K'}\min\{1,
1/y_j(2\varepsilon)\} \le e^{-a(2\varepsilon)^2/(2\beta)}$
by~\eqref{eq:boxmass}. For an $n$-atom $\nu_n = \sum_i a_i\delta_{P_i}$
take $\varepsilon := \tfrac18\gamma^2R\,e^{-\sqrt{2\beta\ln(2n)}}$, so
that $a(2\varepsilon)^2/(2\beta) = \ln(2n)$ and $K' =
\lfloor\sqrt{2\ln(2n)/\beta}\rfloor + 1 \le d$: then
$\Gamma^\beta(\bigcup_i\Ball(P_i,\varepsilon)) \le n e^{-\ln(2n)} =
\tfrac12$, every coupling of $\nu_n$ with $\Gamma^\beta$ moves
$\Gamma^\beta$-mass at least $\tfrac12$ across distance $\ge
\varepsilon$, and $W_1(\nu_n, \Gamma^\beta) \ge \varepsilon/2$.
\end{proof}

\subsection{Reading}\label{app:lowerread}

\emph{Fixed stratum.} The lower exponent of \Cref{lem:massbound},
$1/(2k-1)$, is the upper exponent of \Cref{cor:sparse} on interval values
for $k \ge 2$, by the multiscale coupling of \Cref{prop:multiscale} on the
cropped stratum: at a fixed stratum the uniform law is $n^{-1/(2k-1)}$ from
both sides. The constants are polynomial in $k$ on the lower side and the
comparison of constants is not pursued. At $k = 1$ the stratum is an
interval, where the empirical measure converges at $n^{-1/2}$ and the
quantization bound $c/n$ of every $n$-atom measure is the converse's form.

\emph{The exponential rate is the rate of the problem.} Between
\Cref{cor:exp} and \Cref{thm:lowerexp} the uniform law of the design
$Q^\beta$ is $e^{-c\sqrt{\beta\ln n}}$ with the constant $c$ lying between
$1/\sqrt2$ and $\sqrt2$, up to lower-order factors. The fiber resembles
the exponential ellipsoid of \citet[Thm.~4.2]{lei2020convergence} at
$\log\gamma = \beta$, whose entropy $\ln^2(1/\varepsilon)/(2\beta)$ is his
Lemma~4.3 and whose rate $e^{-\sqrt{2\log\gamma\,\ln n}}$ has the shape of
\Cref{thm:lowerexp}. Neither the partition nor the crop is the
reason for the quasi-polynomial shape: for every fixed $k$ the lower bound
exceeds $n^{-1/(2k-1)}$ eventually (for $d \ge \sqrt{2\ln(2n)/\beta} +
1$), so $Q^\beta$ is harder than every stratum, while the upper bound
places it below every logarithmic rate. The design is $d$-free in the
sense that matters: $d$ enters only through that condition, the number of
coordinates the construction uses, which is the cutoff $K_n \asymp
\sqrt{\ln n/\beta}$ of \Cref{cor:exp} read from below. The whole lower
bound is the mass bound $\Gamma^\beta(\Ball(P,\varepsilon)) \le
\exp(-\ln^2(\gamma^2R/(8\varepsilon))/(2\beta))$: the local dimension of
the lifted design at scale $\varepsilon$ is $\ln(1/\varepsilon)/\beta$,
the profile's own cutoff, and upper and lower bounds read the same
quantity, one as a covering count and the other as a packing mass. The
stratum count of \Cref{thm:covcons} resolves every value at the scale
$\varepsilon$ and covers the weights over all of $\Delta_k$, whereas the
fiber itself has the metric entropy of the anisotropic box
$\prod_j[\pm R\mu^\beta_j]$ and no weights to cover. Whether a count
resolving the atom of rank $j$ at its own scale $\varepsilon/(k\mu_j)$,
which would recover the fiber's entropy up to a polylogarithmic factor,
raises the constant of \Cref{cor:exp} is left open.

\section{The literature in full}\label{app:related}

\Cref{sec:lit} keeps the contrasts the results rest on. This appendix gives
the lines in full.

\paragraph{Symmetric functions and Deep Sets.} The representation of
\Cref{sec:representation} belongs to a tradition in which an invariant
embedding absorbs a symmetry and the functions above it are free. For
permutation symmetry the embedding is a sum of per-element features: the
fundamental theorem of symmetric polynomials
\citep{weyl1939classical,macdonald1995symmetric} gives every continuous
symmetric function of $n$ scalars as $\rho(\sum_i\varphi(x_i))$ with
$\varphi$ the power sums, Deep Sets \citep{zaheer2017deepsets} is this
statement with the two maps learned, and the sum decomposition is the form
of the set networks that followed
\citep{sannai2019universal,segol2020universal,amir2023injective,zweig2022exponential}.
The width is the price: \citet{wagstaff2022universal} show that a
continuous sum decomposition of all symmetric functions of $n$ elements
needs latent dimension $n$, and that below it there is a Lipschitz
symmetric function that no decomposition approximates better than a
constant does. The embedding of these results carries no metric, so the
class it represents is all continuous symmetric functions and the
complexity sits in the width. Here the embedding carries a metric, the
class is its Lipschitz ball, the sum networks are its linear functionals
with a readout (\Cref{sec:representation}), and the complexity sits in the
geometry of the design. The count-free rate of the linear sums
(\Cref{prop:countfree}) is the phenomenon of distribution regression, the
learning of a function of a distribution from the mean embedding of a
sample of it, at rates in the number of distributions and the size of each
sample \citep{poczos2013distribution,szabo2016learning}: a linear sum is a
function of the mean lift of the configuration and learns at the value
space's rate whatever the count, and the class beyond the linear sums is
what the profile prices. The sample-complexity gain of an invariance is
quantified in another line: on an invariant target the least-squares gap
of a linear model exceeds the invariant model's by the dimension of the
non-invariant directions over $n$ \citep{elesedy2021provably}, an
invariant kernel gains a factor of the group's order for a finite group
\citep{bietti2021sample}, the covering of a quotient feature space gains
$\sqrt{d!}$ for a permutation-invariant network \citep{sannai2021improved},
and on a manifold under a Lie group the minimax
exponent carries the dimension of the quotient
\citep{tahmasebi2023exact}. For the symmetric group on $d$ coordinates
the quotient of the configuration space still has dimension of order $d$,
so anonymity alone buys a constant, and it is the profile that removes
$d$ from the exponent (\Cref{sec:profiles}). The stability results on sum embeddings ask how
well a fixed embedding preserves a Wasserstein distance chosen in advance,
under which sum-based multiset networks are H\"older in expectation and not bi-Lipschitz
\citep{davidson2025holder}, and which aggregators are Lipschitz for a
multiset distance fixed in advance, the sum, the mean, the maximum and
attention against three such distances, is \citet{nikolentzos2026lipschitz},
and a sliced-Wasserstein embedding of multisets that is bi-Lipschitz for
that distance, with no bi-Lipschitz embedding of measures at all, is
\citet{amir2025fourier}, while \citet{debie2019stochastic} build networks
on weighted point clouds that are Lipschitz for a Wasserstein distance
chosen in advance. Networks on spaces of probability measures with a
universal approximation theorem are \citet{pevny2019approximation}, and
the learning of Lipschitz and Sobolev functions on Wasserstein space from
samples is \citet{fornasier2023approximation} and \citet{pham2023meanfield}:
functions of the weighted value measure, the pointwise subclass of
\Cref{rem:whyuniform}, in which the weights are mass.
Here the question is which distance the class itself determines. On a fixed hard selection the two views agree, since
the restriction of our class is every symmetric Lipschitz function of the
selected values (\Cref{cor:fiber}).

\paragraph{Graph networks.} A graph network aggregates over the
neighborhood of each node, a set that differs from node to node, and its
neighborhood aggregation is our object with a hard selector. Its
generalization bounds, Rademacher and PAC-Bayes bounds for a fixed
architecture \citep{garg2020generalization,liao2021pacbayes},
charge the maximum degree of the training graphs per layer, or its spectral
relaxation \citep{ju2023generalization} or norms of the propagation matrix
that reduce to degree bounds \citep{esser2021learning}, at the rate
$m^{-1/2}$ with the degree a constant in front. Whether any finer statistic
of the data can be used was posed as a question in that line
\citep{liao2021pacbayes,morris2024position}.
The graphon line \citep{maskey2022generalization}, with its extension to
sparse graphs and mixtures of graphons \citep{maskey2024mixture}, points
the other way:
with mean aggregation on dense graphs its bounds improve with the node
count, with constants that grow as the kernel's minimal degree shrinks,
where ours improve as the design gets sparser and never see the node count. These bounds cover the parameter
ball of one network, whose parameters are shared across nodes so that the
node count drops out, and with it the neighbor count. What they become when
asked to learn our class is \Cref{rem:gnncurrency}, a bound at the maximum
count. Lower bounds exist in graph learning, in Weisfeiler--Leman classes and in
margins \citep{vasileiou2025survey}, and on the Rademacher complexity of a
graph convolutional network at a constant degree on one fixed graph
\citep{lv2021generalization}, none minimax, and none for a neighbor count
that varies with the configuration.

A second line puts a metric on the graphs under which message-passing
networks are Lipschitz. The metrics are the tree mover's distance \citep{chuang2022tree},
the cut distance of graphon signals \citep{levie2023graphon}, unbalanced
transport distances with a mass-difference term
\citep{boker2023fine,rauchwerger2025generalization}, and lately one metric
on graphs of every size, sparse and dense, through graphops
\citep{amran2026graphop}. It covers the graph space and applies a generic
statistical step, a Rademacher bound or the robustness bound of
\citet{xu2012robustness} or its occupied-cell form
\citep{kawaguchi2022robustness}. Covered Forest \citep{vasileiou2025covered}
is its latest form, with covering numbers through Weisfeiler--Leman
classes. Ours has the same architecture, the lift
with $\Dl$, the strata with their covering numbers and the coupling, and
differs in three places. The direction of the geometry: they start from a
model and find a metric in which it is Lipschitz, so the Lipschitz ball is
an envelope of the model and the metric one choice among several, where
\citet{rauchwerger2025generalization} come closest with a converse and the
model class dense in the continuous functions of their space. We start from
the axioms, and the representation theorem is an equality. What is
covered: they cover the space at one scale, or with the sample's cell
proportions at one radius \citep[Prop.~4]{carrasco2025rademacher} and
per-cell radii \citep[Thm.~35]{vasileiou2025covered}, and we cover the
measure, by the profile across the shells with a multiscale coupling
within each, and by the sample's own configurations. What is proved, and
what is non-vacuous, below the trivial value of the risk in the sense of
\citet{dziugaite2017computing} and of the risk certificates of
\citet{perezortiz2021tighter}: the covering number of a graph space is astronomical
where it is known \citep{levie2023graphon} and unbounded where it is not
\citep{amran2026graphop}, so the theorems' values are not reported on a
benchmark, and the non-vacuous values Covered Forest reports evaluate the
Weisfeiler--Leman classes the dataset occupies, with a Lipschitz constant
estimated from the data, a counting of the kind of the occupied-cell bound
of \citet{kawaguchi2022robustness}, while the non-vacuous value of
\citet{ju2023generalization} is a PAC-Bayes bound evaluated at the trained
network through its loss Hessian, a bound on one trained network rather
than over a class. The certificate of \Cref{thm:twosample} is a theorem on
the data at every step, and on the flock it is a thirtieth of the trivial
value (\Cref{sec:flocking}). That a class large enough to hold a trained
network makes every uniform bound vacuous \citep{nagarajan2019uniform} is
an objection to the size of the class: here the class is the Lipschitz
ball of the lift, its uniform deviation on the design is measured rather
than bounded from above by a covering, and on the flock it is small.

\paragraph{Empirical complexity.} The bound behind the certificate is the
classical one, the maximum discrepancy of \citet{bartlett2002model}
through the symmetrization of \citet{vapnik1971uniform}, and what is ours
is that on the class it is the transport itself, by the representation,
and is computed at one observed split. A similar quantity has recently
been used in a different context. \citet[Thm.~2]{chuang2021measuring}
take the expected transport between two independent samples of a feature
distribution, the $k$-variance of \citet{solomon2022kvariance}, as the
complexity term of a margin bound for the readouts of a trained encoder,
with an empirical version at one concentration term (their Lemma~5 and
Cor.~6), and compute it on data as a measure that tracks generalization
rather than evaluating the bound. \citet{hou2023instance} evaluate a transport bound
with the counts of a fixed partition, and \citet{zhu2021understanding}
compute sample covering numbers on image data with the bound itself not
evaluated, in the shape of the sample-cover bound of
\citet{mansour2000generalization}. None of these is the bound of
\Cref{thm:twosample}: theirs is a complexity term or a partition count
entering a bound on a model, in the metric of its features, ours the
observed gap of the class itself in the metric the class implies.
The closest form is in domain adaptation, where the target risk of a
Lipschitz class is bounded by the source risk plus the transport between
the empirical source and target samples \citep[Thm.~2]{redko2017theoretical},
\citep[Thm.~1]{shen2018wasserstein}, with the joint error of the two
domains and one concentration term per sample: the two halves of
\Cref{thm:twosample} share a law, so the joint error is absent and one
term serves, and the transport is a matching in the class's own metric,
computed exactly.

\paragraph{Attention.} Attention \citep{vaswani2017attention} gives each
output its own field of view: the softmax weights over the inputs are a
weight vector of the kind this paper is about, they vary with the input,
and in some heads and on some inputs a few tokens carry most of the
weight. Fixed-pattern attention restricts each output to a window or a
sparse pattern chosen in advance \citep{child2019sparse,beltagy2020longformer},
a hard selector given, with the softmax inside the pattern a graded one.
The Lipschitz constant of a whole attention map, weights included, is
\citet{kim2021lipschitz}, a statement about the composite rather than
about the mechanism with its weights held as an input. Attention as a
map on measures, a self-interacting particle system that is Lipschitz in
a transport metric \citep{vuckovic2020mathematical}, with a constant
bounded by $\sqrt n$ in the sequence length on a compact set and
independent of it in the mean-field regime \citep{castin2024smooth}, is
the same composite with uniform weights and the selector inside it. The learning-theory
results nearest to ours are on transformers whose target depends on a few
tokens of the context: a sample complexity almost independent of the
context length when the positions of the $q$ tokens are given in the
prompt \citep{mousavihosseini2025transformers}, logarithmic in it for
sparse variable creation \citep{edelman2022inductive}, with the $q$-sparse
average as the representational primitive
\citep{sanford2023representational}, length-independent norm bounds
\citep{trauger2024sequence}, an input-dependent importance of the
coordinates \citep{takakura2023approximation}, and a per-sample sensitivity
cap as the locality condition \citep{nagler2023statistical}, a one-layer
transformer that learns the sparse token selection by gradient descent
where a fully connected network cannot \citep{wang2024transformers}, and
a sample complexity in the inverse fraction of the label-relevant tokens,
a statistic of graded sparsity, with an attention map that training
sparsifies \citep{li2023theoretical}. Those have a
selection of tokens, hard at a fixed count or graded by an importance the
network infers, an ordered link and a parametric class, and ask which
architecture learns to retrieve the tokens.
Ours has graded
weights given as input and a varying count, anonymity, the full Lipschitz
class, a profile with lower bounds, and a certificate, and asks the
complexity of the mechanism once the weights are given.
\citet{nikolentzos2026lipschitz} show that attention pooling has instances
that are Lipschitz for none of the three multiset distances they consider, a statement about the
composite \eqref{eq:composite} with the selector inside it; with the weights
an input and the cap an axiom, the mechanism $f$ is Lipschitz in the metric
the class determines, and the composite is not treated here. The mechanism
is nevertheless outside the scope of this paper. Its weights
are learned, together with everything else, and an attention block is not
trained on its own, so the problem studied here, learning a
mechanism whose selector is given, does not arise in that form.

\paragraph{Interaction kernels.} The learning of interaction laws in
particle systems from trajectories is the subject of
\citet{lu2019nonparametric,lu2021learning}: the mechanism is a mean-field
sum $N^{-1}\sum_j\varphi(\abs{x_j - x_i})(x_j - x_i)$ with the scalar
kernel $\varphi$ unknown, and the least-squares estimator converges at the
one-dimensional nonparametric rate, in the $L^2$ norm of the law of the
pairwise distances along the trajectories, a measure-level norm. In the
terms of \Cref{sec:axioms} their class is the linear sums with the weights
uniform over all agents and the score a fixed function of the pairwise
displacement, so the rate is that of the one-dimensional kernel and no
profile enters; here the class is every Lipschitz function of the lift,
whose rate the profile sets, with the selector given. Their treatment of
dependence along a trajectory is what the independent replicas of
\Cref{sec:examples} sidestep. The same mean-field object with the drift
a nonparametric function of the state and of the empirical measure is
estimated from one trajectory per particle at the minimax rate of its
H\"older class \citep{yao2022meanfield}, uniform weights again and no
selector. In the animal groups the object inferred is the selector
itself, the interaction range of a starling flock from field data by a
maximum-entropy fit, a fixed number of topological neighbors
\citep{bialek2012statistical}, and in robot swarms a graph network learns
the flocking controller with the communication graph given
\citep{tolstaya2020learning}: the second is the setting of this paper, the
first its complement.

\paragraph{The matching distance.} The metric of \Cref{sec:representation}
is produced by the axioms, and the derivation is new: none of the lines
that built its ingredients starts from a class of functions. Each
ingredient was met before from the other end, introduced to compare
measures of unequal mass. Mass must be allowed
to appear or vanish, the cheapest way to price that is a reservoir at a
linear price per unit, and the height coordinate of the cone is the mass
itself. The Kantorovich--Rubinstein norm extends from measures of mass zero
to all finite measures by paying the variation of the untransported part,
which gives the flat metric, whose dual is the space of bounded Lipschitz
functions \citep{kantorovich1958space,hanin1992kantorovich}, reached again
from transport equations with a source term as the generalized Wasserstein
distance \citep{piccoli2014generalized,piccoli2016properties}.
\citet{figalli2010new} let mass leak through the boundary of the domain at
the cost of its squared distance to it. In unbalanced transport, the least cost of
carrying mass $m_0$ at $x_0$ to mass $m_1$ at $x_1$, with transport at unit
price and creation or destruction at price $\delta$ per unit, is
$\min(\abs{x_1 - x_0}, 2\delta)\min(m_0, m_1) + \delta\abs{m_1 - m_0}$
\citep{chizat2018unbalanced}, the cone cost \eqref{eq:conecost} at $\delta =
R$. The same function is the marginal perspective of the total-variation
entropy-transport problem, which \citet{liero2018optimal} rewrite as
transport on the cone over the base space with mass-weighted marginals. A persistence diagram is a multiset of birth--death pairs, and
its distance to another diagram lets a feature vanish at the cost of its
distance to the diagonal, available with infinite multiplicity
\citep{cohensteiner2010lipschitz,turner2014frechet}: the matching form with
a reservoir on multisets of unit masses. On multisets of vectors the
matching distances of \citet{chuang2022tree}, \citet{davidson2025holder} and
\citet{nikolentzos2026lipschitz} charge an unmatched element its norm, the
origin as the reservoir. \citet{divol2021understanding}
identify it with boundary transport and compute it by padding each side with
the other's mass at the reservoir, the identity \eqref{eq:padding}
\citep[also][]{lacombe2018large}. \citet{bubenik2022universality} prove that on any metric space with a distinguished subset the Wasserstein
distance on such multisets is universal, the largest subadditive metric
compatible with the ground metric, with Kantorovich--Rubinstein duality at $p = 1$: \Cref{lem:metric}(iii) is
that theorem on the cone with its apex. What \Cref{thm:rep} adds is not the
space but the identification. The configuration is a multiset because
anonymity says so, the value cost at the smaller height is the value cap,
the height is the cap weight rather than a mass, and the price $R$ of a unit
of height is the modulus of the linear sums (\Cref{prop:modulus}). The
axioms' per-coordinate moduli integrate, with constant $1$ and for every
pair of configurations, to exactly the matching cost (\Cref{app:thmproof}).
\Cref{thm:rep}(ii), that the metric is the smallest making the class
Lipschitz, is the class-side statement and a corollary of (i). Two consequences
separate this from the mass-bookkeeping reading. The marginals are counting
marginals, one atom per participant whatever its weight: with mass-weighted
marginals the same cost gives the flat metric, whose Lipschitz ball is a
strictly smaller class in which the selected maximum has no constant at all
(\Cref{rem:whyuniform}). And the two semantics differ only in the height
coordinate. The statistics done on these spaces asks different questions:
consistency and minimax rates for the diagram of a sample under a support
assumption \citep{chazal2021introduction}, populations of diagrams and their
means \citep{mileyko2011probability,turner2014frechet}, and plug-in rates for
the Kantorovich--Rubinstein distance of unbalanced transport, for
point-process intensities, under a covering dimension of the space
\citep{struleva2025sharp}. On the diagram space a hypothesis on the space
alone is weak, since with the bottleneck distance every separable bounded
metric space embeds isometrically into it \citep{bubenik2020embeddings}, and \Cref{sec:complexity}
supplies what is missing, the design's profile over the strata of the lift.

\paragraph{Partitions, transport, and one-number dimensions.} Bounding the transport distance between two measures by the mass they assign
to the cells of a partition is the classical device behind the rates of
convergence of empirical measures in Wasserstein distance from
\citet{dudley1969speed} onward \citep{ajtai1984optimal,dereich2013constructive,boissard2014mean,fournier2015rate}, stated by
\citet[Prop.~4]{weed2019sharp} in the form we use, one partition at one
scale with a trimmed remainder. The stratified argument behind our rates
has been used several times: \citet{fournier2015rate} stratify $\R^d$ into annuli,
\citet{lei2020convergence} an unbounded function space into shells, and
\citet{huang2026wasserstein}, recently, the space of configurations into
cardinalities, for the empirical measures of point processes, a setting
that is the unweighted case of ours. We add two things. The first is
\emph{continuity}, the grading and the tubes: \Cref{thm:graded}
stratifies by the weight tail instead of the count, so that each stratum
is a tube around a sparse set (\Cref{lem:tube}) with the crop in place of
the truncation (\Cref{prop:coupling}), and this is what gives rates on
kernels of full support, where there is no count to stratify by,
including the exponential rate from a decay inside the configuration
(\Cref{cor:exp}). The tail $T_k$ itself is the best $k$-term $\ell^1$
approximation error of the weight vector \citep{cohen2009compressed}, so
an exponential profile is a geometric compressibility of the weights on
average, taken here as a learning parameter. The second is that rate itself, for the mean tail of a
configuration with free weights and values: quasi-polynomial rates of the
same shape are those of \citet[Thm.~4.2]{lei2020convergence}, for
empirical measures in a Hilbert space under an exponential-ellipsoid
moment condition, and of \citet[Thms.~2, 3]{huang2026wasserstein}, for
point processes with a geometric count law, the unweighted case of a
hard-ball selector, and neither hypothesis is ours (\Cref{app:lowerread}
for the fiber, where the class's metric is Lei's ellipsoid). The rates of Weed and Bach are
in one number of the measure, its upper Wasserstein dimension, a limit in
the scale, or at finite $n$ the trimmed covering count optimized over one
scale and one trim \citep[Prop.~10]{weed2019sharp}. The entropic dimension
of \citet{chakraborty2025gan}, restated in \citet{chakraborty2024supervised}, is another such
number, a trimmed covering in the limit of small scale, with a minimax
lower bound in the upper Wasserstein dimension (Theorem 29 of the former), and the regression rates in the intrinsic dimension of
the support \citep{bickel2007local,kpotufe2011knn} read the support, which
for the lifted flock is every stratum up to the crowded one. A design is not
a measure with a dimension but a mixture of strata of different dimension
with a mass on each, and the bulk of the lifted flock is such a union
(\Cref{fig:count}, counts from $2$ to $23$ other than the self between the $5\%$ and $95\%$
views). The
certificate resolves the strata at the sample's own configurations, and the
implied dimension of \Cref{app:flocking} is the estimator of
\citet{block2022intrinsic}, who read the intrinsic dimension of a measure
from the decay in $n$ of the transport between two independent samples,
applied to the lifted design. We keep from these works the point of contact,
that a measure-level dimension controls rates, and the shape of the minimax
lower bound.

\end{document}